\RequirePackage{latexrelease}
\documentclass{article}
\PassOptionsToPackage{numbers, sort&compress}{natbib}
\usepackage[ruled,noend,linesnumbered]{algorithm2e}
\newcommand{\nextnr}{\stepcounter{AlgoLine}\ShowLn}

\usepackage{hyperref}

\usepackage{microtype}

\usepackage{float}

\makeatletter
    \newif\ificml
    \@ifpackageloaded{icml2020}{\icmltrue}{\icmlfalse}
\makeatother
\ificml
    \usepackage{geometry}
\makeatletter
    \Gm@restore@org % restore icml default margins
\makeatother

\usepackage{apxproof}
\usepackage{mathmacros}
\usepackage{formatmacros}
\usepackage{bm}
\usepackage{xspace}
\usepackage{dsfont}

\providecommand{\etil}{\tilde{e}}

\providecommand{\gtil}{\tilde{g}}

\providecommand{\wtil}{\tilde{w}}

\providecommand{\Gtil}{\tilde{G}}

\providecommand{\Otil}{\tilde{O}}
\providecommand{\Ptil}{\tilde{P}}

\providecommand{\Hcal}{\mathcal{H}}

\providecommand{\Ncal}{\mathcal{N}}

\providecommand{\Xcal}{{X}}

\renewcommand{\vec}[1]{\mathbf{\boldsymbol{#1}}}

\providecommand{\uvec}{\vec{u}}
\providecommand{\vvec}{\vec{v}}

\providecommand{\xvec}{\vec{x}}
\providecommand{\yvec}{\vec{y}}

\ifx\BlackBox\undefined
\providecommand{\BlackBox}{\rule{1.5ex}{1.5ex}}  % end of proof
\fi

\ifx\QED\undefined
\def\QED{~\rule[-1pt]{5pt}{5pt}\par\medskip}
\fi

\ifx\proof\undefined
\newenvironment{proof}{\par\noindent{\em Proof:\ }}{\hfill\BlackBox\\[.0mm]}
\fi

\ifx\theorem\undefined
\newtheorem{theorem}{Theorem}
\fi

\ifx\example\undefined
\newtheorem{example}{Example}
\fi

\ifx\property\undefined
\newtheorem{property}{Property}
\fi

\ifx\lemma\undefined
\newtheorem{lemma}{Lemma}
\fi

\ifx\proposition\undefined
\newtheorem{proposition}{Proposition}
\fi

\ifx\remark\undefined
\newtheorem{remark}{Remark}
\fi

\ifx\corollary\undefined
\newtheorem{corollary}{Corollary}
\fi

\ifx\definition\undefined
\newtheorem{definition}{Definition}
\fi

\ifx\conjecture\undefined
\newtheorem{conjecture}{Conjecture}
\fi

\ifx\axiom\undefined
\newtheorem{axiom}[theorem]{Axiom}
\fi

\ifx\assumption\undefined
\newtheorem{assumption}{Assumption}
\fi

\providecommand{\EE}{\mathds{E}} % Expectation
\providecommand{\NN}{\mathds{N}} % Natural numbers
\providecommand{\RR}{\mathds{R}} % Real numbers
\providecommand{\ZZ}{\mathds{Z}} % Integers
\providecommand{\diag}{\mathop{\mathrm{diag}}}
\providecommand{\rank}{\mathop{\mathrm{rank}}}

\providecommand{\mychoose}[2]{\left(\begin{array}{c} #1 \\ #2 \end{array}\right)}

\providecommand{\smallfrac}[2]{{\textstyle \frac{#1}{#2}}}

\providecommand{\eg}{{e.g.}\xspace}
\providecommand{\ie}{{i.e.}\xspace}
\providecommand{\iid}{{i.i.d.}\xspace}

\providecommand{\nbr}[1]{\left\|#1\right\|}
\providecommand{\abr}[1]{\left|#1\right|}
\providecommand{\floor}[1]{\left\lfloor #1 \right\rfloor}
\providecommand{\ceil}[1]{\left\lceil #1 \right\rceil}
\providecommand{\inner}[2]{\left\langle #1,#2 \right\rangle}

\providecommand{\zero}{\mathbf{0}} % Zero

\providecommand*\rmd{\mathop{}\!\mathrm{d}}

\providecommand{\val}{\vec{\alpha}}
\providecommand{\vbeta}{\vec{\beta}}

\providecommand{\vga}{\vec{\gamma}}

\providecommand{\vth}{\vec{\theta}}

\providecommand{\expunder}[1]{\mathop{\EE}\limits_{#1}}

\providecommand{\where}{{\quad \text{where} \quad}}

\providecommand{\parrelu}{{\sf Par-ReLU}}
\providecommand{\parmaxmin}{{\sf Par-MaxMin}}
\providecommand{\glip}{{\sf Gauss-Lip}}
\providecommand{\ilip}{{\sf Inverse-Lip}}

\usepackage{formatmacros}

\usepackage{inputenc}
\DeclareUnicodeCharacter{3B5}{$\epsilon$}
\DeclareUnicodeCharacter{3A6}{$\Phi$}
\DeclareUnicodeCharacter{1D453}{$f$}
\DeclareUnicodeCharacter{1D448}{$U$}
\DeclareUnicodeCharacter{1D45D}{$p$}
\DeclareUnicodeCharacter{0393}{$\Gamma$}

\let\gls\emph
\let\cites\citep

\let\cite\citep
\let\textcite\citet
\let\d\dv

\def\brk#1{\sbr{#1}}                            % [n] = {1, ... , n}

\NewDocumentNotation{\measf}{\mathcal{L}_0}     % Borel measurable functions
\NewDocumentNotation{\leb}{\mathcal{L}}         % Lebesgue space
\NewDocumentNotation{\probm}{\mathfrak{P}}      % Borel probability measures
\NewDocumentNotation{\costb}{\mathrm{B}}        % transportation cost ball
\NewDocumentNotation{\ball}{\mathrm{B}}         % unit ball in a metric space 
\NewDocumentNotation{\cost}{cost}               % transportation cost ball
\NewDocumentNotation{\lips}{lip}                % generalised lipschitz constant
\NewDocumentNotation{\coup}{\upPi}              % set of couplings
\NewDocumentNotation{\conv}{co}                 % convex hull
\NewDocumentNotation{\clconv}{\overline{\conv}} % closed convex hull
\NewDocumentNotation{\cvxity}{\uprho}           % lack of convexity

\NewDocumentNotation{\ind}{\upiota}             % indicator function
\NewDocumentNotation{\dirac}{\updelta}          % Dirac measure

\NewDocumentNotation{\id}{Id}                   % identity function
\NewDocumentNotation{\contf}{C}                 % continuous functions
\NewDocumentNotation{\contfb}{C_\mathrm{b}}     % continuous bounded functions

\NewDocumentCommand{\subdiff}{}{\uppartial}     % subdifferential
\NewDocumentCommand{\esubdiff}{E{_}{\epsilon}}{\uppartial_{#1}} % approximate subdifferential
\NewDocumentCommand{\wtop}{m}{\upsigma\rbr{#1}} % weak topology generator
    \shortautorefnames
\else
    \usepackage[noblocks]{authblk}
    
    \usepackage{natbib}

    \hypersetup{hidelinks}
    \setcitestyle{authoryear,round,citesep={;},aysep={,},yysep={;}}
    \normalautorefnames
\fi

\def\thetitle{Generalised Lipschitz Regularisation Equals Distributional Robustness}
\ificml
    \icmltitlerunning{\thetitle}
\fi

\begin{document}
\ificml
    \twocolumn[
        \icmltitle{\thetitle}



        \icmlsetsymbol{equal}{*}

        \begin{icmlauthorlist}
        \icmlauthor{Zac Cranko}{anu,d61}
        \icmlauthor{Zhan Shi}{uic}
        \icmlauthor{Xinhua Zhang}{uic}
        \icmlauthor{Richard Nock}{d61,anu,usyd}
        \icmlauthor{Simon Kornblith}{goog}
        \end{icmlauthorlist}

        \icmlaffiliation{anu}{Australian National University}
        \icmlaffiliation{d61}{Data 61, CSIRO}
        \icmlaffiliation{usyd}{University of Sydney}
        \icmlaffiliation{uic}{University of Illinios, Chicago}
        \icmlaffiliation{goog}{Google Brain}

        \icmlcorrespondingauthor{Zac Cranko}{zac.cranko@anu.edu.au}

        \icmlkeywords{Machine Learning, ICML}

        \vskip 0.3in
        ]
\else
    \title{\bf\thetitle}
    \renewcommand*{\Authsep}{\quad}
    \renewcommand*{\Authand}{\quad}
    \renewcommand*{\Authands}{\quad}
    \author[1,2]{Zac Cranko}
    \author[4]{Zhan Shi}
    \author[4]{Xinhua Zhang}
    \author[2,1,3]{Richard Nock}
    \author[5]{Simon Kornblith}

    \affil[1]{Australian National University}
    \affil[2]{Data 61, CSIRO}
    \affil[3]{University of Sydney}
    \affil[4]{University of Illinios, Chicago}
    \affil[5]{Google Brain}
    \maketitle
\fi



\ificml
    \printAffiliationsAndNotice % otherwise use the standard text.
\fi

\begin{abstract}
    The problem of adversarial examples has highlighted the need for a theory of regularisation that is general enough to apply to exotic function classes, such as universal approximators. In response, we give a very general equality result regarding the relationship between distributional robustness and regularisation, as defined with a transportation cost uncertainty set. The theory allows us to (tightly) certify the robustness properties of a Lipschitz-regularised model with very mild assumptions. As a theoretical application we show a new result explicating the connection between adversarial learning and distributional robustness. We then give new results for how to achieve Lipschitz regularisation of kernel classifiers, which are demonstrated experimentally.
\end{abstract}

\section{Introduction}
\label{submission}
When learning a statistical model, it is rare that one has complete access to the distribution. More often it is the case that one approximates the risk minimisation by an empirical risk, using sequence of samples from the distribution. In practice this can be problematic --- particularly when the curse of dimensionality is in full force --- to \begin{enumerate*}[label=\textit{\alph*})]
    \item know with certainty that one has enough samples, and
    \item guarantee good performance away from the data. 
\end{enumerate*}
Both of these two problems can, in effect, be cast as problems of ensuring generalisation. A remedy for both of these problems has been proposed in the form of a modification to the risk minimisation framework, wherein we integrate a certain amount of distrust of the distribution. This distrust results in a certification of worst case performance if it turns out later that the distribution was specified imprecisely, improving generalisation.

In order to make this notion of distrust concrete, we introduce some mathematical notation. 
The set of Borel probability measures on an outcome space $\Omega$ is $\probm(\Omega)$. A loss function is a mapping $f:\Omega\to\Rx$ so that $f(\omega)$ is the loss incurred with some prediction under the outcome $\omega\in\Omega$. For example, if $\Omega=X\times Y$ then $f_v(x,y) = (v(x) -y)^2$ could be a loss function for regression or classification with some classifier $v:X\to Y$. For a  distribution $\mu\in\probm(\Omega)$ we replace the objective in the classical risk minimisation $\min_{v}\E_\mu[f_v]$ with the \gls{robust Bayes risk}:
\begin{gather}
   \sup_{\nu\in \costb_c(\mu,r)} \E_\nu[f] \tag{rB}\label{eq:robust_bayes_minimisation}
\end{gather}\noeqref{eq:robust_bayes_minimisation}%
where $\costb_c(\mu,r)\subseteq\probm(\Omega)$ is a set containing $\mu$, called the \gls{uncertainty set} \citetext{\citealp[viz.][]{Berger_StatisticalDecision_1993,Vidakovic_GMinimaxParadigm_2000}, \citealp[\S4]{Grunwald_GameTheory_2004}}. It is in this way that we introduce distrust into the classical risk minimisation, by instead minimising the worst case risk over a set of distributions. 

It is sometimes the case that for an uncertainty set,  $\costb_c(\mu,r)\subseteq\probm(\Omega)$, there is a function, $r\lips_c:\Rx^{\Omega}\to\Rx$ (not necessarily the usual Lipschitz constant), so that 
\ificml
\begin{gather}
    \boxed{\sup_{\nu\in \costb_c(\mu,r)} \E_\nu[f]=\vphantom{\int}\E_\mu[f]  + r\lips_c(f).}  \tag{L}\label{eq:lip_objective}
\end{gather}
\else
\begin{gather}
    \sup_{\nu\in \costb_c(\mu,r)} \E_\nu[f]=\vphantom{\int}\E_\mu[f]  + r\lips_c(f).  \tag{L}\label{eq:lip_objective}
\end{gather}
\fi

There are two reasons we are interested in finding a relationship of the form \eqref{eq:lip_objective}. Firstly, there has been independent interest in the regularised risk, particularly when $\lips_c$ corresponds to the Lipschitz constant of $f$. The applications for Lipschitz regularisation are as disparate as generative adversarial networks \citep{Miyato_SpectralNormalization_2018,Arjovsky_WassersteinGenerative_2017}, generalisation \citep{Yoshida_SpectralNorm_2017,Farnia_GeneralizableAdversarial_2019,Gouk_RegularisationNeural_2018}, and adversarial learning \citep{Cisse_ParsevalNetworks_2017,Anil_SortingOut_2019,Cranko_MongeBlunts_2019,Tsuzuku_LipschitzmarginTraining_2018} among others \cites{Gouk_MaxGainRegularisation_2019,Scaman_LipschitzRegularity_2018}. Secondly, building a model that is robust to a particular uncertainty set is very intuitive and tractable. However, the left hand side of \eqref{eq:lip_objective} involves an optimisation over a subset of an infinite dimensional space.\footnote{Except for when the uncertainty set is chosen in a particularly trivial way.} By comparison, the Lipschitz regularised risk is often much easier to work with in practice. For these reasons then it is always interesting to note when a robust Bayes problem \eqref{eq:robust_bayes_minimisation} admits an equivalent formulation \eqref{eq:lip_objective}. Conversely, by developing such a connection we are able to provide interpretation to the popular Lipschitz regularised objective function. 

Our first major contribution, in \autoref{sec:distributional_robustness}, is to show that for a set of convex loss functions we have a result of the form \eqref{eq:lip_objective}. Furthermore, when the loss functions are nonconvex, \eqref{eq:lip_objective} becomes an inequality, and we prove the slackness is controlled (tightly) by a tractable measure of the loss function's convexity, which to our knowledge is a completely new result. As application, in \autoref{sec:adversarial_learning}, we show that the adversarial learning objective commonly used is, in fact, a special case of a distributionally robust risk, which significantly generalises other similar results in this area.

In practice, however, the evaluation of Lipschitz constant is NP-hard for neural networks \citep{Scaman_LipschitzRegularity_2018}, 
compelling approximations of it, or the explicit engineering of Lipschitz layers and analysing the resulting expressiveness in specific cases \citep[e.g.][$\infty$-norm]{Anil_SortingOut_2019}. By comparison, kernel machines encompass a family of models that is universal \citep{MicXuZha06}.
% for the \emph{given} model space (neural networks).
% We, instead, pursue a new path and study whether there exists a hypothesis space which: 
% \begin{enumerate*}[label=\textit{\alph*})]
%     \item is expressive enough, for example, encompassing all continuous functions;
%     \item allows the {exact} value of Lipschitz constant to be computed {efficiently};
%     \item enforcing the Lipschitz constant leads to a {convex} constraint that is amenable to efficient optimisation.
% \end{enumerate*}
% Interestingly, kernel machines satisfy all these requirements for some kernels.
%
% For example, the Gaussian kernels are universal, posessing a reproducing kernel Hilbert space (RKHS) which is dense in the family of continuous functions on a compact set \citep{MicXuZha06}.
% The RKHS of multi-layer inverse kernels compactly encompasses $1$-norm regularized neural networks \citep{ShaShaSri11}, degrading the generalisation performance by only a polynomial constant \citep{ZhaLeeJor16,ZhaLiaWai17}, with similar results conjectured for Gaussian kernels \citep{ShaShaSri11}.

Our third contribution, in \autoref{sec:kernel}, is to show that product kernels, such as Gaussian kernels, have a Lipzchitz constant that can be efficiently approximated and optimised  with high probability. By using the Nystr\"{o}m approximation \citep{WilSee00b,DriMah05}. we show that an $\epsilon$ approximation error requires only $\mathrm{O}(1/\epsilon^2)$ samples. Such a sampling-based approach also leads to a single convex constraint, making it scalable to large sample sizes, even with an interior-point solver (\autoref{sec:experiment}). As our experiments show, this method achieves higher robustness than state of the art \citep{Cisse_ParsevalNetworks_2017,Anil_SortingOut_2019}.

\section{Preliminaries}

Let $\Rx \defeq [-\infty,\infty]$ and $\Rx+\defeq [0,\infty]$, with similar notations for the real numbers. Let $\brk{n}$ denote the set $\set{1,\dots,n}$ for $n\in\N$. Unless otherwise specified, $X,Y,\Omega$ are topological outcome spaces. Often $X$ will be used when there is some linear structure, compatible with the topology, so that $\Omega=X\times Y$ may be interpreted as the classical outcome space for classification problems \citep[cf.][]{Vapnik_NatureStatistical_2000}. A sequence in $X$ is a mapping $\N\to X$ and is denoted $(x_i)_{i\in\N}\subseteq X$.

The Dirac measure at some point $\omega\in\Omega$ is $\dirac_\omega \in\probm(\Omega)$, and the set of Borel mappings $X\to Y$ is $\measf(X,Y)$. For $\mu\in\probm(\Omega)$, denote by $\leb_p(\Omega,\mu)$ the Lebesgue space of functions $f\in\measf(\Omega,\R)$ satisfying $\rbr{\int \abs{f(\omega)}^p\mu(\dv \omega)}^{1/p}<\infty$ for $p≥1$. The continuous real functions on $\Omega$ are collected in $\contf(\Omega)$.  In many of our subsequent formulas it is more convenient to write an expectation directly as an integral: $\E_\mu[f] = \int f\dv\mu \defeq \int f(\omega)\mu(\d\omega)$.

For two measures $\mu,\nu\in\probm(\Omega)$ the set of $(\mu,\nu)$-\emph{couplings} is $\coup(\mu,\nu)\subseteq\probm(\Omega\times\Omega)$ where $\pi\in \coup(\mu,\nu)$ if and only if the marginals of $\pi$ are $\mu$ and $\nu$:
\begin{gather}
    \mu = \int \pi(\marg, \dv\omega),\quad \nu = \int \pi(\dv\omega, \marg).
\end{gather}
For a  \gls{coupling function} $c:\Omega\times \Omega\to \Rx$, the $c$-\gls{transportation cost} of $\mu,\nu\in\probm(\Omega)$ is 
\begin{gather}
    \cost_c(\mu,\nu) \defeq \inf_{\pi \in\coup(\mu,\nu)} \int c\dv\pi. \label{eq:transportation_cost}
\end{gather}
The $c$-\gls{transportation cost} ball of radius $r≥0$ centred at $\mu\in\probm(\Omega)$ is 
\begin{gather}
    \costb_c(\mu, r) \defeq \setcond{ \nu \in \probm(\Omega) }{ \cost_c(\mu,\nu) ≤ r}, \label{eq:transportation_cost_ball}
\end{gather}
and serves as our \gls{uncertainty set}. Define the \gls{least c-Lipschitz constant} \cite[cf.][]{Cranko_MongeBlunts_2019} of a function $f:X\to\Rx$:
\begin{gather}
    \lips_c(f) \defeq \inf_{x,y\in X : c(x,y)≠0}\frac{f(x) - f(y)}{c(x,y)}.\label{eq:c_lipschitz}
\end{gather}
Thus when $(X,d)$ is a metric space $\lips_d(f)$ agrees with the usual Lipschitz notion. When $c:X\to \Rx$, for example when $c$ is a semi-norm, we take $c(x,y) \defeq c(x-y)$ for all $x,y\in X$. 

To a function $f:X\to \Rx$ we associate another function $\clconv f:X\to\Rx$, called the \gls{convex envelope} of $f$, defined to be the greatest closed convex function that minorises $f$. The quantity $\cvxity(f) \defeq \sup_{x\in X}\rbr{ f(x) - \clconv f(x)}$ was first suggested by \textcite{Aubin_EstimatesDuality_1976} to quantify the lack of convexity of a function $f$, and has since shown to be of considerable interest for, among other things, bounding the duality gap in nonconvex optimisation \cite[cf.][]{Lemarechal_GeometricStudy_2001,Udell_BoundingDuality_2016,Askari_NaiveFeature_2019,Kerdreux_ApproximateShapleyFolkman_2019}. In particular, observe
\begin{gather}
    \cvxity(f) = 0 \iff f = \clconv f \iff \text{$f$ is closed convex}.
\end{gather}
 
When $f:\R^n\to \Rx$ is minorised by an affine function, there is \citetext{\citealp[cf.][Prop.~X.1.5.4]{Hiriart-Urruty_ConvexAnalysis_2010}; \citealp{Benoist_WhatSubdifferential_1996}}
\begin{gather}
    \clconv f(x) = \inf_{\substack{(\alpha_1,\dots, \alpha_{n+1})\in\Delta^{n+1}\\(x_1,\dots,x_{n+1})\in \Sigma^{n}(x)}} 
    \sum_{i\in\brk{n+1}}\alpha_i f(x_i)
\end{gather}
for all $x\in \R^n$, where
\begin{gather}
    \Delta^n \defeq \setcond3{(\alpha_1,\dots, \alpha_{n+1})\in\R+^{n}}{\sum_{i\in\brk{n+1}}\alpha_i=1}\\
    \mathclap{\Sigma^n(x) \defeq \setcond3{(x_1,\dots, x_{n+1}) \in (\R^n)^{n+1}}{\sum_{i\in\brk{n+1}} x_i=x}.}
\end{gather}
Consequentially it is well known that $\cvxity(f)$ can be computed via the finite-dimensional maximisation
\begin{gather}
    \mathclap{\sup_{\substack{(\alpha_1,\dots, \alpha_{n+1})\in\Delta^{n+1}\\(x_1,\dots,x_{n+1})\in (\R^n)^{n+1}}} 
    \rbr{f\rbr{\smashoperator[r]{\sum_{i\in\brk{n+1}}}\alpha_ix_i} - \smashoperator{\sum_{i\in\brk{n+1}}}\alpha_i f(x_i)}.}
\end{gather}

Complete proofs of all technical results are relegated to the supplementary material.

\section{Distributional robustness}\label{sec:distributional_robustness}

In this section we present our major result regarding identities of the form \eqref{eq:lip_objective}.

\begin{toappendix}
    \subsection{Proof of \autoref{thm:robustness_bound} and other technical results}
\end{toappendix}

\begin{lemmaapx}[\protect{\citealp[Thm.~1]{Blanchet_QuantifyingDistributional_2019}}]\label{lem:strong_duality}
    Assume $\Omega$ is a Polish space and fix $\mu\in\probm(\Omega)$. Let $c:\Omega\times\Omega\to\Rx+$ be lower semicontinuous with $c(\omega,\omega)=0$ for all $\omega\in \Omega$, and $f:\Omega\to\R$ is upper semicontinuous. Then for all $ r≥0$ there is 
    \begin{gather}
        \sup_{\nu\in\costb_c(\mu, r)}\int f \dv\nu = \inf_{\lambda≥0}\rbr{\lambda r + \int f^{\lambda c}\dv\mu}.\label{eq:duality}
    \end{gather}
\end{lemmaapx}
\begin{toappendix}
    \par Duality results like \autoref{lem:strong_duality} have been the basis of a number of recent theoretical efforts in the theory of adversarial learning \cite{Sinha_CertifiableDistributional_2018,Gao_DistributionallyRobust_2016,Blanchet_RobustWasserstein_2019,Shafieezadeh-Abadeh_RegularizationMass_2019}, the results of \citet{Blanchet_QuantifyingDistributional_2019} being the most general to date. The necessity for such duality results like \autoref{lem:strong_duality} is because while the supremum on the left hand side of \eqref{eq:duality} is over a (usually) infinite dimensional space, the right hand side only involves only a finite dimensional optimisation. The generalised conjugate in \eqref{eq:duality} also hides an optimisation, but when the outcome space $\Omega$ is finite dimensional, this too is a finite dimensional problem. 

    The following lemma is sometimes stated a consequence of, or in the proof of, the McShane--Whitney extension theorem \cite{McShane_ExtensionRange_1934,Whitney_AnalyticExtensions_1934}, but it is immediate to observe. 
\end{toappendix}
\begin{lemmaapx}\label{lem:mcshane_whitney}
    Let $X$ be a set. Assume $c:X\times X\to \Rx+$ satisfies $c(x,x)=0$ for all $x\in X$, $f:X\to \R$. Then 
    \begin{gather}
        1 ≥ \lips_c(f)\iff\forall{y\in X}  f(y)= \sup_{x\in X}\rbr\big{f(x) - c(x,y)}.
    \end{gather}
\end{lemmaapx}
\begin{proof}
    Suppose $1 ≥ \lips_c(f)$. Fix $y_0\in X$. Then 
    \begin{gather}
        \forall{x\in X} f(x) - c(x,y_0)  ≤ f(y_0),
    \end{gather}
    with equality when $x=y_0$. Next suppose 
    \begin{gather}
        \forall{y\in X} f(y)= \sup_{x\in X}\rbr\big{f(x) - c(x,y)},
    \end{gather}
    then 
    \begin{align}
        \forall{x,y\in X} f(y) ≥ f(x) - c(x,y) 
        &\iff  \forall{x,y\in X} f(x) - f(y) ≤ c(x,y)
        \\&\iff  1≥\lips_c(f),
    \end{align}
    as claimed.
\end{proof}

\begin{lemmaapx}\label{lem:lipschitz_conjugate_condition}
    Assume $X$ is a vector space.
    Suppose $c:X\to \Rx+$ satisfies $c(0)=0$, and $f:X\to\R$ is convex. Then 
    \begin{gather}
       1\geq \lips_c(f) \iff \forall{\epsilon≥0} \esubdiff f(X) \subseteq \esubdiff c(0).
    \end{gather} 
\end{lemmaapx}
\begin{proof}
    Suppose $1\geq \lips_c(f)$. Then $f(x) - f(y) ≤  c(x -y)$ for all $x,y\in X$.
    Fix $\epsilon≥0$, $x\in X$ and suppose $x^*\in \esubdiff f(x)$. Then 
    \begin{align}
        \MoveEqLeft\forall{y\in X}\inp{y-x,x^*} - \epsilon ≤ f(y) - f(x) ≤  c(y-x)
        \\&\iff\forall{y\in X} \inp{y,x^*} -\epsilon ≤ f(y+x) - f(x) ≤  c(y) - c(0),
    \end{align}
    because $c(0)=0$. This shows $x^*\in \esubdiff  c(0)$. 
    
    Next assume $\esubdiff f(x) \subseteq \esubdiff c(0)$ for all $\epsilon≥0$ and $x\in X$. Because $f$ is not extended-real valued, it is continuous on all of $X$ \cite[via][Cor.~2.2.10]{Zalinescu_ConvexAnalysis_2002} and $\subdiff f(x)$ is nonempty for all $x\in X$ \cite[via][Thm.~2.4.9]{Zalinescu_ConvexAnalysis_2002}.  Fix an arbitrary $x\in X$. Then $\emptyset≠\subdiff f(x) \subseteq \subdiff c(0)$, and 
    \begin{align}
        \begin{aligned}
            \MoveEqLeft\exists{x^*\in\subdiff f(x)}\forall{y\in X}
        f(x)-f(y) ≤\inp{x-y,x^*}
        \\&\implies 
        \forall{y\in X}f(x)-f(y) ≤\inp{x-y,x^*}  ≤  c(x-y),
        \end{aligned} \label{eq:lipschitz_ineq}
    \end{align}
    where the implication is because $x^*\in\subdiff  c(0)$ and $c(0)=0$. Since the choice of $x$ in \eqref{eq:lipschitz_ineq} was arbitrary, the proof is complete.
\end{proof}

\begin{lemmaapx}\label{lem:lipschitz_conjugate}
    Assume $X$ is a locally convex Hausdorff topological vector space.
    Suppose $c:X\to \Rx$ is closed sublinear, and $f:X\to\R$ is closed convex. Then there is 
    \begin{gather}
        \forall{y\in X} \sup_{x\in X}\rbr2{ f(x) -  c(x-y)} = \begin{cases}
            f(y)  &1≥\lips_c(f) \\
            \infty &\text{otherwise}.
        \end{cases}
    \end{gather} 
\end{lemmaapx}
\begin{proof} 
    Fix an arbitrary $y_0\in X$. From \autoref{lem:lipschitz_conjugate_condition} we know 
    \begin{gather}
        1\geq \lips_c(f) \iff \forall{\epsilon≥0} \esubdiff f(X) \subseteq \esubdiff c(0).
     \end{gather}

    Assume $\esubdiff f(X) \subseteq \esubdiff c(0)$ for all $\epsilon≥0$.
     Consequentially $\esubdiff f(y_0)\subseteq\esubdiff c(0) = \esubdiff c(\marg-y_0)(y_0)$ for every $\epsilon≥0$.
    From the usual difference-convex global $\epsilon$-subdifferential condition \cite[Thm.~4.4]{Hiriart-Urruty_ConvexOptimization_1989}  it follows that
    \begin{gather}
        \inf_{x\in X} \rbr3{c(x-y_0) - f(x)}  = \underbrace{c(y_0 - y_0)}_0 - f(y_0) = -f(y_0),
    \end{gather}
    where we note that $c(y_0 - y_0) = c(0) = 0$ because $c$ is sublinear.  

    Assume $\esubdiff f(X) \not\subseteq \esubdiff c(0)$ for some $\epsilon≥0$. By hypothesis there exists $\epsilon_0≥0$, $x_0\in X$, and $x^*_0\in X^*$ with 
    \begin{gather}
        x_0^*\in \esubdiff_{\epsilon_0} f(x_0)\quad\text{and}\quad x_0^*\not\in \esubdiff_{\epsilon_0} c(0).
    \end{gather}
    
    Using the \textcite{Toland_DualityPrinciple_1979} duality formula \cite[viz.][Cor.~2.3]{Hiriart-Urruty_GeneralFormula_1986} and the usual calculus rules for the Fenchel conjugate \cite[e.g.][Thm.~2.3.1]{Zalinescu_ConvexAnalysis_2002} we have 
    \begin{align}
        \inf_{x\in X} \rbr3{c(x-y_0) - f(x)} 
        &= \inf_{x^*\in X^*} \rbr3{f^*(x^*) - (c(\marg - y_0))^*(x^*)}
        \\&= \inf_{x^*\in X^*} \rbr3{f^*(x^*) - c^*(x^*) + \inp{y_0,x^*}}
        \\&≤f^*(x^*_0) - c^*(x^*_0) + \inp{y_0,x^*_0}
        \\&\overseteqref{eq:generalised_fy}{≤} \epsilon_0 + \inp{x_0,x_0^*} - f(x_0) - c^*(x^*_0) + \inp{y_0,x^*_0}
        \\&= \underbrace{\epsilon_0 + \inp{x_0+y_0,x_0^*} - f(x_0)}_{<\infty} - c^*(x^*_0),\label{eq:toland_inequality}
    \end{align}
    where the second inequality is because $x_0^*\in \esubdiff_{\epsilon_0} f(x_0)$. 
    
    We have assumed $x^*_0\notin \esubdiff c(0) \supseteq \subdiff c(0)$. Because $c$ is sublinear,  $c^* = \ind_{\subdiff c(0)}$ \cite[Thm.~2.4.14 (i)]{Zalinescu_ConvexAnalysis_2002}, and therefore $c^*(x_0^*) = \infty$. Then \eqref{eq:toland_inequality} yields
    \begin{gather}
        \inf_{x\in X} \rbr3{c(x-y_0) - f(x)}  ≤ -\infty,
    \end{gather}
    which completes the proof.
\end{proof}

\NewDocumentCommand\boundnumber{G{\mu}}{\Delta_{f,c,r}(#1)}
\begin{theoremrep}\label{thm:robustness_bound}
    Assume $X$ is a separable Fréchet space and fix $\mu\in\probm(X)$. Suppose $c:X\to \Rx$ is closed sublinear, and $f\in\leb_1(X,\mu)$ is upper semicontinuous with $\lips_c(f)<\infty$. Then for all $r≥0$, there is a number $\boundnumber≥0$  so that 
    \begin{gather}
        \mathclap{\sup_{\nu\in\costb_c(\mu, r)}\int f\dv\nu + \boundnumber = \int f\dv\mu+ r\lips_c(f).\quad}\label{eq:upperbound}
    \end{gather}
    Furthermore $\boundnumber$ is upper bounded by 
    \begin{align}
        r\lips_c(f) - \sbr2{r\lips_c(\clconv f)- \int\rbr{f-\clconv f}\dv\mu}_+,\label{eq:bound_tightness}
    \end{align}
    where $\sbr{\marg}_+ \defeq \max\cbr{\marg,0}$, so that when $f$ is closed convex $\boundnumber=0$.
\end{theoremrep}
\begin{proof}
    \eqref{eq:upperbound}: Since $c$ is assumed sublinear, it is positively homogeneous and there is $c(x,x) = c(x-x) = c(0) = 0$ for all $x\in X$. %
    % %
    Therefore we can apply \autoref{lem:strong_duality} and  \autoref{lem:mcshane_whitney} to obtain
    \begin{gather}
        \begin{aligned}
            \sup_{\nu\in\costb_c(\mu, r)} \int f\dv\nu 
            &\oversetautoref{lem:strong_duality}{=} \inf_{\lambda≥0}\rbr{  r\lambda + \int f^{\lambda c} \dv\mu} \label{eq:transport_dual}
            \\&≤ \inf_{\lambda≥\lips_c(f)}\rbr{  r\lambda + \int f^{\lambda c} \dv\mu} 
            \\&\oversetautoref{lem:mcshane_whitney}{=}  r\lips_c(f) + \int f\dv\mu,
        \end{aligned}
    \end{gather}
    and therefore $\boundnumber≥0$.

    \eqref{eq:bound_tightness}: Observing that $\clconv f≤f$, from \autoref{lem:lipschitz_conjugate} we find for all $x\in X$
    \begin{align}
        \MoveEqLeft\sup_{\lambda \in [0,\infty)}\rbr\big{f(x) - f^{\lambda c}(x) - r\lambda}
        \\&\oversetautoref*{lem:lipschitz_conjugate}{=}
        \sup_{\lambda \in [0,\infty)}\rbr\big{f(x) - \sup_{y\in X}\rbr\big{f(y) - \lambda c(x-y)}- r\lambda}
        \\&\oversetautoref*{lem:lipschitz_conjugate}{=}
        \sup_{\lambda \in [0,\infty)}\inf_{y\in X}\rbr\big{f(x) -f(y) + \lambda c(x-y) - r\lambda}
        \\&\oversetautoref*{lem:lipschitz_conjugate}{≤}
        \sup_{\lambda \in [0,\infty)}\inf_{y\in X}\rbr\big{f(x) -\clconv f(y) + \lambda c(x-y) -\lambda r}
        \\&\oversetautoref{lem:lipschitz_conjugate}{=}
            \sup_{\lambda \in [0,\infty)}\begin{cases}
                f(x)  - \clconv f(x) -\lambda r & \lips_c(\clconv f) ≤ \lambda \\
                -\infty& \lips_c(\clconv f) > \lambda 
            \end{cases}
        \\&
        \oversetautoref*{lem:lipschitz_conjugate}{=}
        f(x)  - \clconv f(x) - r\lips_c(\clconv f).\label{eq:intermediary_bound_1}
    \end{align}
    Similarly,  for all $x\in X$ there is 
    \begin{align}
        \sup_{\lambda \in [0,\infty)}\rbr2{f(x) - f^{\lambda c}(x) - r\lambda}
        &≤\sup_{\lambda \in [0,\infty)}\rbr2{f(x) - f^{\lambda c}(x)}  + \sup_{\lambda \in [0,\infty)} \rbr2{- r\lambda}
        \\&=\sup_{\lambda \in [0,\infty)}\rbr2{f(x) - f^{\lambda c}(x)}
        \\&=\sup_{\lambda \in [0,\infty)}\inf_{y\in X}\rbr2{f(x) -  f(y) + \lambda c(x-y)}
        \\&≤\inf_{y\in X}\sup_{\lambda \in [0,\infty)}\rbr2{f(x) -  f(y) + \lambda c(x-y)}
        \\&=\inf_{y\in X}\begin{cases}
            \infty  & c(x-y)>0\\
            0  & c(x-y)=0
        \end{cases}
        \\&=0.\label{eq:intermediary_bound_2}
        % \\&≤\norm{f - \clconv f}_\infty -  r\lips_c(\clconv f) .
    \end{align}

    Together, \eqref{eq:intermediary_bound_1} and \eqref{eq:intermediary_bound_2} show 
    \begin{align}
        \MoveEqLeft[8]\int \sup_{\lambda \in [0,\infty)}\rbr\big{f - f^{\lambda c} - r\lambda}\d \mu
        \\&≤ \min\cbr{
            \int (f-\clconv f)\d\mu -  r\lips_c(\clconv f), 0
        }.\label{eq:intermediary_bound_3}
    \end{align}
    Then
    \begin{align}
        \boundnumber
        &=\rbr{ r\lips_c(f) + \int f\d\mu} - \sup_{\nu\in\costb_c(\mu, r)} \int f\d\nu 
        \\&\overseteqref{eq:transport_dual}{=}\rbr{ r\lips_c(f) + \int f\d\mu} -\inf_{\lambda \in [0,\infty)}\rbr{ r\lambda - \int f^{\lambda c} \d\mu}
        \\&\overseteqref*{eq:intermediary_bound_3}{=} r\lips_c(f)+\sup_{\lambda \in [0,\infty)}\int\rbr{f - f^{\lambda c} -\lambda r }\d\mu
        \\&\overseteqref*{eq:intermediary_bound_3}{≤} r\lips_c(f)+\int\sup_{\lambda \in [0,\infty)}\rbr{f - f^{\lambda c} -\lambda r }\d\mu
        \\&\overseteqref{eq:intermediary_bound_3}{≤}
        r\lips_c(f) + \min\cbr{\int\rbr{f-\clconv f}\d\mu -  r\lips_c(\clconv f),0},
    \end{align}
    which implies \eqref{eq:bound_tightness}.
\end{proof}
\begin{proofsketch}
    The duality result of \citet[Thm.~1]{Blanchet_QuantifyingDistributional_2019} yields a tractable, dual formulation of the robust risk, which is easy to upper bound by the regularised risk. Lower bound the function $f$ by its closed convex envelope $\clconv f$ and use classical results from the difference-convex optimisation literature \citep{Toland_DualityPrinciple_1979,Hiriart-Urruty_GeneralFormula_1986,Hiriart-Urruty_ConvexOptimization_1989} to solve the inner maximisation of the  dual robust risk formulation.
 \end{proofsketch}

\autoref{thm:robustness_bound} subsumes many existing results 
\citetext{\citealp[Cor.~2~(iv)]{Gao_DistributionallyRobust_2016}; \citealp[\S3.2]{Cisse_ParsevalNetworks_2017}; \citealp[\S3.2][various]{Sinha_CertifiableDistributional_2018}; \citealp[Thm.~14]{Shafieezadeh-Abadeh_RegularizationMass_2019}} with a great deal more generality, applying to a very broad family of models, loss functions, and outcome spaces. The extension of \autoref{thm:robustness_bound} for robust classification in the absence of label noise is straight-forward:
\begin{corollary}\label{cor:robustness_bound}
    Assume $X$ is a separable Fréchet space and $Y$ is a topological space. Fix $\mu\in\probm(X\times Y)$. Assume $c:(X\times Y)\times (X\times Y)\to\Rx$ satisfies 
    \begin{gather}
        c((x,y),(x',y')) = \begin{cases}
            c_0(x-x') & y=y'\\
            \infty    & y≠y',
        \end{cases}\label{eq:extended_c}
    \end{gather}
    where $c:X\to \Rx$ is closed sublinear, and $f\in\leb_1(X\times Y,\mu)$ is upper semicontinuous and has $\lips_c(f)<\infty$. Then for all $ r≥0$ there is \eqref{eq:upperbound} and \eqref{eq:bound_tightness}, where the closed convex hull is interpreted $\clconv(f)(x,y) \defeq \clconv(f(\marg,y))(x)$.
\end{corollary}
It is the first time to our knowledge that the slackness in \eqref{eq:bound_tightness} has been characterised tightly.  Clearly from \autoref{thm:robustness_bound} the upper bound \eqref{eq:boring_upperbound} is tight for closed convex functions, but \autoref{prop:robust_tightness} shows it is also tight for a large family of nonconvex functions and measures --- particularly the upper semi-continuous loss functions on a compact set, with the collection of probability distributions supported on that set.

Observing that $\boundnumber≥0$, the equality \eqref{eq:upperbound} yields the upper bound 
\begin{gather}
    \sup_{\nu\in\costb_c(\mu, r)}\int f\d\nu  ≤ \int f\d\mu+ r\lips_c(f).\label{eq:boring_upperbound}
\end{gather}
By controlling $\boundnumber$ we are able to guarantee that the regularised risk in \eqref{eq:lip_objective} is a good surrogate for the robust risk. The number $\boundnumber$ itself is quite hard to measure (since it would require computing the robust risk directly), which is why we upper bound it in \eqref{eq:bound_tightness}. \autoref{prop:robust_tightness} shows the slackness bound \eqref{eq:bound_tightness} is tight for a large family of distributions after observing 
\begin{gather}
    \forall{f\in\measf(X,\Rx)}\forall{\mu\in\probm(X)}\int\rbr{f-\clconv f}\d\mu ≤ \cvxity(f).
\end{gather}
This yields 
\begin{align}
    \MoveEqLeft[6] r\lips_c(f) - \sbr2{r\lips_c(\clconv f)- \int\rbr{f-\clconv f}\d\mu}_+ 
    \\&≤ r\lips_c(f) - \sbr2{r\lips_c(\clconv f)- \cvxity(f)}_+,
\end{align}
for all $f\in\measf(X,\Rx)$, $\mu\in\probm(X)$, and $r≥0$.

\begin{propositionrep}\label{prop:robust_tightness}
    Let $X$ be a separable Fréchet space with $X_0\subseteq X$. Suppose $c:X\to \Rx$ is closed sublinear, and $f\in\bigcap_{\mu\in\probm(X_0)}\leb_1(X,\mu)$ is upper semicontinuous, has $\lips_c(f)<\infty$, and attains its maximum on $X_0$. Then for all $r≥0$
    \begin{gather}
        \mathclap{
            \sup_{\mu\in\probm(X_0)}\boundnumber  = r\lips_c(f) - \sbr2{r\lips_c(\clconv f)- \cvxity(f)}_+. 
        }\label{eq:slackness_bound_tightness}
    \end{gather}
\end{propositionrep}
\begin{proofsketch}
    Let $f$ achieve its maximum on $X_0$ at $x_0$. Then $\sup_{\nu\in\probm(X_0)}\E_\nu[f]=\sup_{\nu\in\costb_c(\dirac_{x_0},r)}\E_\nu[f] = \E_{\dirac_{x_0}}[f]$, which implies the result. 
\end{proofsketch}
\begin{appendixproof}
    Let $x_0\in X_0$ be a point at which $f(x_0) = \sup f(X_0)$. Then $\cost_c(\dirac_{x_0},\dirac_{x_0}) = 0 ≤ r$, and $\sup_{\nu\in\ball_c(\dirac_{x_0}, r)}\int f\d\nu = f(x_0)$. Therefore 
    \begin{align}
        \boundnumber{\dirac_{x_0}}
        % = r\lips_c(f) + \int f \d\dirac_{x_0} - \sup_{\nu\in\ball_c(\dirac_{x_0}, r)}\int f\d\dirac_{x_0}
        = r\lips_c(f) + f(x_0) - f(x_0)
        = r\lips_c(f).\label{eq:lower_bound_thing}
    \end{align}
    And so we have 
    \begin{align}
        r\lips_c(f) 
        &\overseteqref{eq:lower_bound_thing}{≤} \sup_{\mu\in\probm(X_0)} \boundnumber
        \\&\oversetautoref{thm:robustness_bound}{≤}  r\lips_c(f) - \max\mathopen{}\cbr2{r\lips_c(\clconv f)- \cvxity(f),0}
        \\&\oversetautoref*{thm:robustness_bound}{≤} r\lips_c(f),
    \end{align}
    which implies the claim.
\end{appendixproof}

\begin{remark}
    In particular, for any compact subset of a Fréchet space $X_0$ (such as the set of $n$-dimensional images, $X_0=[0,1]^n\subseteq\R^n$) the bound \eqref{eq:upperbound} is tight with respect to the set $\probm(X_0)$ for any upper semicontinuous $f\in \bigcap_{\mu\in\probm(X_0)}\leb(X,\mu)$. Since the behaviour of $f$ away from $X_0$ is not important, the $c$-Lipschitz constant in \eqref{eq:upperbound} need only be computed here. To do so one may replace $c$ with $\tilde c$, where $ \tilde c(x) = c(x)$ for $x\in X_0$ and $\tilde c(x) = \infty$ for $x \in X\setminus X_0$, and observe $\lips_{\tilde c}(f) ≤ \lips_c(f)$, because $\tilde c ≥ c$.
\end{remark}

\section{Adversarial learning}\label{sec:adversarial_learning}

\textcite{Szegedy_IntriguingProperties_2014} observe that deep neural networks, trained for image classification using empirical risk minimisation, exhibit a curious behaviour whereby an image, $x\in\R^n$, and a small, imperceptible amount of noise, $\delta_x\in\R^n$, may found so that the network classifies $x$ and $x+\delta_x$ differently. Imagining that the troublesome noise vector is sought by an adversary seeking to defeat the classifier, such pairs have come to be known as \emph{adversarial examples} \cite{MoosaviDezfooli_UniversalAdversarial_2017,Goodfellow_ExplainingHarnessing_2015,Kurakin_AdversarialExamples_2017}.

When $(X,c)$ is a normed space, the closed ball of radius $r≥0$, centred at $x\in X$ is denoted $\ball_c(x,r) \defeq \setcond{ y\in X }{ c(x-y) ≤ r}$. Let $X$ be a linear space and $Y$ a topological space. Fix $\mu\in\probm(X\times Y)$, $r≥0$, and let $c$ be a norm on $X$. 

The following objective has been proposed \cite[viz.][]{Madry_DeepLearning_2018,Shaham_UnderstandingAdversarial_2018,Carlini_EvaluatingRobustness_2017,Cisse_ParsevalNetworks_2017} as a means of learning classifiers that are robust to adversarial examples
\begin{gather}
    \begin{aligned}
        \int \sup_{\smash{\delta\in \ball_c(0,r) }} f(x + \delta, y) \mu(\d x\times \d y ) 
    \end{aligned}\label{eq:adv_risk}
\end{gather}
where $f:X\times Y\to \Rx$ is the loss of some classifier.

\begin{toappendix}
    \subsection{Proof of \autoref{thm:adversarial_risk_is_robust_risk}}
    \autoref{lem:robust_risk_equality} will be used to show an equality result in \autoref{thm:adversarial_risk_is_robust_risk}.
\end{toappendix}
\begin{lemmaapx}\label{lem:robust_risk_equality}
    Assume $(\Omega,c)$ is a compact Polish space and $\mu\in\probm(\Omega)$ is non-atomic. For $r>0$ and $\nu^\star\in\costb_c(\mu,r)$  there is a sequence $(f_i)_{i\in \N}\subseteq A_\mu(r)\defeq \setcond{f\in\measf(\Omega,\Omega) }{ \int c\dv(\id,f)_\#\mu ≤ r}$ with $(f_i)_\#\mu$ converging at $\nu^\star$ in $\wtop{\probm(\Omega),\contf(\Omega)}$.
\end{lemmaapx}

\begin{proof}
    Let $P(\mu,\nu) \defeq \setcond{f\in\measf(X,X)}{ f_\#\mu = \nu}$. Since $\mu$ is non-atomic and $c$ is continuous we have \cite[via][Thm.~B]{Pratelli_EqualityMonge_2007}
    \begin{gather}
        \forall{\nu\in\probm(\Omega)} \inf_{f \in P(\mu,\nu)} \int c\dv(\id,f)_\#\mu = \cost_c(\mu,\nu).
    \end{gather}
    Let $r^\star\defeq \cost_c(\mu,\nu^\star)$, obviously $r^\star ≤ r$. Assume $r^\star>0$, otherwise the lemma is trivial. Fix a sequence $(\epsilon_k)_{k\in \N}\subseteq (0,r^\star)$ with $\epsilon_k\to 0$. For $u≥0$ let  $\nu(u) \defeq \mu + u(\nu^\star- \mu)$. Then 
    \begin{gather}
        \cost_c(\mu,\nu(0)) = 0 \quad\text{and}\quad \cost_c(\mu,\nu(1)) = r^\star,
    \end{gather}
    and because $\cost_c$ metrises the $\wtop{\probm(\Omega),\contf(\Omega)}$-topology on $\probm(\Omega)$ \cite[Cor.~6.13]{Villani_OptimalTransport_2009}, the mapping $u\mapsto \cost_c(\mu,\nu(u))$ is $\wtop{\probm(\Omega),\contf(\Omega)}$-continuous. 
    Then by the intermediate value theorem for every $k\in\N$ there is some $u_k>0$ with $\cost_c(\mu,\nu(u_k)) = r^\star - \epsilon_k$, forming a sequence $(u_k)_{k\in \N}\subseteq [0,1]$. Then for every $k$ there is a sequence $(f_{jk})_{j\in \N}\subseteq P(\mu,\nu(u_k))$ so that $(f_{jk})_\#\mu \to \nu(k)$ in $\wtop{\probm(\Omega),\contf(\Omega)}$ and 
    \begin{align}
        \lim_{j\in \N} \int c\dv(\id, f_{jk})_\#\mu 
        &= \inf_{f \in P(\mu,\nu(k))} \int c\dv(\id,f_k)_\#\mu 
        \\&= \cost_c(\mu,\nu(k)) 
        \\&= r^\star -\epsilon_k.
    \end{align}
    Therefore for every $k\in\N$ there exists $j_k≥0$ so that for every $j≥j_k$
    \begin{gather}
       \int c\dv(\id, f_{jk})_\#\mu ≤ r^\star. \label{eq:feasibility}
    \end{gather}
    Let us pass directly to this subsequence of $(f_{jk})_{j\in \N}$ for every $k\in\N$ so that \eqref{eq:feasibility} holds for all $j,k\in\N$. Next by construction we have $\nu(u_k) \to \nu^\star$. Therefore $(f_{jk})_{j,k\in N}$ has a subsequence in $k$ so that $(f_{jk})_\#\mu \to \nu^\star$ in in $\wtop{\probm(\Omega),\contf(\Omega)}$. By ensuring \eqref{eq:feasibility} is satisfied, the sequences  $(f_{jk})_{j\in \N}\subseteq A_\mu(r)$ for every $k\in \N$.
\end{proof}

\begin{toappendix}
    We can now prove our main result \autoref{thm:adversarial_risk_is_robust_risk}.  When $(X,c)$ is a normed space, the closed ball of radius $r≥0$, centred at $x\in X$ is denoted $\ball_c(x,r) \defeq \setcond{ y\in X }{ c(x-y) ≤ r}$.
\end{toappendix}
% For $\mu\in\probm(\Omega)$  Then 
\begin{theoremrep}\label{thm:adversarial_risk_is_robust_risk}
    Assume $(X,c_0)$ is a separable Banach space. Fix $\mu\in\probm(X)$ and for $r≥0$ let  
    \begin{gather}
        R_\mu(r) \defeq \setcond{g\in\measf( X,\R+) }{ \int g\d\mu ≤ r}.
    % \end{gather}
        \intertext{Then for $f\in\measf(\Omega,\Rx)$ and $r≥0$ there is} 
    % \begin{gather}
        \mathclap{\sup_{g\in R_\mu(r)}\int \mu(\d\omega) \sup_{\omega'\in\ball_{c_0}(\omega,g(\omega))}f(\omega') 
        ≤ \sup_{\nu\in\costb_{c_0}(\mu,r)} \int f\d\nu,}
    \end{gather}
    with equality if, furthermore,  $\mu$ is non-atomically concentrated on a compact subset of $X$, on which $f$ is continuous with the subspace topology.
\end{theoremrep}

\begin{proof}
    For convenience of notation let $c\defeq c_0$.

    When $r=0$, the set $R_\mu(r)$ consists of the set of functions $g$ which are $0$ $\mu$-almost everywhere, in which case $\ball_c(x,g(x))=\cbr{0}$ for $\mu$-almost all $x\in X$. Thus \eqref{eq:inequality_thing} is equal to $\int f(x)  \mu(\d x)$. Since $c$ is a norm, $c(0)=0$, and by a similar argument there is equality with the right hand side. We now complete the proof for the cases where $r>0$.

    \emph{Inequality}:
    For $g\in R_\mu(r)$, let $\Gamma_g: X\to  2^X$ denote the set-valued mapping with $\Gamma_g( x)\defeq \ball_c( x,g( x))$. Let $\measf( X,\Gamma_g)$ denote the set of Borel $a: X\to X$ so that $a( x) \in \Gamma_g( x)$ for $\mu$-almost all $ x\in X$. Let $A_\mu(r) \defeq  \bigcup_{g\in \R_\mu(r)} \measf( X,\Gamma_g)$. Clearly for every $a\in A_\mu(r)$ there is 
    \begin{gather}
        r≥ \int c( x,a( x))\d\mu = \int c\d (\id, a)_\#\mu,
    \end{gather}
    which shows $\setcond{a_\#\mu }{ a\in A_\mu(r)} \subseteq \costb_c(\mu,r)$. Then if there is equality in \eqref{eq:integral_interchange}, we have 
    \begin{align}
        \sup_{g\in R_\mu(r)} \int \sup_{ x'\in \Gamma_g( x)} f( x) 
        &= \sup_{g\in R_\mu(r)}\sup_{a\in \measf( X,\Gamma_g)} \int f\d a_\#\mu\label{eq:integral_interchange}
        \\&= \sup_{a\in A_\mu(r)} \int f\d a_\#\mu
        \\&≤ \smash{\sup_{\nu\in\costb_c(\mu,r)}} \int f\d \nu,
    \end{align}
    which proves the inequality.

    To complete the proof we will now justify the exchange of integration and supremum in  \eqref{eq:integral_interchange}. The set $\measf( X,\Gamma_g)$ is trivially decomposable \cite[see the remark at the bottom of p.~323, Def.~2.1]{Giner_NecessarySufficient_2009}. By assumption $f$ is Borel measurable. Since $f$ is measurable, any decomposable subset of $\measf( X, X)$ is $f$-decomposable \cite[Prop.~5.3]{Giner_NecessarySufficient_2009} and $f$-linked \cite[Prop.~3.7~(i)]{Giner_NecessarySufficient_2009}.  \textcite[Thm.~6.1~(c)]{Giner_NecessarySufficient_2009} therefore allows us to exchange integration and supremum in \eqref{eq:integral_interchange}.

    \emph{Equality}:
    Under the additional assumptions there exists $\nu^\star\in\probm(\Omega)$ with \cite[via][Prop.~2]{Blanchet_QuantifyingDistributional_2019}
    \begin{gather}
        \smash\int f\d\nu^\star = \sup_{\nu\in\costb_c(\mu,r)}\smash\int f\d\nu.
    \end{gather}
    The compact subset where $\mu$ is concentrated and non-atomic is a Polish space with the Banach metric. Therefore
    using \autoref{lem:robust_risk_equality} there is a sequence $(f_i)_{i\in \N}\subseteq A_\mu(r)$ so that 
    \begin{gather}
        \lim_{i\in\N}\int f_i\d\mu = \int f\d\nu^\star = \sup_{\nu\in\costb_c(\mu,r)}\int f\d\nu,
    \end{gather}
    proving the desired equality.
\end{proof}

\begin{proofsketch}
    \textcite[Thm.~6.1~(c)]{Giner_NecessarySufficient_2009} allows us to interchange the integral and supremum, the inequality then follows from the definition of the transportation cost risk. To show the equality under the added assumptions, there is a distribution that achieves the robust supremum \cite[Prop.~2]{Blanchet_QuantifyingDistributional_2019}, and a Monge map that achieves the transportation cost infimum \cite[Thm.~B]{Pratelli_EqualityMonge_2007}. 
\end{proofsketch}

\begin{remark}
    By observing the constant function $ g_r\equiv r$ is included in the set $R_\mu(r)$, it's easy to see that the adversarial risk \eqref{eq:adv_risk} is upper bounded as follows
    \begin{align}
        \MoveEqLeft[5]\int \sup_{\smash{\delta\in \ball_{c_0}(0,r) }} f(x + \delta, y) \mu(\d x\times \d y )
        \\&=\int \sup_{\omega' \in \ball_{c}(\omega,r) } f(\omega') \mu(\d\omega),
        \\&≤\sup_{g\in R_\mu(r)}\int \mu(\dv\omega) \sup_{\omega'\in\ball_{\tilde c}(\omega,g(\omega))}f(\omega'),\quad\label{eq:inequality_thing}
    \end{align}
    where in the equality we extend $c_0$ to a metric on $X\times Y$ in the same way as \eqref{eq:extended_c}.
\end{remark}

\autoref{thm:adversarial_risk_is_robust_risk} generalises and subsumes a number of existing results \citetext{\citealp[Cor.~2 (iv)]{Gao_DistributionallyRobust_2016}; \citealp[Prop.~3.1]{Staib_DistributionallyRobust_2017}; \citealp[Thm.~12]{Shafieezadeh-Abadeh_RegularizationMass_2019}} to relate the adversarial risk minimisation \eqref{eq:adv_risk} to the distributionally robust risk in \autoref{thm:robustness_bound}. The previous results mentioned are all are formulated with respect to an empirical distribution, that is, an average of Dirac masses. Of course any finite set is compact, and so these empirical distributions satisfy the concentration assumption.

% In general, it is difficult to characterise the tightness of the upper bounds in Theorem \ref{thm:multiclass_robust_dual_minimisation} and \ref{thm:adversarial_risk_inbetween}.
A simulation is in place demonstrating that the sum of the three gaps in \eqref{eq:inequality_thing} and \autoref{thm:robustness_bound,thm:adversarial_risk_is_robust_risk} is relatively low.
% So we resorted to an empirical demonstration that the sum of all the three gaps in \eqref{eq:adversarial_risk_inbetween} is relatively low. 
We randomly generated 100 Gaussian kernel classifiers $f = \sum_{i = 1}^{100}\gamma_i k(x^i, \cdot)$, with $x^i$ sampled from the MNIST dataset and $\gamma_i$ sampled uniformly from $[-2, 2]$. The bandwidth was set to the median of pairwise distances. In Figure \ref{fig:certificate_gap},
the $x$-axis is the adversarial risk (LHS of \eqref{eq:inequality_thing}) where the perturbation $\delta$ is bounded in $\ell_p$ ball and computed by PGD.
The $y$-axis is the Lipschitz regularised empirical risk (RHS of \eqref{eq:upperbound}).
% with perturbation $\nbr{\epsilon}_2\le 3$ (\textbf{left}) and $\nbr{\epsilon}_\infty\le 0.3$ (\textbf{right}). 
The scattered dots lie closely to the diagonal, demonstrating that the above bounds are tight in practice.

\begin{figure*}[t]
	\centering
	% \makebox[1\textwidth][c]{
    \begin{minipage}{0.5\linewidth-0.5em}
        \subcaptionbox{$\nbr{\delta}_2 \le 3$\label{fig:certificate_gap_l2}}[0.5\linewidth-0.5em]{
            \includegraphics[clip=true, width=\linewidth, viewport = 4.8cm 8.5cm 16cm 19cm]{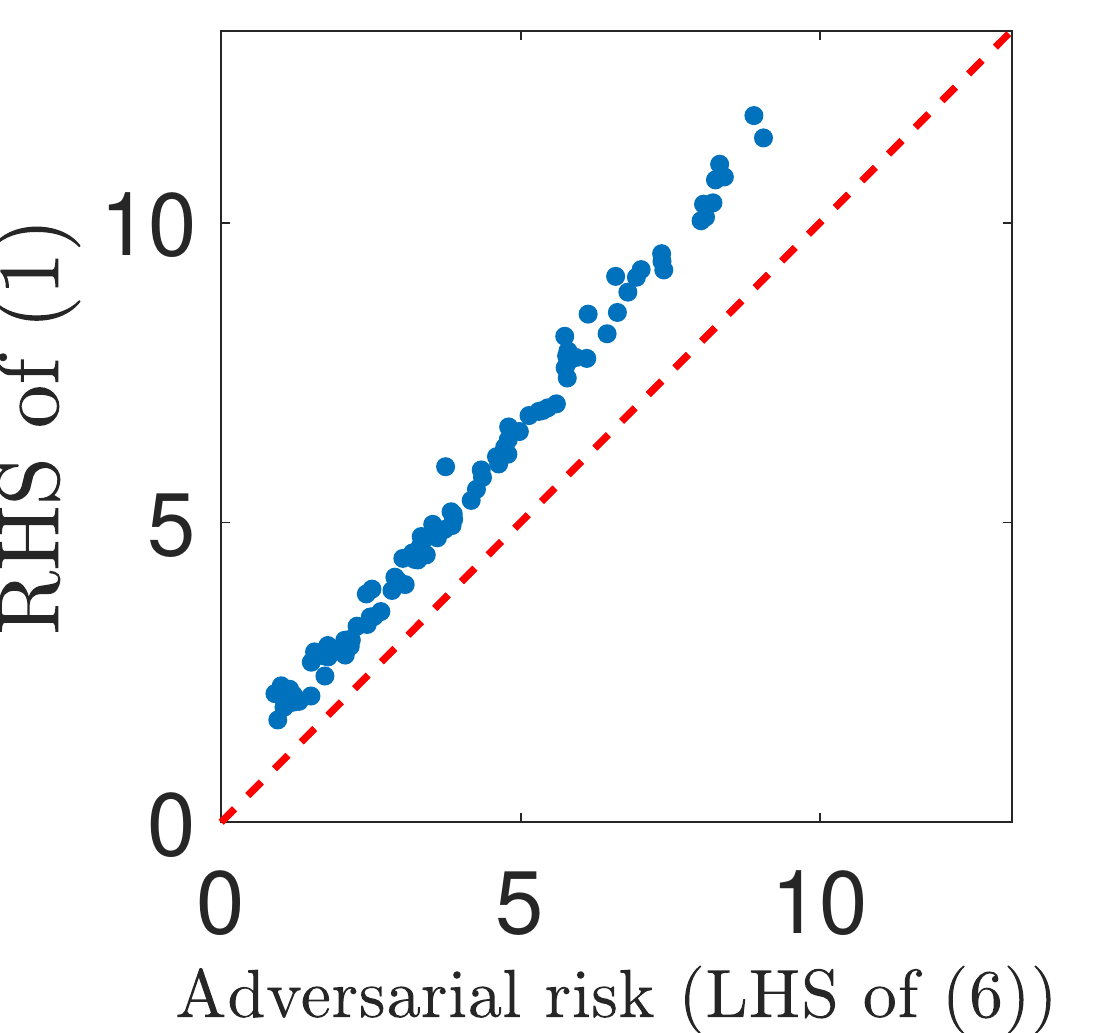}
        }
        \hfill
        \subcaptionbox{$\nbr{\delta}_\infty \le 0.3$\label{fig:certificate_gap_linf}}[0.5\linewidth-0.5em]{
            \includegraphics[clip=true, width=\linewidth, viewport = 4.8cm 8.5cm 16cm 19cm]{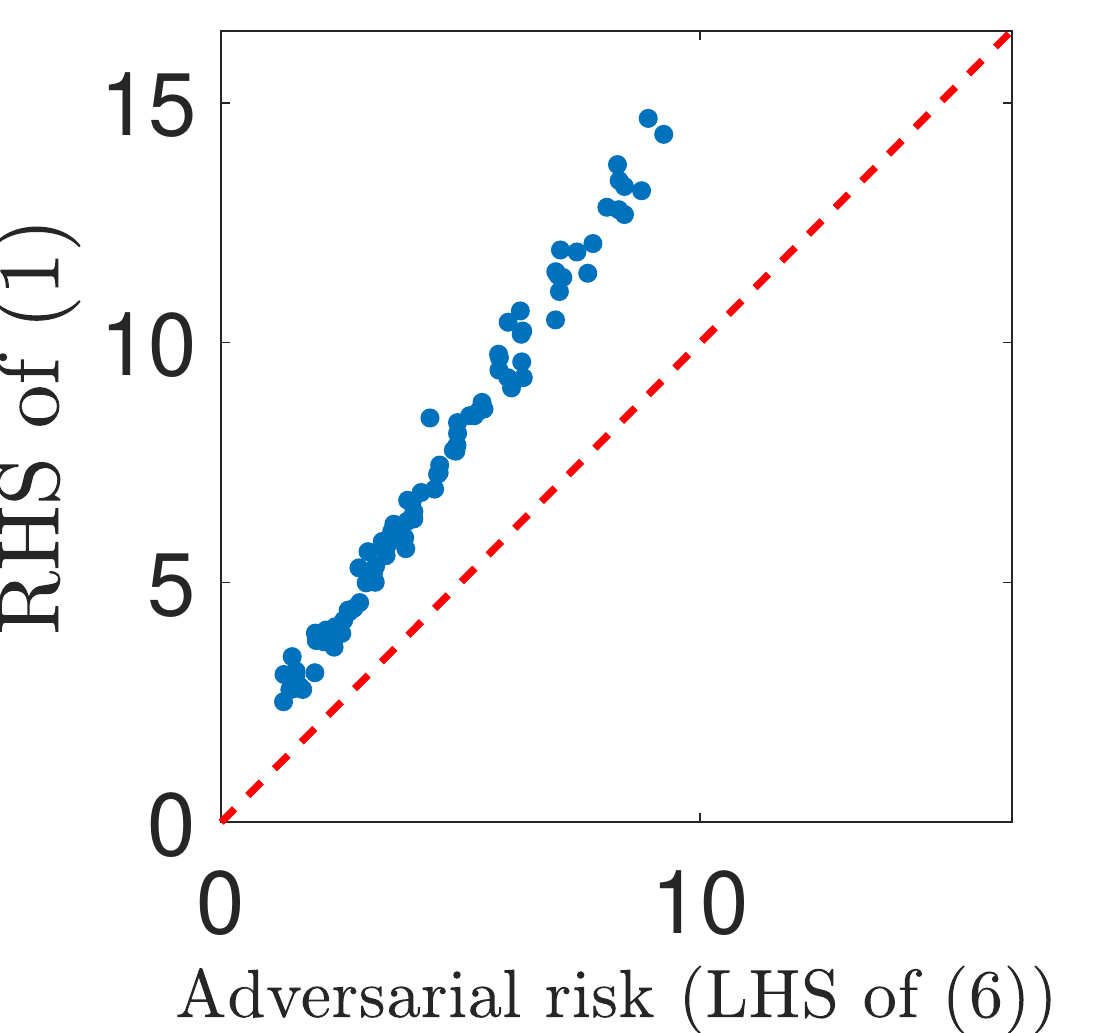}
        }
        \caption{Empirical evaluation of the sum of the gaps from \autoref{thm:robustness_bound,thm:adversarial_risk_is_robust_risk}.
        The Lipschitz constants $\sup_{x \in X} \nbr{\nabla f(x)}_q$ (left: $p=2$, right: $ p = \infty$, $1/p + 1/q = 1$) were estimated by BFGS.\label{fig:certificate_gap}}
    \end{minipage}
    \hfill
    \begin{minipage}{0.5\linewidth-0.5em}
        \centering
        \subcaptionbox{5-layer inv.\ kernel}[0.5\linewidth-0.5em]{
    %		\label{fig:acc_pgd_train}
            \includegraphics[clip=true, width=\linewidth, viewport = 2.8cm 6.5cm 17.5cm 20.5cm]{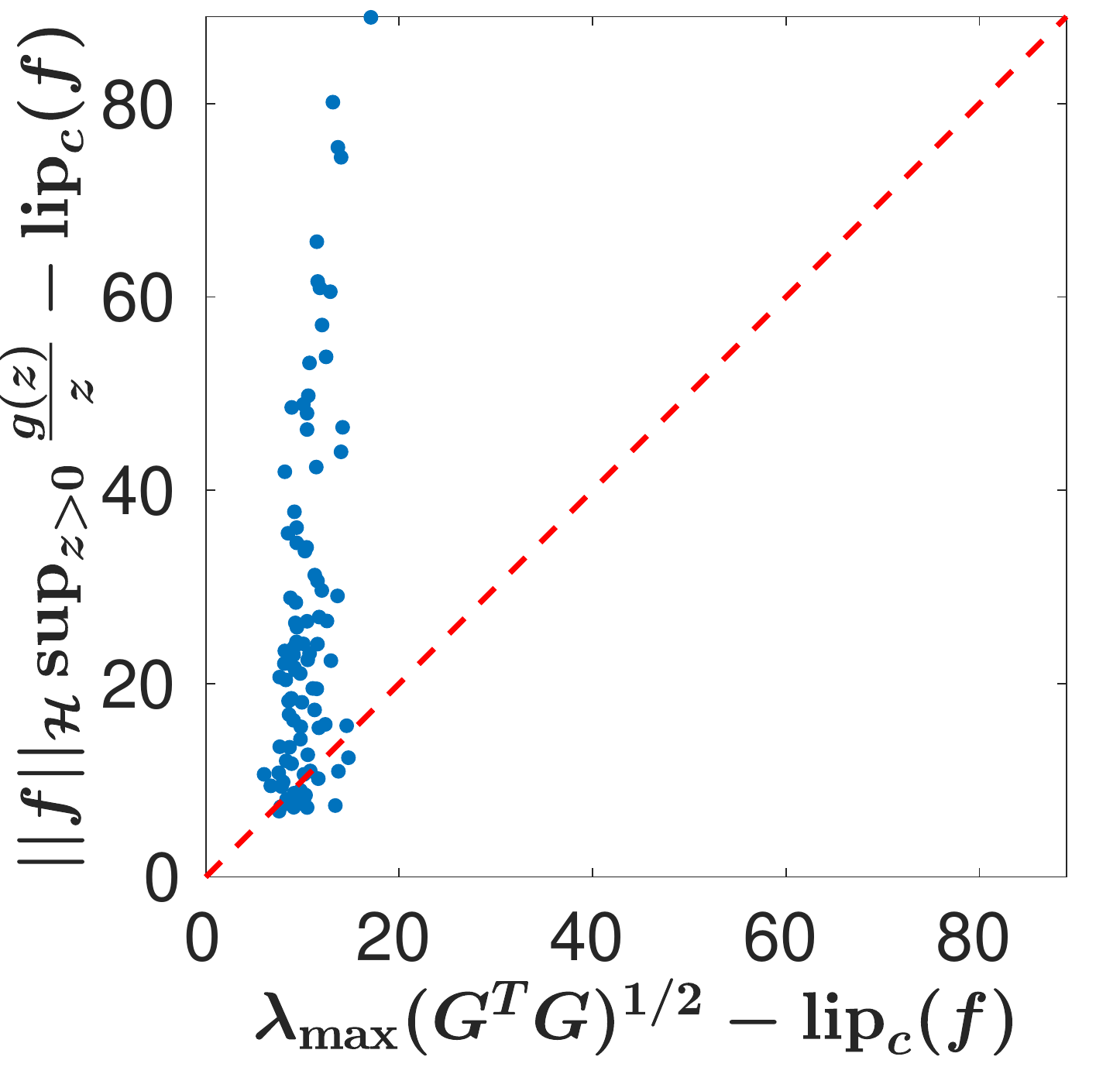}
            }
        \hfill
        \subcaptionbox{Gaussian kernel}[0.5\linewidth-0.5em]{
    %		\label{fig:acc_pgd_train}
            \includegraphics[clip=true, width=\linewidth, viewport = 2.8cm 6.5cm 17.5cm 20.5cm]{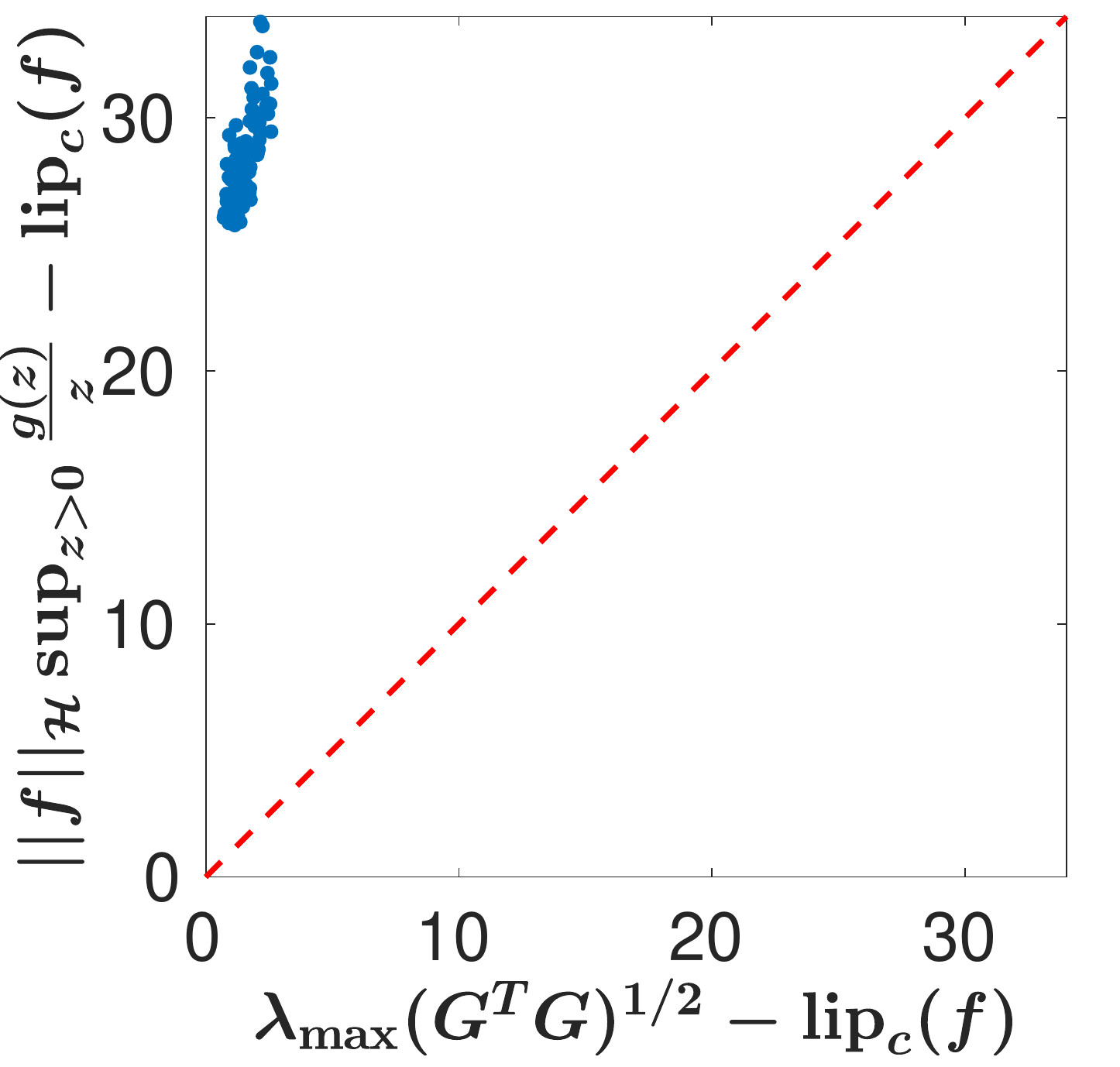}
            }
        \caption{Comparison of $\lambda_{\max} (G^\top G)$ and the RHS of \eqref{eq:growth_fun}, 
        as upper bounds for the Lipschitz constant.
        Smaller values are tighter.
        The 100 functions were sampled in the same way as in \autoref{fig:certificate_gap}.
    % 	Larger $\sigma$ in Gaussian kernel means smoother functions, 
    % 	making the two bounds closer.
        \label{fig:comp_bounds}}
    \end{minipage}
	% }
\end{figure*}

\let\oldnu\nu\let\nu\mu
\section{Lipschitz regularisation for kernel methods}
\label{sec:kernel}
\autoref{thm:robustness_bound,thm:adversarial_risk_is_robust_risk} 
open up a new path to optimising the adversarial risk \eqref{eq:adv_risk} by Lipschitz regularisation (RHS of \eqref{eq:boring_upperbound}),
where the upper bounding relationship is established through DRR.
In general, however, it is still hard to compute the Lipschitz constant for a nonlinear model.
Interestingly, we will show that for some types of kernels,
this can be done efficiently on functions in its RKHS.
Thanks to the known connections between kernel method and deep learning,
this technique will also potentially benefit the latter.
For example, $\ell_1$-regularised neural networks are compactly contained in the RKHS of multi-layer inverse kernels 
$k(x,y) = (2-x^\top y)^{-1}$ with $\nbr{x}_2 \le 1$ and $\nbr{y}_2 \le 1$
\citep[Lem.~1 and Thm.~1]{ZhaLeeJor16} and \citep{ShaShaSri11,ZhaLiaWai17}, 
and even possibly Gaussian kernels $k(x,y) = \exp(-\nbr{x-y}^2 / (2\sigma^2))$ \citep[][\S 5]{ShaShaSri11}.

Consider a Mercer's kernel $k$ on a convex domain $\Xcal \subseteq \R^d$,
with the corresponding RKHS denoted as $\Hcal$.
The standard kernel method seeks a discriminant function $f$ from $\Hcal$ with the conventional form of finite kernel expansion
$f(x) = \frac{1}{l} \sum_{a=1}^l \gamma_a k(x^a, \cdot)$,
such that the regularised empirical risk can be minimised with
the standard (hinge) loss and RKHS norm.
We start with real-valued $f$ for univariate output such as binary classification, and later extend it to multiclass.

Our goal here is to additionally enforce, while retaining a convex optimisation in $\vga \defeq \{\gamma_a\}$, that the Lipschitz constant of $f$ 
falls below a prescribed threshold $L > 0$,
which is equivalent to $\sup_{x \in \Xcal} \nbr{\nabla f(x)}_2 \le L$ thanks to the convexity of $X$.
A quick but primitive solution is to piggyback on the standard RKHS norm constraint $\nbr{f}_\Hcal \le C$,
in view that it already induces an upper bound on $\nbr{\nabla f(x)}_2$ as shown in Example 3.23 of \citet{Shafieezadeh-Abadeh_RegularizationMass_2019},
\begin{gather}
    \sup_{x \in \Xcal} \nbr{\nabla f(x)}_2 \ \le \ \nbr{f}_\Hcal \ \sup_{z > 0} \smallfrac{1}{z} g(z), \label{eq:growth_fun}
    \intertext{where }
    g(z) \ge \sup_{x, x' \in \Xcal: \nbr{x-x'}_2 = z}\nbr{k(x,\cdot) - k(x', \cdot)}_\Hcal.
\end{gather}
%
%where $\partial^j f(x)$ is the partial derivative with respect to the $j$-th coordinate of $x$, and $\partial^{j,j} k(x,x)$ stands for the partial derivative of $k$ with respect to the $j$-th coordinate of the first and second arguments.
%evaluated at $(x,x)$.
%For many kernels such as Gaussian, Laplacian, inverse, and periodic kernels that we will detail in the sequel,
%$\partial^{j,j} k(x,x)$ is bounded by a universal constant.
For Gaussian kernels,
$g(z) = \max\{\sigma^{-1},1\} z$.
For exponential and inverse kernels, $g(z) = z$ \citep{BieMai19}.
\citet{BieMaiCheetal19} justified that the RKHS norm of a neural network may serve as a surrogate for Lipschitz regularisation.
But the quality of such an approximation,
\ie, the gap in \eqref{eq:growth_fun}, % in \eqref{eq:gnorm_by_RKHS_norm} 
can be loose as we will see later in \autoref{fig:comp_bounds}.
Besides, $C$ and $L$ are supposed to be independent parameters.

How can we tighten the approximation?
A natural idea is to directly bound the gradient norm at $n$ random  locations $\{w^s\}_{s=1}^n$ sampled \iid\ from $\Xcal$.
% draw $n$ \iid\ samples $\{w^s\}_{s=1}^n$ from $\Xcal$,
% and to bound the gradient norm at $w^s$ by $L$.
These are obviously convex constraints on $\vga$.
But how many samples are needed in order to ensure 
$\nbr{\nabla f(x)}_2 \le L + \epsilon$ for \emph{all} $x \in \Xcal$?
Unfortunately, as shown in \autoref{sec:app_exp_sample},
$n$ may have to grow exponentially by $1/\epsilon^d$ for a $d$-dimensional space.
Therefore we seek a more efficient approach by first slightly relaxing $\nbr{\nabla f(x)}_2$. 
Let $g_j(x) \defeq \partial^j f(x)$ be the partial derivative with respect to the $j$-th coordinate of $x$,
and $\partial^{i,j} k(x,y)$ be the partial derivative to $x_i$ and $y_j$.
$i$ or $j$ being 0 means no derivative.
Assuming $\sup_{x \in \Xcal} k(x,x) = 1$ and $g_j \in \Hcal$ 
(true for various kernels considered by \autoref{assum:tail_bnd,assum:uniform_bnd} below), we get a bound
\begin{align}
    \label{eq:norm_grad_by_rkhs}
    \MoveEqLeft[3]
    \sup_{x \in \Xcal} \nbr{\nabla f(x)}^2_2 
    %= \sup_{x \in \Xcal} \sum_{j=1}^d g_j(x)^2 
      \ = \ \sup_{x \in \Xcal} \smash{\sum\nolimits_{j=1}^d} \inner{g_j}{k(x,\cdot)}^2_\Hcal
    \\&\le \sup_{\phi: \nbr{\phi}_\Hcal = 1} \sum\nolimits_{j=1}^d \inner{g_j}{\phi}^2_\Hcal 
    %= \sup_{\phi: \nbr{\phi}_\Hcal = 1} \nbr{G^\top \phi}_2^2,
    \ =   \lambda_{\max} (G^\top G),
    %\ \le \ \sup_{x \in \Xcal} \sum_{j=1}^d \nbr{g_j}_\Hcal^2 \nbr{k(x,\cdot)}_\Hcal^2 
    %\ = \ \sum_{j=1}^d \nbr{g_j}_\Hcal^2.
\end{align}
where $\lambda_{\max}$ evaluates the maximum eigenvalue,
and $G \defeq (g_1, \ldots, g_d)$.
The ``matrix'' is only a notation because each column is a function in $\Hcal$,
and obviously the $(i,j)$-th entry of $G^\top G$ is $\inner{g_i}{g_j}_\Hcal$.
%and we adopt this notation for simplicity.
%In analogy to operator norm, we denote the right-hand side of \eqref{eq:norm_grad_by_rkhs} as $\nbr{G}_\op^2$.
%
Interestingly,
$\lambda_{\max} (G^\top G)$ delivers \emph{significantly lower} (\ie, tighter) value in approximating the Lipschitz constant $\sup_{x \in \Xcal} \nbr{\nabla f(x)}_2$,
compared with $\nbr{f}_\Hcal \max_{z > 0} \frac{g(z)}{z}$ from \eqref{eq:growth_fun}.
\autoref{fig:comp_bounds} compared these two approximants,
where $\lambda_{\max} (G^\top G)$ was computed from \eqref{eq:grad_coordinate} derived below, 
and the landmarks $\{w^s\}$ used the whole training set;
drawing more samples led to little difference.
The gap is smaller when the bandwidth $\sigma$ is larger,
making functions smoother.
To be fair, both Figure \ref{fig:certificate_gap} and \ref{fig:comp_bounds}
set $\sigma$ to the median of pairwise distances, 
a common practice.
	
Such a positive result motivated us to develop refined algorithms to address the only remaining obstacle to leveraging $\lambda_{\max} (G^\top G)$: 
no analytic form for computation.
Interestingly, it is readily approximable in both theory and practice. 
Indeed, the role of $g_j$ can be approximated by $\gtil_j$,
where 
$\gtil_j \in \R^n$ is the Nystr\"om approximation \citep{WilSee00b,DriMah05}:
\begin{align}
% \begin{aligned}
    \gtil_j &\defeq K^{-\frac12} g_j(w)^\top
    %K^{-\frac{1}{2}} 
    %\begin{pmatrix}
    %\inner{g_j}{k(w^1,\cdot)}_\Hcal \\
    %\vdots \\
    %\inner{g_j}{k(w^n,\cdot)}_\Hcal
    %\end{pmatrix}
    = (Z^\top Z)^{-\frac12}  Z^\top g_j,\label{eq:nystrom_approx}
% \end{aligned}
\end{align}
noting $g_j(w^i) = \inner{g_j}{k(w^i,\cdot)}_\Hcal$, where
\begin{align}
    g_j(w) &\defeq (g_j(w^1), \ldots, g_j(w^n)), &  \Gtil &\defeq (\gtil_1, \ldots, \gtil_d),
    \\
    Z &\defeq (k(w^1,\cdot), \ldots, k(w^n, \cdot)), &    K &\defeq [k(w^i, w^{i'})]_{i,i'}.
\end{align}
%
% Here the $Z$ ``matrix'' is only a notation since each column is a function in $\Hcal$.
% The meaning of the ``matrix-vector'' multiplication $Z^\top g_j$ is obviously $(g_j(w^1), \ldots, g_j(w^n))^\top$, and we adopt this notation for simplicity.
%The term $(Z^\top Z)^{-\frac{1}{2}} Z^\top$ is effectively orthogonalizing $Z$,
%in exactly the same spirit as \eqref{eq:pseudo_inverse_W}.
%
So to ensure $\lambda_{\max} (G^\top G) \le L^2 + \epsilon$
intuitively we can resort to enforcing $\lambda_{\max} (\Gtil^\top \Gtil) \le L^2$,
which also retains the convexity in the constraint in $\vga$.
However, to guarantee $\epsilon$ error, 
the number of samples ($n$) required is generally \emph{exponential} \citep{Barron94}.
Fortunately, we will next show that $n$ can be reduced to \emph{polynomial} for quite a general class of kernels that possess some decomposed structure.

\subsection{A Nystr\"om approximation for product kernels}

A number of kernels factor multiplicatively over the coordinates,
such as periodic kernels \citep{MacKay98}, 
Gaussian kernels, and Laplacian kernels.
We will consider $k(x, y) = \prod_{j=1}^d k_0(x_j, y_j)$ 
where $\Xcal = \Xcal_0^d$ and $k_0$ is a base kernel on $\Xcal_0$.
%Denote the RKHS on $\Xcal_0$ induced by $k_0$ as $\Hcal_0$.
Let the RKHS of $k_0$ be $\Hcal_0$, 
and let $\nu_0$ be a finite Borel measure with $\operatorname{supp}[\nu_0] = \Xcal_0$.
Periodic kernels have $k_0(x_j,y_j)  =  \exp\rbr1{-\sin \rbr{\frac{\pi}{v} (x_j  -  y_j)}^2 / (2\sigma^2)}$.

We emphasize that product kernels can induce vqery rich function spaces.
For example, Gaussian kernel is universal \citep{MicXuZha06}, 
meaning that its RKHS is \emph{dense} in the space of continuous functions in the maximum norm over any bounded domain.

The key benefit of this decomposition is that the derivative $\partial^{0,1} k(x,y)$ can be written as $\partial^{0,1} k_0(x_1, y_1) \prod_{j=2}^d k_0(x_j, y_j)$.
Since $k_0(x_j, y_j)$ can be easily dealt with, 
approximation will be needed \emph{only} for $\partial^{0,1} k_0(x_1, y_1)$.
Applying this idea to $g_1 = \frac{1}{l} \sum_{a=1}^l \gamma_a \partial^{0,1} k(x^a, \cdot)$,
we can derive
\begin{align}
    \nbr{g_1}^2_\Hcal 
    &= \frac{1}{l^2} \sum_{a,b = 1}^l\rbr3{\gamma_a \gamma_{b} \eta^{a,b}\prod_{j=2}^d k_0(x^{a}_j, x^{b}_j)},
    \label{eq:grad_coordinate}
    \\
    \inner{g_1}{g_2}_\Hcal 
    &= \frac{1}{l^2} \sum_{a,b = 1}^l\rbr3{\gamma_a \gamma_{b} \eta^{a,b}\prod_{j=2}^d k_0(x^{a}_j, x^{b}_j)},
\end{align}
where 
\begin{gather}
    \eta^{a,b} \defeq \inner{\partial^{0,1}k_0(x^a_1,\cdot)}{\partial^{0,1} k_0(x^{b}_1,\cdot)}_{\Hcal_0}.
\end{gather}
% \begin{multline}
% \label{eq:grad_coordinate}
% \!\!\!\!\!\!\!\! \nbr{g_1}^2_\Hcal = \frac{1}{l^2} \sum_{a,b = 1}^l \Big[\gamma_a \gamma_{b} c
% \prod\nolimits_{j=2}^d k_0(x^{a}_j, x^{b}_j) \Big],
% \end{multline}
% \begin{multline}
% \!\!\!\!\!\!\!\! \inner{g_1}{g_2}_\Hcal = \frac{1}{l^2}
% \sum_{a,b = 1}^l \Big[\gamma_a \gamma_{b} 
% \partial^{0,1}k_0(x^a_1,x^b_1) \partial^{0,1} k_0(x^b_2,x^a_2) \\
% \prod\nolimits_{j=3}^d k_0(x^{a}_j, x^{b}_j)\Big].
% \end{multline}
Thus the off-diagonal entries of $G^\top G$ can be computed exactly.
To approximate the diagonal, 
we sample $\{w^1_1, \ldots, w^n_1\}$ from $\nu_0$, and apply the Nystr\"om approximation of $\inner{\partial^{0,1} k_0(x^a_1,\cdot)}{\partial^{0,1} k_0(x^{b}_1,\cdot)}_{\Hcal_0}$:
\begin{align}
    \partial^{0,1} k_0(x^a_1,\cdot)^\top Z_1 \cdot (Z_1^\top Z_1)^{-1} \cdot Z_1^\top \partial^{0,1} k_0(x^{b}_1,\cdot),\label{eq:nystrom_coordinate}
\end{align}
where $Z_1 \defeq (k_0(w^1_1,\cdot), \ldots, k_0(w^n_1, \cdot))$, yielding 
\begin{align}
\label{eq:nystrom_coordinate_grad}
\MoveEqLeft[3]Z_1^\top \partial^{0,1} k_0(x^{a}_1,\cdot)
    \\&= (\partial^{0,1} k_0(x^{a}_1, w^1_1), \ldots, \partial^{0,1} k_0(x^{a}_1, w^n_1))^\top,
%&= 		\begin{pmatrix}
%\partial^{0,1} k_0(x^{a}_1, w^1_1) \\
%\vdots \\
%\partial^{0,1} k_0(x^{a}_1, w^n_1)
%\end{pmatrix}^\top
%(Z^\top_1 Z_1)^{-1}
%\begin{pmatrix}
%\partial^{0,1} k_0(x^{b}_1, w^1_1) \\
%\vdots \\
%\partial^{0,1} k_0(x^{b}_1, w^n_1)
%\end{pmatrix}.
\end{align}
and similarly for $Z_1^\top \partial^{0,1} k_0(x^{b}_1,\cdot)$.
Denote this approximation of $G^\top G$ as $\Ptil_G$.
Clearly, $\lambda_{\max}(\Ptil_G) \le L^2$ is a convex constraint on $\vga$,
based on \iid\ samples 
$\{w_j^s | s \in [n], j \in [d]\}$ 
% $\setcond1{w^s_j}{s \in [n], j \in [d]}$
from $\nu_0$.
The overall \emph{convex} training algorithm is summarised in \autoref{sec:algo_detail},
along with detailed derivations.

\subsection{General sample complexity and assumptions}
 
Finally, it is important to analyse how many samples $w^s_j$ are needed, 
such that with high probability
\begin{gather}
    \lambda_{\max}(\Ptil_G) \le L^2  
    \implies
    \lambda_{\max} (G^\top G) \le L^2 + \epsilon .
\end{gather}
Fortunately, product kernels only require approximation bounds for each coordinate,
making the sample complexity immune to the exponential growth in the dimensionality $d$.
Specifically, we first consider base kernels $k_0$ with a scalar input, 
\ie, $\Xcal_0 \subseteq \R$.
%This will serve as the basis for the general product kernels.
%in the form of $k(x, y) = \prod_{j=1}^d k_0(x_j, y_j)$
%defined over $\Xcal_0^d$.
%
Recall from \citet[\S4]{steinwart2008support}  that the integral operator for $k_0$ and $\nu_0$ is $ T_{k_0} \defeq I \circ S_{k_0}$, where $S_{k_0}: \leb_2(\Xcal_0, \nu_0) \to \contf(\Xcal_0)$ operates according to 
\begin{gather}
\label{eq:def_Tk}
    \forall{f \in \leb_2(\Xcal_0, \nu_0)} (S_{k_0} f)(x) \defeq \int k_0(x,y) f(y)\nu_0(\dv y), 
\end{gather}
and $I$: $\contf(\Xcal_0) \hookrightarrow \leb_2(\Xcal_0, \nu_0)$ is the inclusion operator.
By the spectral theorem, 
if $T_{k_0}$ is compact, 
then there is an at most countable orthonormal set $\{\etil_j\}_{j \in J}$ of $\leb_2(\Xcal_0, \nu_0)$ and $\{\lambda_j\}_{j \in J}$ with $\lambda_1 \ge \lambda_2 \ge \ldots > 0$ such that 
$T_{k_0} f = \sum_{j \in J} \lambda_j \inner{f}{\etil_j}_{\leb_2(\Xcal_0, \nu_0)} \etil_j$ for all $f \in \leb_2(\Xcal_0, \nu_0)$.
It is easy to see that $\varphi_j \defeq  \sqrt{\lambda_j} e_j$ is an orthonormal basis of $\Hcal_0$ \citep{steinwart2008support}.
% and we will denote
% $\Phi_m \defeq  (\varphi_1, \ldots, \varphi_m)$.

Our proof is built upon the following two assumptions on the base kernel.
The first one asserts that fixing $x$, the energy of $k_0(x,\cdot)$ and $\partial^{0,1} k_0(x,\cdot)$ ``concentrates'' on the leading eigenfunctions.
\begin{assumption}
	\label{assum:tail_bnd}
	Suppose $k_0(x,x) = 1$ and
	$\partial^{0,1} k_0(x, \cdot) \in \Hcal_0$ for all $x \in \Xcal_0$. 
	For all $\epsilon > 0$, there exists $N_\epsilon \in \NN$ such that the tail energy of $\partial^{0,1} k_0(x, \cdot)$ beyond the $N_\epsilon$-th eigenpair is less than $\epsilon$, 
	uniformly for all $x \in \Xcal_0$.
    That is, denoting $\Phi_m \defeq  (\varphi_1, \ldots, \varphi_m)$, for all $x\in X_0$:
    \begin{align}
        \nbr{\partial^{0,1} k_0(x, \cdot) - \Phi_m \Phi_m^\top \partial^{0,1} k_0(x, \cdot)}_{\Hcal_0} &< \epsilon, \text{ and} \\
        \nbr{k_0(x,\cdot) - \Phi_m \Phi_m^\top k_0(x,\cdot)}_{\Hcal_0} &< \epsilon.
    \end{align}
\end{assumption}

% Note since all $f \in \Hcal_0$ can be written as $\sum_{j \in J} a_j e_j$ with $\{a_j /\sqrt{\lambda_j}\} \in \ell^2(J)$,
% there must be a large enough $m$ such that $\nbr{f - \Phi_m \Phi_m^\top f}_{\Hcal_0} \le \epsilon$.
% So the key assumption here is that this holds \emph{uniformly} for all $x \in \Xcal_0$.
% We our sample complexity will be linear in $N_\epsilon$.
The second assumption asserts the smoothness and range of eigenfunctions \emph{in a uniform sense}.
\begin{assumption}
	\label{assum:uniform_bnd}
	Under \autoref{assum:tail_bnd}, 
	$\{e_j(x) : j \in N_\epsilon\}$ is uniformed bounded over $x \in X_0$,
	and 
	the RKHS inner product of $\partial^{0,1} k_0(x, \cdot)$ with 
	$\{e_j : j \in N_\epsilon\}$ is also uniformly bounded over $x \in \Xcal_0$:
	\begin{align}
        M_\epsilon &\defeq  \sup_{x \in \Xcal_0} \max_{j \in [N_\epsilon]}
            \abr{\inner{\partial^{0,1} k_0(x, \cdot)}{e_j}_{\Hcal_0}}< \infty,
        \\Q_\epsilon &\defeq  \sup_{x \in \Xcal_0} \max_{j \in [N_\epsilon]} \abr{e_j(x)}< \infty.
	\end{align}
\end{assumption}
% With Assumptions \ref{assum:tail_bnd} and \ref{assum:uniform_bnd}, our main result of sample complexity can be stated informally as follows.
% A formal statement and proof are relegated to \autoref{sec:complexity}.
%
\begin{theorem}
	\label{thm:sample_complexity_informal}
	Suppose $k_0$, $\Xcal_0$, and $\nu_0$ satisfy Assumptions \ref{assum:tail_bnd} and \ref{assum:uniform_bnd}.
	Let $\{w^s_j: s\in [n], j \in [d]\}$ be sampled i.i.d. from $\nu_0$.
	% Let $f^*$ be the ,	
	%
	Then for any $f$ whose coordinate-wise Nystr\"om approximation \eqref{eq:grad_coordinate} and \eqref{eq:nystrom_coordinate} satisfy $\lambda_{\max}(\Ptil_G) \le L^2$, 
	the Lipschitz condition 
	$\lambda_{\max} (G^\top G) \le L^2 + \epsilon$ is met with probability $1-\delta$,
	as long as $n \ge \tilde{\Theta}\rbr{\frac{1}{\epsilon^2} N^2_\epsilon M_\epsilon^2 Q^2_\epsilon \log \frac{d N_\epsilon}{\delta}}$,
	almost independent of $d$.
	Here $\tilde{\Theta}$ hides all poly-log terms.
\end{theorem}

The $\log d$ dependence on dimensionality $d$ is interesting,
but not surprising.
After all, only the diagonal entries of $G^\top G$ need approximation, and the quantity of interest is its $\lambda_{\max}$.

\paragraph{Satisfaction of assumptions}
In \autoref{sec:assump_check_periodic} and \ref{sec:assump_check_gauss}, 
we will show that for periodic kernel and Gaussian kernel, 
Assumptions \ref{assum:tail_bnd} and \ref{assum:uniform_bnd} hold true with $\tilde{O}(1)$ values of  
$N_\epsilon$, $M_\epsilon$, and $Q_\epsilon$.
It remains open whether non-product kernels such as inverse kernel also enjoy this polynomial sample complexity.
\autoref{sec:assump_check_inverse} suggests that its complexity may be \emph{quasi-polynomial}.

\let\nu\oldnu
\section{Experimental results}
\begin{figure*}[t]
	\centering
	\subcaptionbox{MNIST\label{fig:mnist_l2_cw}}[.33\linewidth-0.5em]{
		\includegraphics[clip=true, width=\linewidth, viewport = 3.5cm 9.5cm 17.7cm 17.5cm]{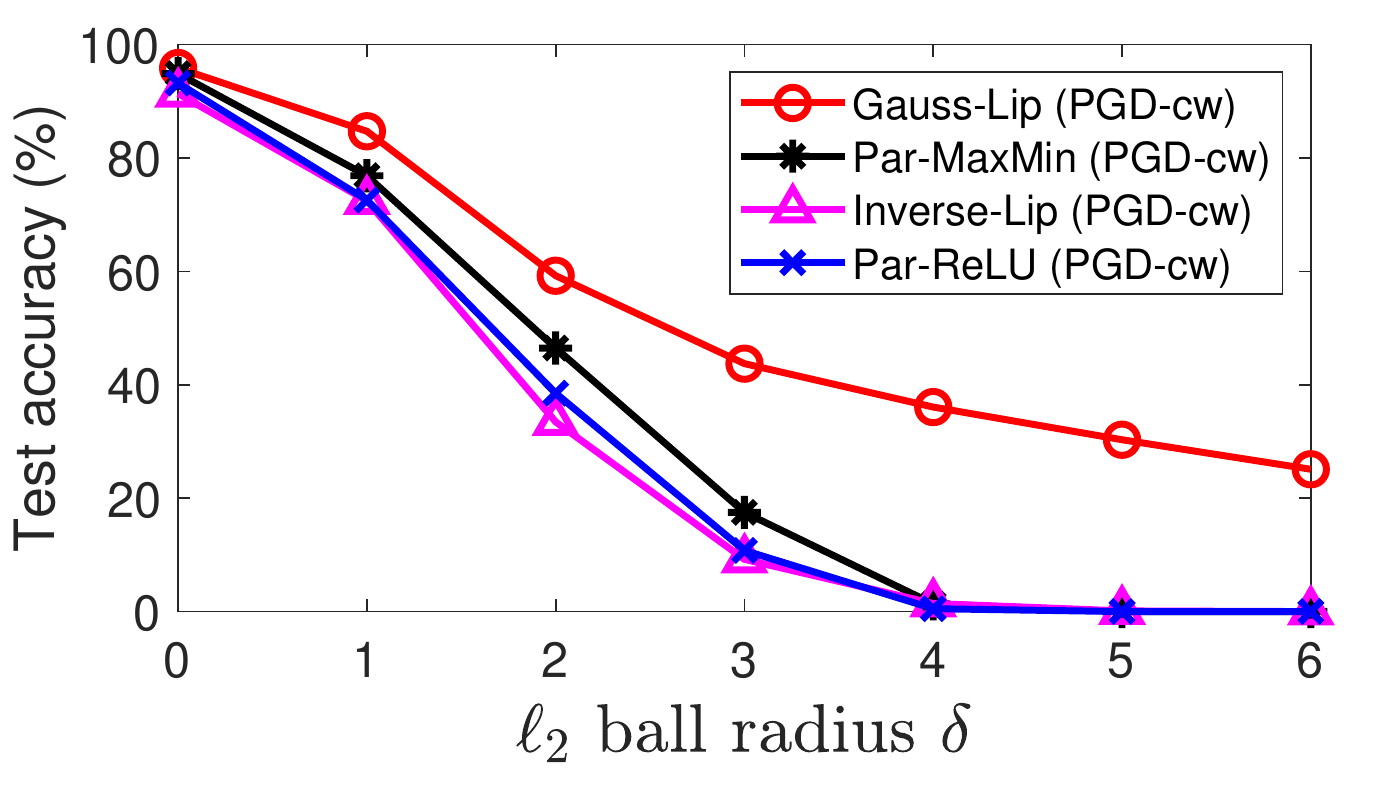}
        }
    \hfill
	\subcaptionbox{Fashion-MNIST\label{fig:fashion_l2_cw}}[.33\linewidth-0.5em]{
		\includegraphics[clip=true, width=\linewidth, viewport = 3.5cm 9.5cm 17.7cm 17.5cm]{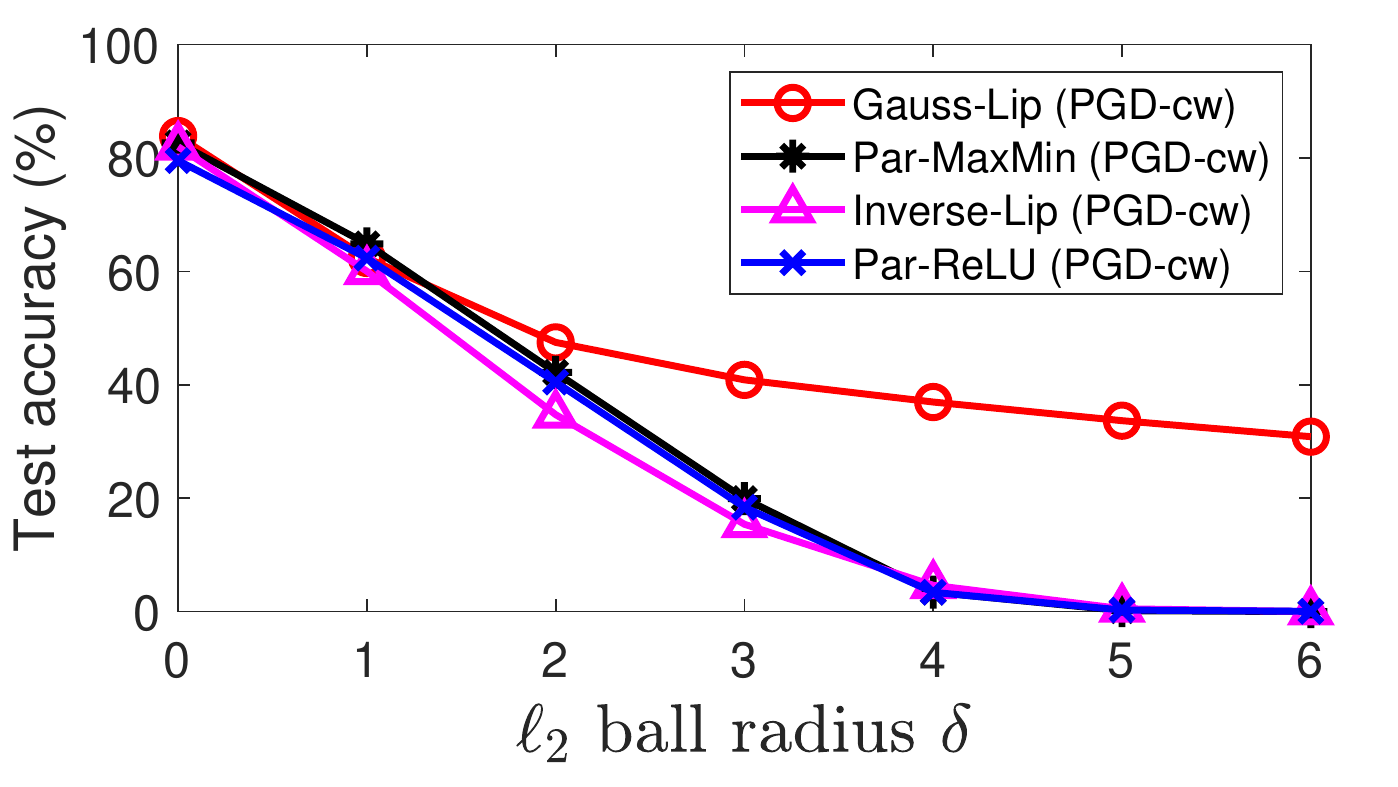}
        }
    \hfill
	\subcaptionbox{CIFAR10\label{fig:cifar_l2_cw}}[.33\linewidth-0.5em]{
		\includegraphics[clip=true, width=\linewidth, viewport = 3.5cm 9.5cm 17.7cm 17.5cm]{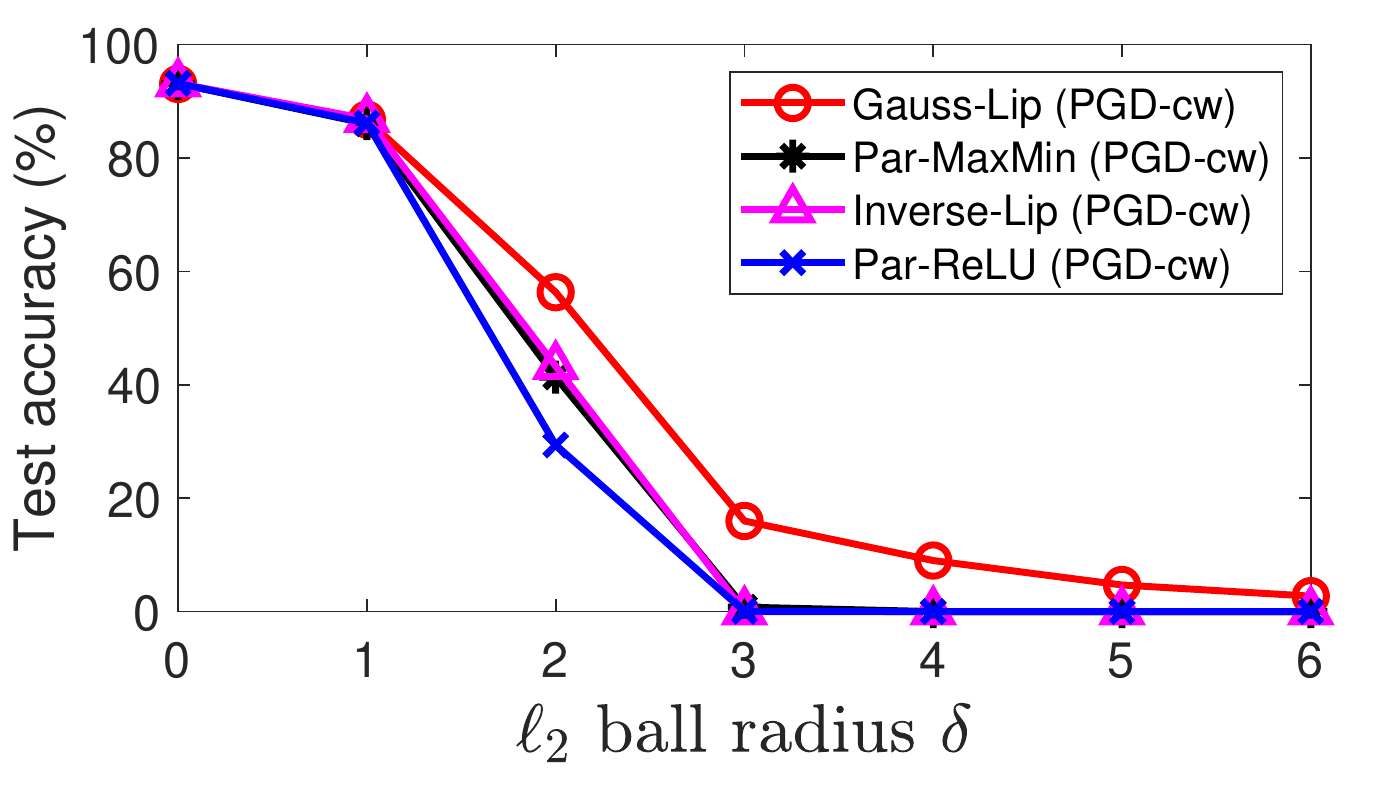}
		}
    \caption{Test accuracy under PGD attacks on the C\&W approximation with $2$-norm norm bound\label{fig:robustness_l2_cw}}
    \vspace{\abovecaptionskip}\par
	\subcaptionbox{MNIST\label{fig:mnist_linf_cw}}[.33\linewidth-0.5em]{
		\includegraphics[clip=true, width=\linewidth, viewport = 3.5cm 9.5cm 17.7cm 17.5cm]{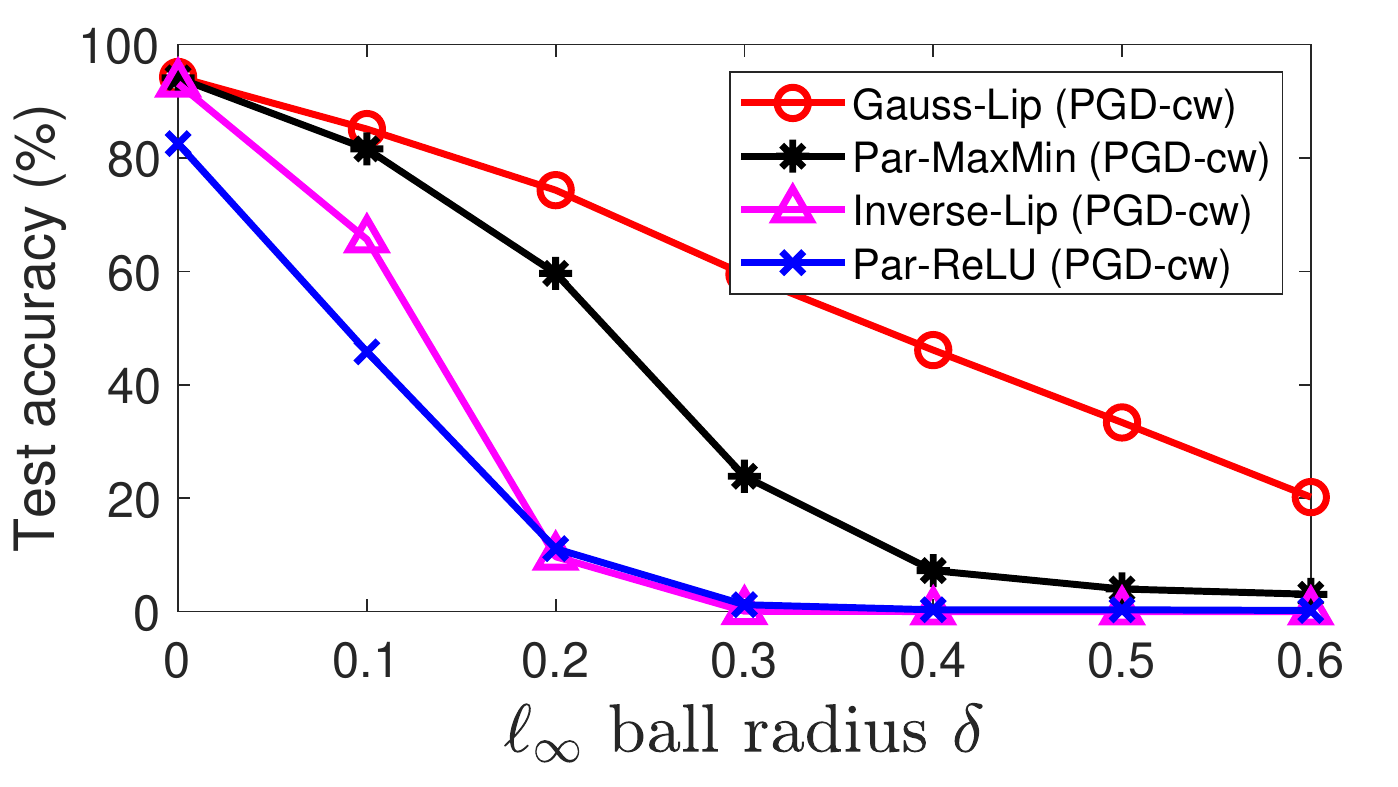}
        }
    \hfill
	\subcaptionbox{Fashion-MNIST\label{fig:fashion_linf_cw}}[.33\linewidth-0.5em]{
		\includegraphics[clip=true, width=\linewidth, viewport = 3.5cm 9.5cm 17.7cm 17.5cm]{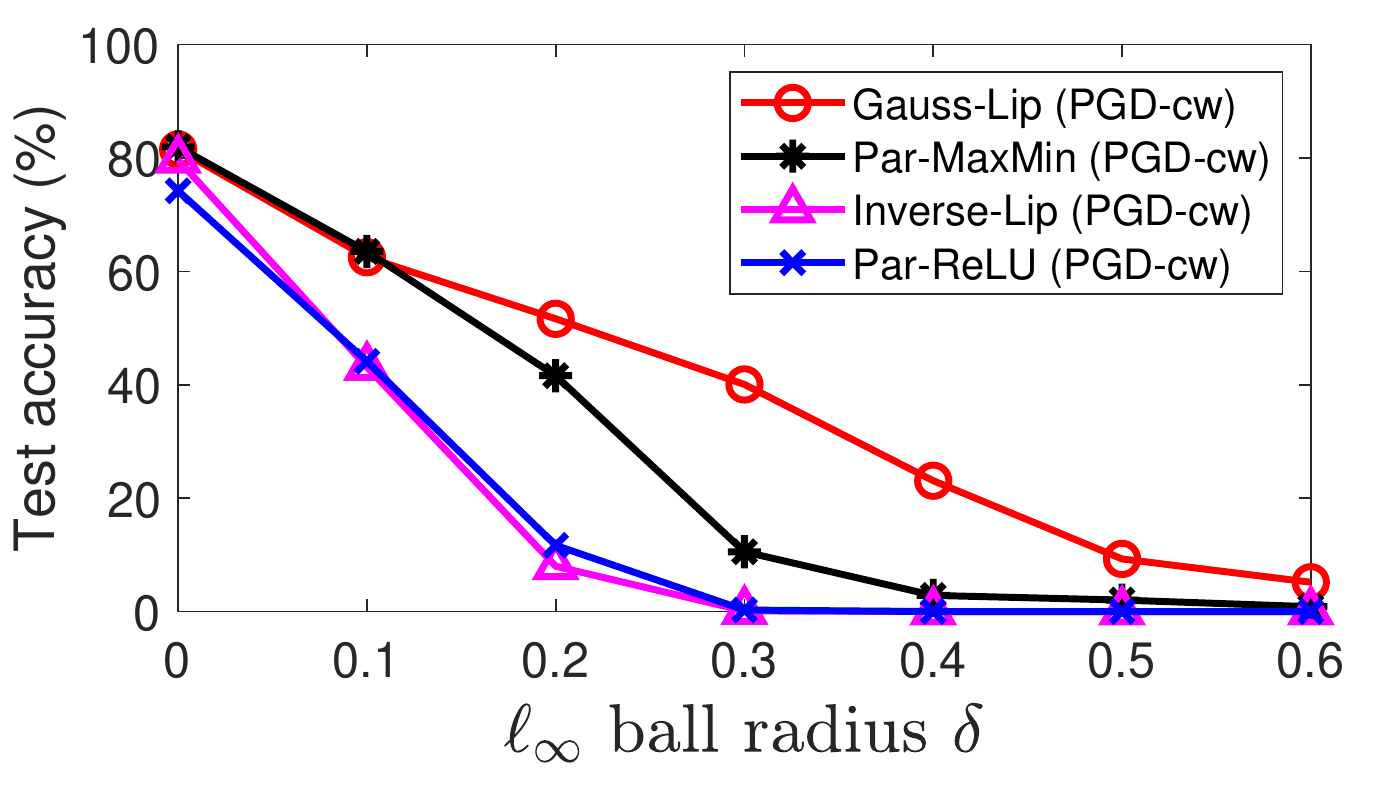}
        }
    \hfill
	\subcaptionbox{CIFAR10\label{fig:cifar_linf_cw}}[.33\linewidth-0.5em]{
		\includegraphics[clip=true, width=\linewidth, viewport = 3.5cm 9.5cm 17.7cm 17.5cm]{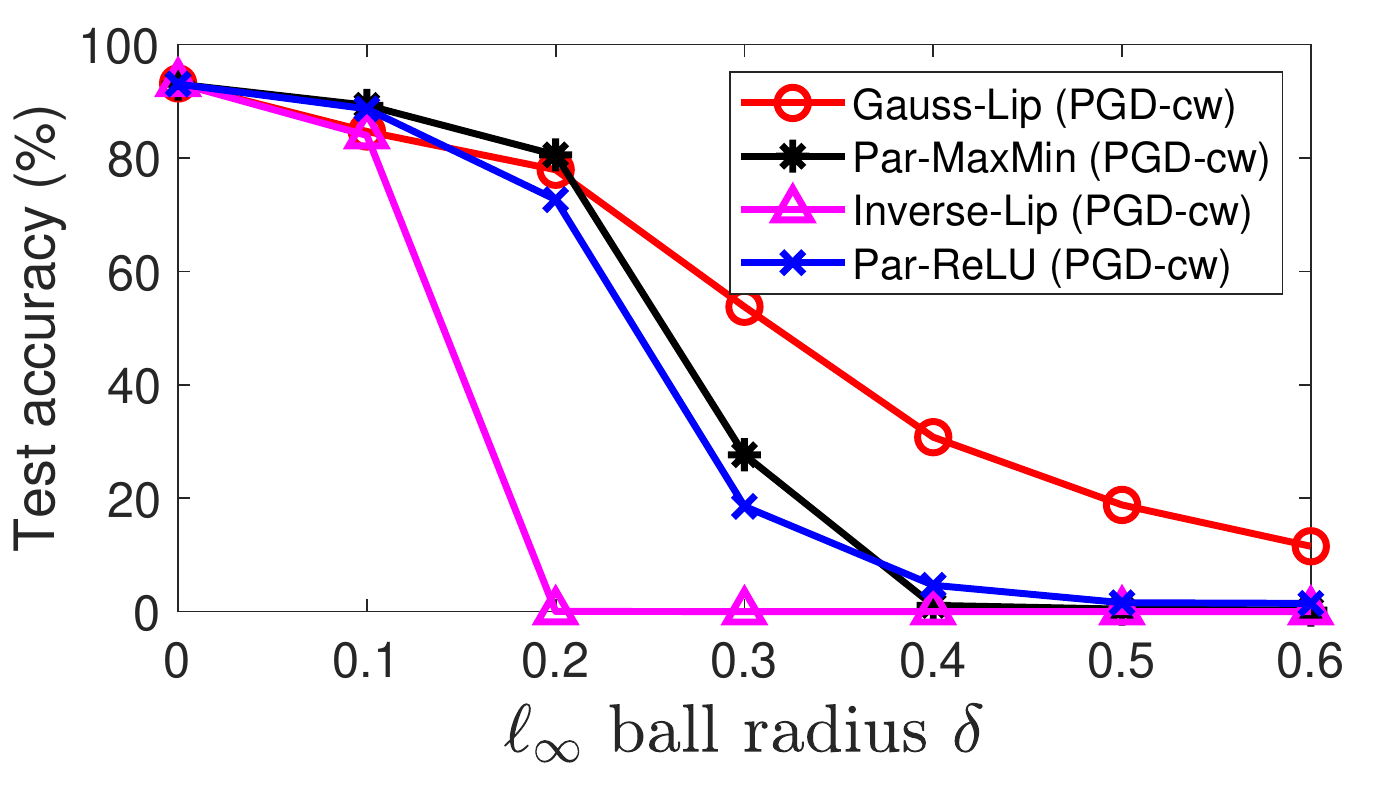}
		}
	\caption{Test accuracy under PGD attacks on the C\&W approximation with $\infty$-norm norm bound\label{fig:robustness_linf_cw}}
\end{figure*}

\label{sec:experiment}

We studied the empirical robustness and accuracy of the proposed Lipschitz regularisation technique for adversarial training of kernel methods,
under both Gaussian kernel and inverse kernel.
Comparison will be made with state-of-the-art defence algorithms under effective attacks.

\paragraph{Datasets}
We tested on three datasets: MNIST, Fashion-MNIST, and CIFAR10.
The number of training/validation/test examples for the three datasets are
54k/6k/10k, 54k/6k/10k, 45k/5k/10k, respectively.
% Both datasets have 10 classes, 
% and all the classes get almost the same number of examples, 
% both in training and test sets.
Each image in MNIST and Fashion-MNIST is represented as a 784-dimensional feature vector, with each feature/pixel normalised to $[0,1]$. 
For CIFAR10, we trained it on a residual network to obtain a 512-dimensional feature embedding, which were subsequently normalised to $[0,1]$. 
They were used as the input for training all the competing algorithms and were subject to attack.

\paragraph{Attacks}
To evaluate the robustness of the trained model,
we attacked them on test examples using the random initialized Projected Gradient Descent method with 100 steps \citep[PGD,][]{Madry_DeepLearning_2018} under two losses: cross-entropy and C\&W loss \citep{Carlini_EvaluatingRobustness_2017}.
% and the Fast Gradient Sign method \cite[FGS,][]{PapFagCar18}, 
The perturbation $\delta$ was constrained in an $2$-norm or $\infty$-norm ball. 
To evaluate robustness, 
we scaled the perturbation bound $\delta$ 
from $0.1$ to $0.6$ for $\infty$-norm norm,
and from $1$ to $6$ for $2$-norm norm 
(when $\delta = 6$, the average magnitude per coordinate is 0.214).
% New images after perturbation are renormalised to unit $2$-norm norm before fed into \emph{all} the trained models for testing.
We normalised gradient and fine-tuned the step size.

\paragraph{Algorithms}

We compared four training algorithms. 
The Parseval network orthonormalises the weight matrices to enforce the Lipschitz constant \citep{Cisse_ParsevalNetworks_2017}. 
We used three hidden layers of $1024$ units and ReLU activation (\parrelu). %
Also considered is the Parseval network with MaxMin activations (\parmaxmin), 
which enjoys much improved robustness \citep{Anil_SortingOut_2019}.
Both algorithms can be customised for $2$-norm or $\infty$-norm attacks,
and were trained under the corresponding norms.
Using multi-class hinge loss, 
they constitute strong baselines for adversarial learning.
We followed the code from \citet{LNetscode} with $\beta = 0.5$, 
which is equivalent to the first order Bjorck algorithm. 
The final upper bound of Lipschitz constant computed from the learned weight matrices satisfied the orthogonality constraint as shown in Figure 13 of \citet{Anil_SortingOut_2019}.

Both Gaussian and inverse kernel machines applied Lipschitz regularisation by randomly and greedily selecting $\{w^s\}$,
and they will be referred to as \glip\ and \ilip, respectively.
In practice, \glip\ with the coordinate-wise Nystr\"om approximation ($\lambda_{\max}(\Ptil_G)$ from \eqref{eq:nystrom_coordinate}) can approximate $\lambda_{\max}(G^\top G)$ with a much smaller number of sample than if using the holistic approximation as in \eqref{eq:nystrom_approx}.
Furthermore, we found an even more efficient approach.
Inside the iterative training algorithm,
we used L-BFGS to find the input that yields the steepest gradient under the current solution,
and then added it to the set $\{w^s\}$ (which was initialized with 15 random points).
Although L-BFGS is only a local solver,
this greedy approach empirically reduces the number of samples by an order of magnitude.
See the empirical convergence results in \autoref{sec:convergence}.
Its theoretical analysis is left for future investigation.
We also applied this greedy approach to \ilip.

\paragraph{Extending binary kernel machines to multiclass}

The standard kernel methods learn a discriminant function $f^c \defeq \sum_a \gamma^c_a k(x^a, \cdot)$ for each class $c \in [10]$,
based on which a large supply of multiclass classification losses can be applied, 
\eg, CS \citep{crammer2001algorithmic} which was used in our experiment.
Since the Lipschitz constant of the mapping from $\{f^c\}$ to a real-valued loss 
is typically at most 1,
it suffices to bound the Lipschitz constant of $x \mapsto (f^1(x),\ldots,f^{10}(x))^\top$ by $\max_x \lambda_{\max} (G(x) G(x)^\top)$, which is upper bounded by
\begin{align}
        \max_{\nbr{\phi}_\Hcal = 1} \lambda_{\max} \rbr4{\sum_{c=1}^{10} G_c^\top \phi \phi^\top G_c}
        ≤  L^2,\label{eq:lip_multiclass}
\end{align}
where $G_c  \defeq  (g^c_1, \cdots\! , g^c_d)$, and 
\begin{align}
    \quad G(x) 
    &\defeq  [\nabla f^1(x), \cdots , \nabla f^{10}(x)]
    \\&= [G_1^\top k(x,\cdot), \cdots, G_{10}^\top k(x,\cdot)].
\end{align}
The last term in \eqref{eq:lip_multiclass} can be approximated using the same technique as in the binary case. Furthermore, the principle can be extended to $\infty$-norm attacks,
whose details are relegated to \autoref{sec:algo_multiclass}.

\paragraph{Parameter selection}
We used the same parameters as in \citet{Anil_SortingOut_2019} for training \parrelu\ and \parmaxmin. 
To defend against $2$-norm attacks, 
we set $L = 100$ for all algorithms.
% The RKHS norm $C$ was set to $10^{2.5}$ so that they had enough expressiveness.
\glip\ achieved high accuracy and robustness on the validation set with bandwidth $\sigma=1.5$ for FashionMNIST and CIFAR-10, and $\sigma=2$ for MNIST. 
To defend against $\infty$-norm attacks, we set $L = 1000$ for all the four methods as in \citet{Anil_SortingOut_2019}. The best $\sigma$ for \glip\ is 1 for all datasets. 
\ilip\ used 5 stacked layers.  

\paragraph{Results}
\autoref{fig:robustness_l2_cw,fig:robustness_linf_cw} show how the test accuracy decays as an increasing amount of perturbation ($\delta$) in $2$-norm and $\infty$-norm norm is added to the test images, respectively.
%
% Clearly \glip\ achieves higher accuracy than both Parseval networks. 
% {\color{red}check}: We also note that the MaxMin activation is not better than ReLU for Parseval network, 
% probably because we used only one hidden layer where ReLU can still preserve certain expressiveness. 
%
Clearly \glip\ achieves higher accuracy and robustness than \parrelu\ and \parmaxmin\ on the three datasets,
under both $2$-norm and $\infty$-norm bounded PGD attacks with C\&W loss. 
In contrast, \ilip\ only performs similarly to \parrelu. 
Interestingly, $2$-norm based \parmaxmin\ are only slightly better than \parrelu\ under $2$-norm attacks,
although the former does perform significantly better under $\infty$-norm attacks.

The results for cross-entropy PGD attacks are deferred to \autoref{fig:robustness_l2_crossent,fig:robustness_linf_crossent} in \autoref{sec:cw_result}.
Here cross-entropy PGD attackers find stronger attacks to Parseval networks but not to our kernel models.
Our \glip\ again significantly outperforms \parmaxmin\ on all the three datasets and under both $2$-norm and $\infty$-norm norms.
% Considerations of no obfuscated gradient \citep{AthCarWag18} are also deferred to \autoref{sec:obfuscate}.
The improved robustness of \glip\ does not seem to be attributed to the obfuscated (masked) gradient \citep{AthCarWag18},
because as shown \autoref{fig:robustness_l2_cw,fig:robustness_linf_cw,fig:robustness_l2_crossent,fig:robustness_linf_crossent},
increased distortion bound does increase attack success, 
and unbounded attacks drive the success rate to very low.
In practice, we also observed that random sampling finds much weaker attacks, and taking 10 steps of PGD is much stronger than one step.

\ificml\else
\paragraph{Visualization}
The gradient with respect to inputs is plotted in \autoref{fig:grad_visualization} for $2$-norm trained \parmaxmin\ and \glip.  
The $i$-th row and $j$-th column corresponds to the targeted attack of turning the original class $j$ into a new class $i$,
hence the gradient is on the cross-entropy loss with class $i$ as the ground truth.
%
% The latter appears to align better with human perception  \citep{TsiSanEngetal18}. 
These two figures also explained why \glip\ is more robust than \parmaxmin: the attacker can easily reduce the targeted cross-entropy loss by following the gradient as shown in \autoref{fig:vis_maxmin}, and hence successfully attack \parmaxmin. 
In contrast, the gradient shown in \autoref{fig:vis_gauss} does not  provide much information on how to flip the class.
\fi

\paragraph{Obfuscated gradient}
To further illustrate the property of \glip\ trained models, we visualised ``large perturbation'' adversarial examples with the $2$-norm norm bounded by 8.
\autoref{fig:crossent_PGD100_Delta6_new} in \autoref{sec:visualization} shows the result of running PGD attack for 100 steps on \glip\ trained model using 
(\textbf{targeted}) cross-entropy approximation.  
On a randomly sampled set of 10 images from MNIST, 
PGD successfully turned all of them into any target class by following the gradient. 
We further ran PGD on C\&W approximation in \autoref{fig:CW_PGD100_Delta8_new}, 
and this \textbf{untargeted} attack succeeds on all 10 images.
In both cases, the final images are quite consistent with human’s perception.

\section{Conclusion}
\ificml
\begin{minipage}{\linewidth}
    Risk minimisation can fail to be optimal when there is some misspecification of the distribution, such as when working with its empirical counterpart. Therefore we must turn to other techniques in order to ensure stability when learning a model. The robust Bayes framework provides a systematic approach to these problems, however it leaves open the choice as to which uncertainty set is most appropriate. We show that in many cases, the popular Lipschitz regularisation corresponds to robust Bayes with a transportation-cost-based uncertainty set.
\end{minipage}
\else
Risk minimisation can fail to be optimal when there is some misspecification of the distribution, such as when working with its empirical counterpart. Therefore we must turn to other techniques in order to ensure stability when learning a model. The robust Bayes framework provides a systematic approach to these problems, however it leaves open the choice as to which uncertainty set is most appropriate. We show that in many cases, the popular Lipschitz regularisation corresponds to robust Bayes with a transportation-cost-based uncertainty set.

To further justify this choice of uncertainty set we have seen that there are strong connections linking the transportation cost uncertainty set to phenomenon of adversarial examples. To do this we have borrowed tools from the nonconvex optimisation literature. In particular the closed convex envelope  appears to be of somewhat novel application in this area. By its introduction we have been able to maintain tractability while making minimal assumptions about the model class or loss function so that this theory can be applied to popular exotic model classes such as universal approximators.\label{thispage}
\fi

% \begin{center}
%   \textbf{\textit{\pageref*{thispage} pages}}
% \end{center}

\bibliography{xinhua_bibfile,zac_bibfile}
\bibliographystyle{icml2020}

\pagestyle{plain}
\appendix
\numberwithin{equation}{section}
\ificml
    \newgeometry{margin=1.875in, footskip=30.0pt}
\fi
\setlength{\parindent}{1.5em}
\onecolumn

\ificml
    \icmltitle{\texorpdfstring{Supplementary Material for\\\thetitle}{Supplementary Material}}
\else
    \begin{center}
        \huge\bf Supplementary Material
    \end{center}
\fi



\section{Preliminaries}

For a topological vector space $X$ we denote by $X^*$ its topological dual. These are in a duality with the pairing $\inp{\marg,\marg}: X\times X^*\to\R$. The weakest topology on $X$ so that $X^*$ is its topological dual is denoted $\wtop{X,X^*}$. The continuous real functions on a topological space $\Omega$ are collected in $\contf(\Omega)$, and the subset of these that are bounded is $\contfb(\Omega)$. For a measure $\mu\in\probm(X)$ and a Borel mapping $f:X\to Y$, the push-forward measure is denoted $f_\#\mu\in\probm(Y)$ where $f_\#\mu(A) \defeq \mu(f^{-1}(A))$ for every Borel $A\subseteq Y$.
    
The \emph{$\epsilon$-subdifferential} of a convex function $f:X\to\Rx$ at a point $x\in X$ is 
\begin{gather}
    \esubdiff f(x) \defeq \setcond{x^*\in X^* }{ \forall{y\in X} \inp{y-x,x^*} -\epsilon ≤ f(y) - f(x)},
\end{gather}
where $\epsilon≥0$. The \emph{Moreau--Rockafellar subdifferential} is $\subdiff f(x) \defeq  \esubdiff_0f(x)$ and satisfies  $\subdiff f(x) = \bigcap_{\epsilon>0}\esubdiff f(x)$. The Legendre--Fenchel conjugate of a function $f:X\to\Rx$ is the function $f^*:X^*\to\Rx$ defined by 
\begin{gather}
    \forall{x^*\in X^*} f^*(x^*) \defeq \sup_{x\in X} \rbr{\inp{x,x^*} - f(x)},
\end{gather}
and satisfies the following Fenchel--Young rule when $f$ is closed convex
\begin{gather}
    \forall{x\in f^{-1}(\R)}\forall{x^*\in\esubdiff f(x)} f(x) + f^*(x^*) - \inp{x,x^*} ≤ \epsilon. \label{eq:generalised_fy}
\end{gather}
A coupling function $c:X\times X\to\Rx$ has an associated conjugacy operation with 
\begin{gather}
    f^{c}(x) \defeq \sup_{y\in X}\rbr\big{f(y) - c(x,y)},\label{eq:c_conjugate}
\end{gather}
for any function $f:X\to\Rx$. The \emph{indicator function} of a set $A\subseteq X$ is $\ind_A(x) \defeq 0$ for $x\in A$ and $\ind_A(x)\defeq \infty$ for $x\notin A$.

\section{Technical results on distributional robustness and regularisation}\label{sec:technical_appendix}

\printproofs

\section{Proofs and additional results on the Lipschitz regularisation of kernel methods}
\label{sec:app_kernel}

\subsection{Random sampling requires exponential cost}
\label{sec:app_exp_sample}

The most natural idea of leveraging the samples is to add the constraints $\nbr{g(w^s)} \le L$.
For Gaussian kernel, we may sample from $\Ncal(\zero, \sigma^2 I)$ while for inverse kernel we may sample uniformly from $B$.
This leads to our training objective:
\begin{align}
\min_{f \in \Hcal} \quad &\frac{1}{l} \sum_{i=1}^l \text{loss}(f(x^i), y^i) + \frac{\lambda}{2} \nbr{f}_\Hcal^2 \qquad
s.t. \qquad \nbr{g(w^s)} \le L, \quad \forall s \in [n].
\end{align}

Unfortunately, this method may require $O(\frac{1}{\epsilon^d})$ samples to guarantee $\sum_j \nbr{g_j}^2_\Hcal \le L^2 + \epsilon$ w.h.p.
This is illustrated in \autoref{fig:exp_sample},
where $k$ is the polynomial kernel with degree 2 whose domain $ X$ is the unit ball $B$,
and $f(x) = \frac{1}{2} (v^\top x)^2$.
We seek to test whether the gradient $g(x) = (v^\top x) v$ has norm bounded by 1 for all $x \in B$,
and we are only allowed to test whether $\nbr{g(w^s)} \le 1$ for samples $w^s$ that are drawn uniformly at random from $B$.
This is equivalent to testing $\nbr{v}\le 1$, 
and to achieve it at least one $w^s$ must be from the $\epsilon$ ball around $v / \nbr{v}$ or $-v / \nbr{v}$, intersected with $B$.
But the probability of hitting such a region decays exponentially with the dimensionality $d$.

The key insight from the above counter-example is that 
in fact $\nbr{v}$ can be easily computed by $\sum_{s=1}^d (v^\top \wtil_s)^2$,
where $\{\wtil^s\}_{s=1}^d$ is the \emph{orthonormal} basis computed from the Gram–Schmidt process on $d$ random samples $\{w^s\}_{s=1}^d$ ($n=d$).
With probability 1, $n$ samples drawn uniformly from $B$ must span $\RR^d$ as long as $n \ge d$, \ie, $\rank(W) = d$ where $W = (w^1, \ldots, w^n)$.
The Gram–Schmidt process can be effectively represented using a pseudo-inverse matrix (allowing $n > d$) as 
\begin{align}
\label{eq:pseudo_inverse_W}
\nbr{v}_2 = \nbr{(W^\top W)^{-1/2} W^\top v}_2,
\end{align}
where $(W^\top W)^{-1/2}$ is the square root of the pseudo-inverse of $W^\top W$.
This is exactly the intuition underlying the Nystr\"om approximation that we will leveraged.

\begin{figure}
	\centering
	\includegraphics[width=0.31\textwidth]{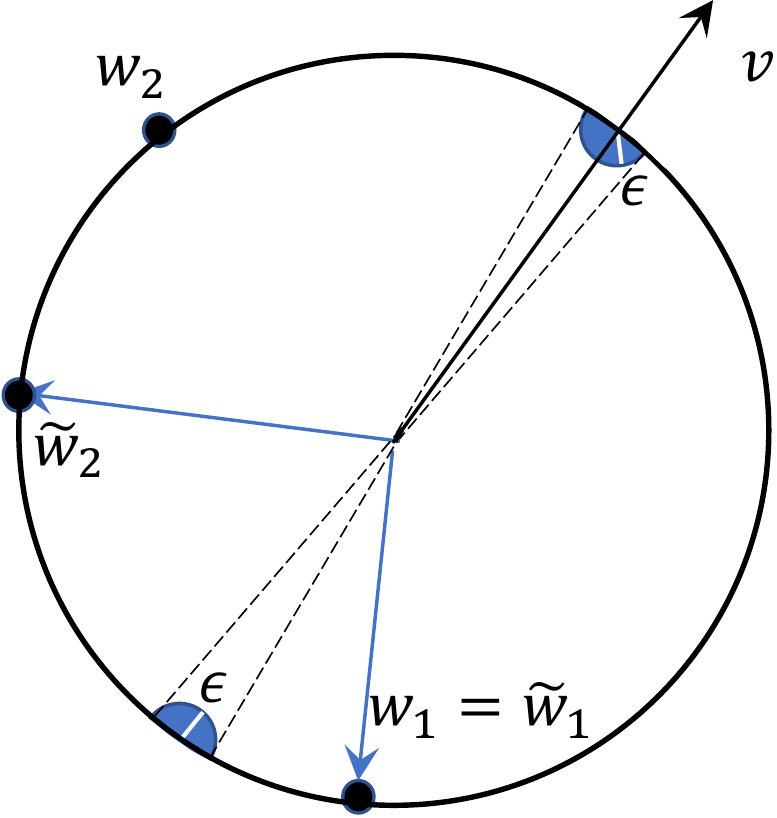}
	\caption{Suppose we use a polynomial kernel with degree 2, 
		and $f(x) = \frac{1}{2} (v^\top x)^2$ for $x \in B$.  
		Then $g(x) = (v^\top x) v$.
		If we want to test whether $\sup_{x \in B} \nbr{g(x)}_2 \le 1$ by evaluating $\nbr{g(w)}_2$ on $w$ that is randomly sampled from $B$ such as $w_1$ and $w_2$,
		we must sample within the $\epsilon$ balls around the intersection of $B$ and the ray along $v$ (both directions).  See the blue shaded area.
		The problem, however, becomes trivial if we use the orthonormal basis 
		$\{\wtil_1, \wtil_2\}$.}
	\label{fig:exp_sample}
\end{figure}

\subsection{Spectrum of Kernels}

%Our sampling based approach is essentially Nystr\"om approximation,
%and we extend it beyond the conventional use of approximating kernel matrices only.
%To ensure $\nbr{g}_\infty \le L$ for $g \in \Hcal$ of inverse kernel, 
%it suffices to ensure $\nbr{g}_\Hcal \le L$ because in that case
%$\abr{g(x)} \le \nbr{k(x,\cdot)}_\Hcal \nbr{g}_\Hcal \le L$.
%
%To bound $\nbr{g}_\Hcal$, we resort to the orthonormal bases of $\Hcal$.
%We now state our result on sample complexity to achieve \eqref{eq:bounded_g_rkhsnorm}, after introducing the basic constructs in kernel spectrum.
%
Let $k$ be a continuous kernel on a compact metric space $ X$, 
and $\mu$ be a finite Borel measure on $ X$ with $supp[\mu] =  X$.
We will re-describe the following spectral properties in a more general way than in \autoref{sec:kernel}.
Recall \textcite[\S4]{steinwart2008support} that the integral operator for $k$ and $\mu$ is defined by 
\begin{align}
\label{eq:def_Tk_
app}
&T_k \ = \ I_k \circ S_k: \leb_2( X, \mu) \to \leb_2( X, \mu) \\
\text{where } \ &S_k \ : \ L_2 ( X,\mu) \to \contf( X), \quad (S_k f)(x) \ = \ \int k(x,y) f(y) d \mu(y), \quad f \in \leb_2( X, \mu),\\
&I_k \ : \ \contf( X) \hookrightarrow \leb_2( X,\mu), \ \text{inclusion operator}.
\end{align}
By the spectral theorem, 
if $T_k$ is compact, 
then there is an at most countable orthonormal set (ONS) $\{\etil_j\}_{j \in J}$ of $\leb_2( X, \mu)$ and $\{\lambda_j\}_{j \in J}$ with $\lambda_1 \ge \lambda_2 \ge \ldots > 0$ such that 
\begin{align}
T f = \sum_{j \in J} \lambda_j \inner{f}{\etil_j}_{\leb_2( X,\mu)} \etil_j, \qquad f \in \leb_2( X, \mu).
\end{align}
In particular, we have $\inner{\etil_i}{\etil_j}_{\leb_2( X,\mu)} = \delta_{ij}$ (i.e., equals 1  if $i=j$, and 0 otherwise),
and $T \etil_i = \lambda_i \etil_i$.
Since $\etil_j$ is an equivalent class instead of a single function, 
we assign a set of continuous functions $e_j = \lambda_j^{-1} S_k \etil_j \in \contf( X)$, which clearly satisfies
\begin{align}
\inner{e_i}{e_j}_{\leb_2( X,\mu)} = \delta_{ij}, \quad T e_j = \lambda_j e_j.
\end{align}
We will call $\lambda_j$ and $e_j$ as eigenvalues and eigenfunctions respectively,
and $\{e_j\}_{j \in J}$ clearly forms an ONS.
By Mercer's theorem, 
\begin{align}
\label{thm:mercer}
k(x,y) = \sum_{j \in J} \lambda_j e_j(x) e_j(y),
\end{align}
and all functions in $\Hcal$ can be represented by $\sum_{j \in J} a_j e_j$ where $\{a_j /\sqrt{\lambda_j}\} \in \ell^2(J)$.
The inner product in $\Hcal$ is equivalent to $\inner{\sum_{j \in J} a_j e_j}{\sum_{j \in J} b_j e_j}_\Hcal = \sum_{j \in J} a_j b_j / \lambda_j$.
Therefore it is easy to see that 
\begin{align}
\varphi_j \defeq \sqrt{\lambda_j} e_j, \qquad j \in J
\end{align}
is an orthonormal basis of $\Hcal$, with
Moreover, for all $f \in \Hcal$ with $f = \sum_{j \in J} a_j e_j$,
we have
$\inner{f}{e_j}_\Hcal = a_j / \lambda_j$,
$\inner{f}{\varphi_j}_\Hcal = a_j / \sqrt{\lambda_j}$,
and
\begin{align}
f = \sum_j \inner{f}{\varphi_j}_\Hcal \varphi_j 
= \sum_j \sqrt{\lambda_j} \inner{f}{e_j}_\Hcal \varphi_j 
= \sum_j \lambda_j \inner{f}{e_j}_\Hcal e_j.
\end{align} 
Most kernels used in machine learning are infinite dimensional, 
\ie, $J = \NN$. 
For convenience, we define $\Phi_m \defeq (\varphi_1, \ldots, \varphi_m)$
and $\Lambda_m = \diag(\lambda_1, \ldots, \lambda_m)$.

\subsection{General sample complexity and assumptions on the product kernel}
\label{sec:complexity}

In this section, we first consider kernels $k_0$ with \textbf{scalar input}, 
\ie, $ X_0 \subseteq \RR$.
Assume there is a measure $\mu_0$ on $ X_0$.
This will serve as the basis for the more general product kernels in the form of
$k(x, y) = \prod_{j=1}^d k_0(x_j, y_j)$
defined over $ X_0^d$.

With Assumptions \ref{assum:tail_bnd} and \ref{assum:uniform_bnd},
we now state the formal version of \autoref{thm:sample_complexity_informal} by first providing the sample complexity for approximating the partial derivatives.
In the next subsection, we will examine how three different kernels satisfy/unsatisfy the Assumptions \ref{assum:tail_bnd} and \ref{assum:uniform_bnd}, and what the value of $N_\epsilon$ is.
For each case, we will specify $\mu_0$ on $ X_0$, 
and the measure on $ X_0^d$ is trivially $\mu = \mu_0^d$.

\begin{theorem}
	\label{thm:sample_prod_ker_1d}
	Suppose $\{w^s\}_{s=1}^n$ are drawn iid from $\mu_0$ on $ X_0$,
	where $\mu_0$ is the uniform distribution on $[-v/2, v/2]$ for periodic kernels or periodized Gaussian kernels.
	Let $Z \defeq (k_0(w^1,\cdot), k_0(w^2,\cdot), \ldots, k_0(w^n, \cdot))$, and 
	$g_1 = \frac{1}{l} \sum_{a=1}^l \gamma_a g^a_1$: $ X^d_0 \to \RR$, where 
	$\nbr{\vga}_\infty \le c_1$ and
	\begin{align}
	g^a_1(y) = \partial^{0,1} k(x^a, y) 
	= h^a_1(y_1) \prod_{j=2}^d k_0(x^a_j, y_j) \quad \text{with} \quad
	h^a_1(\cdot) \defeq \partial^{0,1}k_0(x^a_1, \cdot).
	\end{align}
	Given $\epsilon \in (0, 1]$, let $\Phi_m = (\varphi_1, \ldots \varphi_m)$ 
	where $m = N_\epsilon$.
	Then with probability $1 - \delta$, the following holds when the sample size 
	$n = \max(N_\epsilon, \frac{5}{3 \epsilon^2} N_\epsilon Q^2_\epsilon \log \frac{2N_\epsilon}{\delta})$:
	\begin{align}
	\label{eq:bound_sample_complexity_1d}
	\nbr{g_1}_\Hcal^2
	&\le 
	\frac{1}{l^2} \vga^\top K_1 \vga + 3 c_1 \rbr{1 + 2\sqrt{N_\epsilon} M_\epsilon} \epsilon, \\
	\where
	(K_1)_{a,b} &= (h^a_1)^\top Z (Z^\top Z)^{-1} Z^\top h^b_1 \prod_{j=2}^d k_0(x^a_j, x^b_j).
	\end{align}
\end{theorem}

Then we obtain the formal statement of sample complexity, as stated in the following corollary, by combining all the coordinates from \autoref{thm:sample_prod_ker_1d}.

\begin{corollary}
	\label{cor:sample_prod_ker_1d}
	Suppose all coordinates share the same set of samples $\{w^s\}_{s=1}^n$.
	Applying the results in \eqref{eq:bound_sample_complexity_1d} for coordinates from 1 to $d$ and using the union bound, we have that with sample size 
	$n=\max(N_\epsilon, \frac{5}{3 \epsilon^2} N_\epsilon Q^2_\epsilon \log \frac{2N_\epsilon}{\delta})$, 
	the following holds with probability $1-d\delta$,
	\begin{align}
	\label{eq:bound_sample_complexity}
	\lambda_{\max}(G^\top G)
	\le 
	\lambda_{\max}(\Ptil_G)
	+ 3 c_1 \rbr{1 + 2 \sqrt{N_\epsilon} M_\epsilon} \epsilon.
	\end{align}
	Equivalently, if $N_\epsilon$, $M_\epsilon$ and $Q_\epsilon$ are constants or poly-log terms of $\epsilon$ which we treat as constant, 
	then to ensure 
	$\lambda_{\max}(G^\top G)
	\le 
	\lambda_{\max}(\Ptil_G) + \epsilon$ with probability $1-\delta$,
	the sample size needs to be
	\begin{align}
	n = \frac{15}{\epsilon^2} c_1^2 \rbr{1+2 \sqrt{N_\epsilon} M_\epsilon}^2 N_\epsilon Q_\epsilon^2 \log \frac{2d N_\epsilon}{\delta}.
	\end{align}
\end{corollary}

\begin{remark}
	The first term on the right-hand side of \eqref{eq:bound_sample_complexity} is explicitly upper bounded by $L^2$ in our training objective.
	%The last term is the tail energy, which we will bound by Assumption \ref{assum:tail_bnd}.
	In the case of \autoref{thm:Meps_period},
	the values of $Q_\epsilon$, $N_\epsilon$, and $M_\epsilon$ lead to a $\Otil(\frac{1}{\epsilon^2})$ sample complexity.
	If we further zoom into the dependence on the period $v$,
	then note that $N_\epsilon$ is almost a universal constant while $M_\epsilon =\frac{\sqrt{2}\pi}{v}(N_\epsilon-1)$.
	So overall, $n$ depends on $v$ by $\frac{1}{v^2}$.
	This is not surprising because smaller period means higher frequency,
	hence more samples are needed.
\end{remark}

\begin{remark}
	Corollary \ref{cor:sample_prod_ker_1d} postulates that all coordinates share the same set of samples $\{w^s\}_{s=1}^n$.  
	When coordinates differ in their domains, we can draw different sets of samples for them.
	The sample complexity hence grows by $d$ times as we only use a weak union bound.
	More refined analysis could save us a factor of $d$ as these sets of samples are independent of each other.
\end{remark}

%We now prove \autoref{thm:sample_prod_ker_1d} by using the Assumptions \ref{assum:tail_bnd} and \ref{assum:uniform_bnd}.
%We will assume that $ X =  X_0^d$ and $k(x,y) = \prod_i k_0(x_i,y_i)$.

\begin{proof}[Proof of \autoref{thm:sample_prod_ker_1d}]
	Let $\epsilon'\defeq\rbr{1 + 2 \sqrt{m} M_\epsilon} \epsilon$.
	Since 
	\begin{gather}
		\inner{g^a_1}{g^b_1}_{\Hcal}  =  \inner{h^a_1}{h^b_1}_{\Hcal_0} \prod_{j=2}^d k_0(x^a_j, x^b_j)
	\end{gather}
	 and $\abr{k_0(x^a_j, x^b_j)} \le 1$,
	it suffices to show that for all $a, b \in [l]$,
	\begin{align}
	\abr{\inner{h^a_1}{h^b_1}_{\Hcal_0} - (h^a_1)^\top Z (Z^\top Z)^{-1} Z^\top h^b_1} \le 3\epsilon'.
	\end{align}
	Towards this end, it is sufficient to show that for any 
	$h(\cdot) = \theta_x \partial^{0,1}k_0(x,\cdot) + \theta_y \partial^{0,1}k_0(y,\cdot)$ 
	where $x, y \in  X_0$ and $\abr{\theta_x} + \abr{\theta_y} \le 1$, we have
	\begin{align}
	\label{eq:approx_H0_norm}
	\abr{h^\top Z (Z^\top Z)^{-1} Z^\top h - \nbr{h}_{\Hcal_0}^2} \le \epsilon'.
	\end{align}
	This is because, if so, then
	\begin{align}
		\MoveEqLeft\abr{\inner{h^a_1}{h^b_1}_{\Hcal_0} - (h^a_1)^\top Z (Z^\top Z)^{-1} Z^\top h^b_1} \\
		&=\Bigl|\frac{1}{2} \rbr{\nbr{h^a_1 + h^b_1}_{\Hcal_0}^2 - \nbr{h^a_1}_{\Hcal_0}^2- \nbr{h^b_1}_{\Hcal_0}^2} \\
		&\qquad
		 - \frac{1}{2} \Bigl[(h^a_1 + h^b_1)^\top Z (Z^\top Z)^{-1} Z^\top (h^a_1 + h^b_1)\\
		&\qquad		
		- (h^a_1)^\top Z (Z^\top Z)^{-1} Z^\top h^a_1 - (h^b_1)^\top Z (Z^\top Z)^{-1} Z^\top h^b_1 \Bigr] \Bigr|\\
		&≤\frac{1}{2} \rbr{4\epsilon' + \epsilon' + \epsilon'} \\
		&= 3\epsilon'.
	\end{align}

	The rest of the proof is devoted to \eqref{eq:approx_H0_norm}.
	Since $n \ge m$, %$\Phi_m^\top Z$ has rank $m$ with probability 1.
	the SVD of $\Lambda_m^{-1/2} \Phi_m^\top Z$ can be written as $U \Sigma V^\top$,
	where $U U^\top = U^\top U = V^\top V = I_m$ ($m$-by-$m$ identity matrix),
	and $\Sigma = \diag(\sigma_1, \ldots, \sigma_m)$. % with all $\sigma_i > 0$.
	Define 
	\begin{align}
	\val = n^{-1/2} V U^\top \Lambda_m^{-1/2} \Phi_m^\top h.
	\end{align}
	Consider the optimization problem $o(\val) \defeq \frac{1}{2} \nbr{Z \val - h}_{\Hcal_0}^2$.
	It is easy to see that its minimal objective value is 
	$o^* \defeq \frac{1}{2} \nbr{h}_{\Hcal_0}^2 - \frac{1}{2} h^\top Z (Z^\top Z)^{-1} Z^\top h$.
	So
	\begin{align}
	0 \le 2 o^* =
	\nbr{h}^2_{\Hcal _0} - h^\top Z (Z^\top Z)^{-1} Z^\top h 
	\le 2 o(\val). % \\
	%		&\le \nbr{(Z^\top Z)^{-\frac{1}{2}}  Z^\top G}_F^2 + 2 \sum_{j=1}^d o_j(\val_j).
	\end{align}
	
	Therefore to prove \eqref{eq:approx_H0_norm}, 
	it suffices to bound $o(\val) = \nbr{Z \val - h}_{\Hcal_0}$.
	Since $\sqrt{n} \Phi_m \Lambda^{1/2}U V^\top \val = \Phi_m \Phi_m^\top h$,
	we can decompose $\nbr{Z \val - h}_{\Hcal_0}$ by
	\begin{gather}
		\begin{aligned}
			\nbr{Z \val - h}_{\Hcal_0} 
			&\le \nbr{(Z - \Phi_m \Phi_m^\top Z) \val}_{\Hcal_0} 
			\\&\qquad + \nbr{(\Phi_m \Phi_m^\top Z - \sqrt{n} \Phi_m \Lambda_m^{1/2}U V^\top) \val}_{\Hcal_0} 
			\\&\qquad + \nbr{\Phi_m \Phi_m^\top h - h}_{\Hcal_0}.
		\end{aligned}
		\label{eq:Zalpha_minus_h}
	\end{gather}
	The last term $\nbr{\Phi_m \Phi_m^\top h - h}_{\Hcal_0}$ is clearly below $\epsilon$ because by Assumption \ref{assum:tail_bnd} and $m=N_\epsilon$
	\begin{align}
		\nbr{\Phi_m \Phi_m^\top h - h}_{\Hcal_0}  
		&≤\abr{\theta_x} \nbr{\Phi_m \Phi_m^\top \partial^{0,1}k_0(x,\cdot) - \partial^{0,1}k_0(x,\cdot) }_{\Hcal_0}
		\\&\qquad+ \abr{\theta_y} \nbr{\Phi_m \Phi_m^\top \partial^{0,1} k_0(y,\cdot) - \partial^{0,1}k_0(y,\cdot) }_{\Hcal_0}
		\\&≤ (|\theta_x|+|\theta_y|) \epsilon 
		\\&≤ \epsilon.
	\end{align}
	We will next bound the first two terms on the right-hand side of \eqref{eq:Zalpha_minus_h}.
	
	(i) By Assumption \ref{assum:tail_bnd}, 
	$\nbr{k_0(w^s, \cdot) - \Phi_m \Phi_m^\top k_0(w^s, \cdot)}_{\Hcal_0} \le \epsilon$,
	hence
	\begin{gather}
		\nbr{(Z - \Phi_m \Phi_m^\top Z) \val}_{\Hcal_0} \le \epsilon \sqrt{n} \nbr{\val}_2.
	\end{gather}

	To bound $\nbr{\val}_2$, note all singular values of $VU^\top$ are 1, and so Assumption \ref{assum:uniform_bnd} implies that for all $i \in [m]$,
	\begin{align}
		\abr{\lambda_j^{-1/2} \inner{\varphi_j}{h}_{\Hcal_0}} 
		&=\abr{\inner{e_j}{h}_{\Hcal_0}} 
		\\&= \abr{\inner{e_j}{\theta_x \partial^{0,1}k_0(x,\cdot) + \theta_y \partial^{0,1}k_0(y,\cdot)}_{\Hcal_0}} 
		\\&≤ \sup_{x \in  X} \abr{\inner{e_j}{\partial^{0,1} k(x, \cdot)}_{\Hcal_0}}
		\\&\le M_\epsilon.\label{eq:bound_eigj}
	\end{align}
	As a result,  
	\begin{align}
	\nbr{(Z - \Phi_m \Phi_m^\top Z) \val_j}_{\Hcal_0} 
	&\le \epsilon n^{1/2} \cdot n^{-1/2} \nbr{\Lambda_m^{-1/2} \Phi_m^\top h} 
	\le \epsilon \sqrt{m} M_\epsilon.
	\end{align}
	
	(ii) We first consider the concentration of the matrix 
	\begin{gather}
		R\defeq \frac{1}{n} \Lambda^{-1/2}_m \Phi_m^\top Z Z^\top \Phi_m \Lambda^{-1/2}_m \in \RR^{m \times m}.
	\end{gather}
	Clearly,
	\begin{align}
	\expunder{\{w_s\}}[R_{ij}] = \expunder{\{w_s\}} \sbr{\frac{1}{n} \sum_{s=1}^n e_i(w_s) e_j(w_s)} = \int e_i(x) e_j(x) \rmd \mu(x) = \delta_{ij}.
	\end{align}
	
	By matrix Bernstein theorem \citep[Theorem 1.6.2]{Tropp15},
	we have 
	\begin{gather}
		\Pr \rbr{\nbr{R - I_m}_{sp} \le \epsilon} \ge 1-\delta
	\end{gather}
	when $n \ge O(.)$.
	This is because 
	\begin{gather}
		\nbr{(e_1(x), \ldots, e_m(x))}^2 \le m Q^2_\epsilon,\quad 
		\nbr{\EE_{\{w_s\}}[RR^\top]}_{sp} \le m Q_\epsilon^2 /n,
	\end{gather}
	and
	\begin{align}
		\Pr\rbr{\nbr{R - I_m}_{sp} \le \epsilon} 
		&\ge 1- 2m \exp \rbr{\frac{-\epsilon^2}{\frac{m Q_\epsilon^2}{n} \rbr{1+ \frac{2}{3}\epsilon}}}
		\\&\ge 1- 2m \exp \rbr{\frac{-\epsilon^2}{\frac{5m Q_\epsilon^2}{3n} }} 
		\\&\ge 1 - \delta,		
	\end{align}
	where the last step is by the definition of $n$.
	Since $R = \frac{1}{n} U \Sigma^2 U^\top$, this means with probability $1-\delta$,
	$\nbr{\frac{1}{n} U \Sigma^2 U^\top - I_m}_{sp} \le \epsilon$.
	So for all $i \in [m]$,
	\begin{align}
	\label{eq:bound_sigma}
	\abr{\frac{1}{n} \sigma^2_i - 1} \le \epsilon
	\implies
	\abr{\frac{1}{\sqrt{n}} \sigma_i - 1} 
	< \epsilon \abr{\frac{1}{\sqrt{n}} \sigma_i + 1}^{-1} \le \epsilon.
	\end{align}
	Moreover, $\lambda_1 \le 1$ since $k_0(x,x) = 1$.
	It then follows that
	\begin{align}
	&\nbr{(\Phi_m \Phi_m^\top Z - \sqrt{n} \Phi_m \Lambda_m^{1/2}U V^\top) \val}_{\Hcal_0} \\
	= &\nbr{\Phi_m \Lambda_m^{1/2} U \Sigma V^\top \frac{1}{\sqrt{n}} V U^\top \Lambda_m^{-1/2} \Phi_m^\top h  - 
		\sqrt{n} \Phi_m \Lambda_m^{1/2} U V^\top \frac{1}{\sqrt{n}} V U^\top \Lambda_m^{-1/2} \Phi_m^\top h}_{\Hcal_0} \\
	= & \nbr{\Lambda_m^{1/2} U \rbr{\frac{1}{\sqrt{n}} \Sigma - I_m} U^\top \Lambda_m^{-1/2} \Phi_m^\top h}_2 \qquad (\text{because } \Phi_m^\top \Phi_m = I_m) \\
	\le & \sqrt{\lambda_1} \max_{i\in[m]} \abr{\frac{1}{\sqrt{n}} \sigma_i - 1} \nbr{\Lambda_m^{-1/2} \Phi_m^\top h}_2 \\
	\le & \epsilon \sqrt{m} M_\epsilon
	\qquad (\text{by } \eqref{eq:bound_sigma}, \eqref{eq:bound_eigj},  \text{ and } \lambda_1 \le 1).
	\end{align}
	Combining (i) and (ii), we arrive at the desired bound in \eqref{eq:bound_sample_complexity_1d}.
\end{proof}

\begin{proof}[Proof of Corollary \ref{cor:sample_prod_ker_1d}]
    Since $\Ptil_G$ approximates $G^\top G$ only on the diagonal,
    $\Ptil_G - G^\top G$ is a diagonal matrix which we denote as 
    $\diag(\delta_1, \ldots, \delta_d)$.
    Let $\uvec \in \RR^d$ be the leading eigenvector of $\Ptil_G$.
    Then
    \begin{align}
        \lambda_{\max}(\Ptil_G) - \lambda_{\max}(G^\top G) 
        &\le \uvec^\top \Ptil_G \uvec - \uvec^\top G^\top G \uvec
        = \uvec^\top (\Ptil_G - G^\top G) \uvec
        = \sum_j \delta_j \uvec_j^2 \\
       \text{(by \eqref{eq:bound_sample_complexity_1d})} \quad &\le 3 c_1 \rbr{1 + 2\sqrt{N_\epsilon} M_\epsilon} \epsilon.
    \end{align}
    The proof is completed by applying the union bound and rewriting the results.
\end{proof}

\subsection{Case 1: Checking Assumptions \ref{assum:tail_bnd} and \ref{assum:uniform_bnd} on periodic kernels}
\label{sec:assump_check_periodic}

Periodic kernels on $ X_0 \defeq \RR$ are translation invariant, 
and can be written as $k_0(x,y) = \kappa(x-y)$ where 
$\kappa: \RR \to \RR$ is 
a) periodic with period $v$;
b) even, with $\kappa(-t) = \kappa(t)$;
and
c) normalized with $\kappa(0) = 1$.
A general treatment was given by \cite{WilSmoSch01},
and an example was given by David MacKay in \cite{MacKay98}:
\begin{align}
\label{eq:kernel_MacKay}
k_0(x,y) = \exp \rbr{-\frac{1}{2\sigma^2} \sin \rbr{\frac{\pi}{v} (x-y)}^2}.
\end{align}
%
%By properly scaling the input, we can set the period to $\pi$ and simplify the kernel as
%$k_0(x,y) = \exp(-\frac{1}{2\sigma^2} \sin(x-y)^2)$.
We define $\mu_0$ to be a uniform distribution on $[-\frac{v}{2}, \frac{v}{2}]$,
and let $\omega_0 = 2 \pi / v$.

%Then clearly the eigenfunctions of the integral operator $T_k$ are proportional to constant, $\cos(j \omega_0 x)$, and $\sin(j \omega_0 x)$ for integer $j \ge 1$,
Since $\kappa$ is symmetric, we can simplify the Fourier transform of $\kappa(t) \delta_v(t)$, where $\delta_v(t) = 1$ if $t \in [-v/2, v/2]$, and 0 otherwise:
\begin{align}
F(\omega) = \frac{1}{\sqrt{2\pi}} \int_{-v/2}^{v/2} \kappa(t) \cos(\omega t) \rmd t.
\end{align}
It is now easy to observe that thanks to periodicity and symmetry of $\kappa$, for all $j \in \ZZ$,
\begin{align}
&\frac{1}{v} \int_{-v/2}^{v/2} k_0(x,y) \cos(j \omega_0 y) \rmd y 
= \frac{1}{v} \int_{-v/2}^{v/2} \kappa(x - y) \cos(j \omega_0 y) \rmd y \\
= &\frac{1}{v} \int_{x-v/2}^{x+v/2} \kappa(z) \cos(j \omega_0 (x-z)) \rmd z  \quad \text{(note } \cos(j \omega_0 (x-z)) \text{ also has period } v)  \\
= &\frac{1}{v} \int_{-v/2}^{v/2} \kappa(z) [\cos(j \omega_0 x) \cos (j \omega_0 z) + \sin(j \omega_0 x) \sin (j \omega_0 z)) \rmd z \quad \text{(by periodicity)} \\
\
=& \frac{1}{v} \cos(j \omega_0 x) \int_{-v/2}^{v/2} \kappa(z) \cos (j \omega_0 z) \rmd z \quad \text{(by symmetry of } \kappa) \\
=& \frac{\sqrt{2\pi}}{v}  F(j \omega_0) \cos(j \omega_0 x).
\end{align}
And similarly,
\begin{align}
%\int_{-v/2}^{v/2} k_0(x,y) \rmd y &= \lambda_0 \quad (\text{independent of } x) \\
%\label{eq:period_cos}
%\frac{1}{v} \int_{-v/2}^{v/2} k_0(x,y) \cos(j \omega_0 y) \rmd y &= \lambda_j \cos(j \omega_0 x) \\
\label{eq:period_sin}
\frac{1}{v} \int_{-v/2}^{v/2} k_0(x,y) \sin(j \omega_0 y) \rmd y &= \frac{\sqrt{2\pi}}{v}  F(j \omega_0) \sin(j \omega_0 x).
\end{align}

Therefore the  eigenfunctions of the integral operator $T_k$ are
\begin{equation}
e_0(x) = 1, \quad \ \
e_j(x) \defeq \sqrt{2} \cos(j \omega_0 x), \qquad 
e_{-j}(x) \defeq \sqrt{2} \sin(j \omega_0 x) \quad (j \ge 1)
\end{equation}
and the eigenvalues are $\lambda_j = \frac{\sqrt{2\pi}}{v}  F(j \omega_0)$ for all $j \in \ZZ$ with $\lambda_{-j} = \lambda_j$.
An important property our proof will rely on is that 
\begin{align}
e'_j(x) = -j \omega_0 e_{-j}(x), \quad \text{for all } j \in \ZZ.
\end{align}
Applying Mercer's theorem in \eqref{thm:mercer} and noting $\kappa(0)=1$,
we derive $\sum_{j \in \ZZ} \lambda_j = 1$.

\iffalse
Extension to multi-dimension is straightforward with $k_0(x,y) \defeq \prod_i k_0(x_i, y_i)$ where we overloaded the symbol $k_0$ because whether it is on $\RR$ or $\RR^d$ is clear from the context.
Naturally $\mu$ is now the uniform distribution in the hypercube $[-\frac{v}{2}, \frac{v}{2}]^d$.
The eigensystem can be expressed using the standard multi-index notation with 
$\val = (\alpha_1, \ldots, \alpha_d) \in \ZZ^d$ and 
$\abr{\val} \defeq \sum_{i=1}^d \abr{\alpha_i}$. 
Then the eigenvalues are 
$\lambda_\val = \prod_i \lambda_{\alpha_i}$ 
and the eigenfunctions are
$e_{\val}(x) = \prod_i e_{\alpha_i}(x_i)$.
We also overloaded the symbols $\lambda$ and $e$ because the context is clear whether the index is single or multiple.
\fi

\paragraph{Checking the Assumptions \ref{assum:tail_bnd} and \ref{assum:uniform_bnd}.}
The following theorem summarizes the assumptions and conclusions regarding the 
satisfaction of Assumptions \ref{assum:tail_bnd} and \ref{assum:uniform_bnd}.
Again we focus on the case of $ X \subseteq \RR$.

\begin{theorem}
	\label{thm:Meps_period}
	Suppose the periodic kernel with period $v$ has eigenvalues $\lambda_j$ that satisfies
	\begin{align}
	\label{eq:eigenvalue_cond_period_2}
	\lambda_j (1+j)^2 \max(1, j^2) (1 + \delta(j \ge 1)) &\le c_6 \cdot c_4^{-j}, \quad \text{for all } j \ge 0,
	\end{align} 
	where $c_4 > 1$ and $c_6 > 0$ are universal constants.
	Then Assumption \ref{assum:tail_bnd} holds with 
	\begin{align}
	\label{eq:Meps_periodic}
	N_\epsilon = 1+2 \floor{n_\epsilon}, \where 
	n_\epsilon \defeq 
	\log_{c_4} \rbr{\frac{2.1 c_6}{\epsilon^2} \max \rbr{1, \frac{v^2}{4\pi^2}}}.
	\end{align}
	In addition, Assumption \ref{assum:uniform_bnd} holds with
	$Q_\epsilon = \sqrt{2}$ 
	and
	$M_\epsilon = \frac{2 \sqrt{2} \pi}{v} \floor{n_\epsilon} 
	= \frac{\sqrt{2} \pi}{v} (N_\epsilon - 1)$.
\end{theorem}

For example, if we set $v = \pi$ and $\sigma^2 = 1/2$ in the kernel in \eqref{eq:kernel_MacKay},
elementary calculation shows that the condition \eqref{eq:eigenvalue_cond_period_2} is satisfied with
$c_4 = 2$ and $c_6 = 1.6$.

\begin{proof}[Proof of \autoref{thm:Meps_period}]
	First we show that $h(x) \defeq \partial^{0,1} k_0(x_0, x)$ is in $\Hcal_0$ for all $x_0 \in  X_0$.
	Since $k_0(x_0,x) = \sum_{j \in \ZZ} \lambda_j e_j(x_0) e_j(x)$, we derive
	\begin{align}
	\label{eq:expansion_h_periodic}
	h(x) = \sum_{j \in \ZZ} \lambda_j e_j(x_0) \partial^{1} e_{j}(x) 
	= \sum_{j \in \ZZ} \lambda_j e_j(x_0) \rbr{-j \omega_0 e_{-j}(x) }
	= \omega_0 \sum_{j \in \ZZ} \lambda_j j e_{-j}(x_0) e_{j}(x).
	%	&= \omega_0 \sum_{j_1 \in \ZZ} \lambda_{j_1} j_1^2 e_{j_1}(x_{0,1}) e_{j_1}(x_1) \cdot \prod_{i=2}^d \sum_{j \in \ZZ} \lambda_{j} e_{j}(x_{0,i}) e_{j}(x_i) \\
	%	&= \omega_0 \partial^{0,1} k_0(x_{0,1}, x_1) \prod_{i=2}^d k_0(x_{0,i}, x_i) 
	\end{align}
	$h(x)$ is in $\Hcal$ if the sequence $\lambda_j j e_{-j}(x_0) / \sqrt{\lambda_j}$ is square summable.
	This can be easily seen by \eqref{eq:eigenvalue_cond_period_2}:
	\begin{align}
	\omega_0^{-2} \nbr{h}_{\Hcal_0}^2 
	&= \sum_j \lambda_j j^2 e^2_{-j}(x_0) 
	=  \sum_{j \in \ZZ} \lambda_j j^2 e^2_{-j}(x_0)  \\
	&= \sum_{j \in \ZZ} \lambda_j j^2 e^2_{-j}(x_0) 
	= \lambda_0 + 2 \sum_{j \ge 1} j^2 \lambda_j \le \frac{2 c_4 c_5}{c_4 -1}.
	%	&\le 2^d \sum_\val \lambda_\val \alpha_1^2 
	%	\le 2^d \cdot \prod_{i=1}^d \sum_{\alpha_i \in Z} \lambda_{\alpha_i} \alpha_i^2 
	%	\le 2^d \rbr{c_5 \sum_{j \in \ZZ} c_4^{-|j|}}^d   < \infty.
	\end{align}
	
	Finally to derive $N_\epsilon$, we reuse the orthonormal decomposition of $h(x)$ in \eqref{eq:expansion_h_periodic}.
	For a given set of $j$ values $A$ where $A \subseteq \ZZ$,
	we denote as $\Phi_A$ the ``matrix'' whose columns enumerate the $\varphi_j$ over $j \in A$.  
	Let us choose 
	\begin{align}
	A \defeq 
	\cbr{j: \lambda_j \max(1, j^2) (1+j^2) (1 + \delta(j \ge 1))
		\ge \min (1, w_0^{-2}) \frac{\epsilon^2}{2.1} }.
	\end{align}
	If $j \in A$, then $-j \in A$.
	Letting $\NN_0 = \{0, 1, 2, \ldots\}$, we note
	$\sum_{j \in \NN_0} \frac{1}{1+j^2} \le 2.1$.
	So
	\begin{align}
	\nbr{h - \Phi_A \Phi_A^\top h}_{\Hcal_0}^2 
	&= w_0^2 \sum_{j \in \ZZ \backslash A} \lambda_j j^2 e^2_{-j}(x_0) \\
	&= w_0^2 \sum_{j \in \NN_0 \backslash A} \lambda_j j^2 
	\sbr{(e^2_{j}(x) + e^2_{-j}(x)) \delta(j \ge 1) + \delta(j = 0)} \\	
	&= w_0^2 \sum_{j \in \NN_0 \backslash A} \lambda_j j^2 (1+\delta(j\ge 1))\\
	&= w_0^2 \sum_{j \in \NN_0 \backslash A} \cbr{\lambda_j j^2 (1+j^2) (1+\delta(j \ge 1)) \frac{1}{1+j^2}} \\
	&\le \frac{\epsilon^2}{2.1} \sum_{j \in \NN_0} \frac{1}{1+j^2} 
	= \frac{\epsilon^2}{2.1} \sum_{j \in \NN_0} \frac{1}{1+j^2} \le \epsilon^2.
	\end{align}

	Similarly, we can bound $\nbr{k_0(x_0, \cdot) - \Phi_A \Phi_A^\top k_0(x_0,\cdot)}_{\Hcal_0}$ by
	\begin{align}
		\MoveEqLeft\nbr{k_0(x_0, \cdot) - \Phi_A \Phi_A^\top k_0(x_0,\cdot)}_{\Hcal_0}^2 \\
		&=\sum_{j \in \ZZ \backslash A} \lambda_j e^2_{j}(x_0)
		\le \sum_{j \in \ZZ \backslash A} \lambda_j \max(1, j^2) e^2_{j}(x_0) 
		\\
		&=\sum_{j \in \NN_0 \backslash A} \lambda_\val \max(1, j^2)
		[\rbr{e^2_{j}(x) + e^2_{-j}(x)} \delta(j \ge 1) + \delta(j=0)] \\
		&=\sum_{j \in \NN_0 \backslash A} \cbr{\lambda_j \max(1, j^2) 
			(1+j^2) (1+\delta(j \ge 1)) \frac{1}{1+j^2}} \\
		&\le \frac{1}{2.1} \epsilon^2 \sum_{j \in \NN_0} \frac{1}{1+j^2}	
		\\&\le \epsilon^2.
	\end{align}	
	
	To upper bound the cardinality of $A$,
	we consider the conditions for $j \notin A$.
	%	Since 
	%	\begin{align}
	%	\label{eq:eigenvalue_cond_period}
	%	\lambda_j (1+j)^2 \le c_5 \cdot c_4^{-j}, \quad  \text{and} \quad
	%	\lambda_j (1+j)^2 j^2 \le c_6 \cdot c_4^{-j}, \where c_4 = 2, \ c_5 = 0.7, \ c_6 = 1.6,
	%	\end{align} 
	Thanks to the conditions in \eqref{eq:eigenvalue_cond_period_2},
	we know that any $j$ satisfying the following relationship cannot be in $A$:
	\begin{align}
	c_6 \cdot c_4^{-\abr{j}} < \min(1, w_0^{-2}) \frac{\epsilon^2}{2.1}
	\iff
	c_4^{-\abr{j}} < 
	\frac{1}{2.1 \cdot c_6} \min \rbr{1, \frac{4\pi^2}{v^2}} \epsilon^2.
	\end{align}
	So $A \subseteq \cbr{j : \abr{j} \le n_\epsilon}$, which yields the conclusion \eqref{eq:Meps_periodic}.
	Finally $Q_\epsilon  \le \sqrt{2}$, and to bound $M_\epsilon$, we simply reuse \eqref{eq:expansion_h_periodic}.
	For any $j$ with $\abr{j} \le n_\epsilon$,
	\begin{equation*}
	\abr{\inner{h}{e_j}_\Hcal} \le \omega_0 \abr{j e_{-j}(x_0)}
	\le \frac{2\pi}{v} \sqrt{2} \floor{n_\epsilon}
	= \frac{\sqrt{2} \pi}{v} (N_\epsilon - 1).
	\end{equation*}
\end{proof}

%In this case, it is easy to see that regardless of $v$,
%$\lambda_j \le 0.7\cdot c_3^{-|j|}$ 
%and $j^2 \lambda_j \le 0.4 \cdot c_4^{-|j|}$ for all $j \in \ZZ$, 
%where $c_3 = 4$ and $c_4 = 2$.
%
%
%
%
%If $v \ge 4.2 \cdot c_5 = 2.94$, then $N_\epsilon \le d^{1+\ceil{\log_{c_4}  \frac{c_6}{c_5 \epsilon}}}$ which is quasi-polynomial if we take $d$ and $\frac{1}{\epsilon}$ as large variables.
%In general, if the periodic kernel takes other forms than \eqref{eq:kernel_MacKay},
%we can still get the bound $N_\epsilon$ as in \eqref{eq:Meps_periodic},
%as long as the corresponding eigenvalues satisfy the condition \eqref{eq:eigenvalue_cond_period} with different values of $c_4$, $c_5$, and $c_6$.

\subsection{Case 2: Checking Assumptions \ref{assum:tail_bnd} and \ref{assum:uniform_bnd} on Gaussian kernels}
\label{sec:assump_check_gauss}

Gaussian kernels $k(x, y) = \exp(-\nbr{x - y}^2/(2\sigma^2))$ are obviously product kernels with $k_0(x_1,y_1) = \kappa(x_1-y_1) = \exp(-(x_1-y_1)^2/(2\sigma^2))$.
It is also translation invariant.
The spectrum of Gaussian kernel $k_0$ on $\RR$ is known;
see, \eg, Chapter 4.3.1 of \cite{RasWil06} and Section 4 of \cite{ZhuWilRohMor98}.
Let $\mu$ be a Gaussian distribution $\Ncal(0, \sigma^2)$.
Setting $\epsilon^2 = \alpha^2 = (2\sigma^2)^{-1}$ in Eq 12 and 13 of \cite{Fasshauer11},
the eigenvalue and eigenfunctions are (for $j \ge 0$):
\begin{align}
\label{eq:eigenvalue_gauss}
\lambda_j &= c_0^{-j-1/2}, \where c_0 = \frac{1}{2}(3 + \sqrt{5})\\
e_j(x) &= \frac{5^{1/8}}{2^{j/2}} 
\exp\rbr{-\frac{\sqrt{5}-1}{4} \frac{x^2}{\sigma^2}}
\frac{1}{\sqrt{j!}} H_j \rbr{\sqrt[4]{1.25} \  \frac{x}{\sigma}},
\end{align}
where $H_j$ is the Hermite polynomial of order $j$.

Although the eigenvalues decay exponentially fast, 
the eigenfunctions are not uniformly bounded in the $L_\infty$ sense.
%Take $j=1$, and it is easy to see that $\lim_{x \to \infty} \abr{e_1(x)} = \infty$,
Although the latter can be patched if we restrict $x$ to a bounded set,
the above closed-form of eigen-pairs will no longer hold,
and the analysis will become rather challenging.  

To resolve this issue,
we resort to the period-ization technique proposed by \cite{WilSmoSch01}.
Consider $\kappa(x) = \exp(-x^2/(2\sigma^2))$ when $x \in [-v/2, v/2]$,
and then extend $\kappa$ to $\RR$ as a periodic function with period $v$.
Again let $\mu$ be the uniform distribution on $[-v/2, v/2]$.
As can be seen from the discriminant function 
$f = \frac{1}{l} \sum_{i=1}^l \gamma_i k(x^i, \cdot)$,
as along as our training and test data both lie in $[-v/4, v/4]$,
the modification of $\kappa$ outside $[-v/2,v/2]$ does not effectively make any difference.
Although the term $\partial^{0,1} k_0(x^{a}_1, w^1_1)$ in 
\eqref{eq:nystrom_coordinate_grad} may possibly evaluate $\kappa$ outside $[-v/2,v/2]$,
it is only used for testing the gradient norm bound of $\kappa$.

With this periodized Gaussian kernel, it is easy to see that 
$Q_\epsilon = \sqrt{2}$.
If we standardize by $\sigma=1$ and set $v=5\pi$ as an example,
it is not hard to see that \eqref{eq:eigenvalue_cond_period_2} holds with
$c_4 = 1.25$ and $c_6 = 50$.
The expressions of $N_\epsilon$ and $M_\epsilon$ then follow from \autoref{thm:Meps_period} directly.

\subsection{Case 3: Checking Assumptions \ref{assum:tail_bnd} and \ref{assum:uniform_bnd} on non-product kernels}
\label{sec:assump_check_inverse}

The above analysis has been restricted to product kernels.
But in practice,
there are many useful kernels that are not decomposable.
A prominent example is the inverse kernel: $k(x,y) = (2-x^\top y)^{-1}$.
In general, it is extremely challenging to analyze eigenfunctions, 
which are commonly \emph{not} bounded \citep{Zhou02,LafLeb05}, \ie, 
$\sup_{i \to \infty} \sup_x \abr{e_i(x)} = \infty$.
The opposite was (incorrectly) claimed in Theorem 4 of \citet{WilSmoSch01} by citing an incorrect result in \citet[][p. 145]{Konig86}, 
which was later corrected by \citet{Zhou02} and Steve Smale.
Indeed, uniform boundedness is not known even for Gaussian kernels with uniform distribution on $[0,1]^d$ \cite{LinGuoZho17}, 
and \citet[][Theorem 5]{MinNiyYao06} showed the unboundedness for Gaussian kernels with uniform distribution on the unit sphere when $d \ge 3$.

Here we only present the limited results that we have obtained on the eigenvalues of the integral operator of inverse kernels with a uniform distribution on the unit ball.
The analysis of eigenfunctions is left for future work.
Specifically, in order to drive the eigenvalue $\lambda_i$ below $\epsilon$,
$i$ must be at least $d^{\ceil{\log_2 \frac{1}{\epsilon}}+1}$.
This is a quasi-quadratic bound if we view $d$ and $1/\epsilon$ as two large variables.

It is quite straightforward to give an explicit characterization of the functions in $\Hcal$.
The Taylor expansion of $z^{-1}$ at $z=2$ is 
$\frac{1}{2} \sum_{i=0}^\infty (-\frac{1}{2})^i x^i$.
Using the standard multi-index notation with $\val = (\alpha_1, \ldots, \alpha_d) \in (\NN \cup \{0\})^d$, $\abr{\val} = \sum_{i=1}^d \alpha_i$, 
and
$\xvec^\val = x_1^{\alpha_1} \ldots x_d^{\alpha_d}$,
we derive
\begin{align}
	k(\xvec, \yvec) 
	&= \frac{1}{2 - \xvec^\top \yvec}
	\\&= \frac{1}{2} \sum_{k=0}^\infty \rbr{-\frac{1}{2}}^k (-\xvec^\top \yvec)^k
	\\&= \sum_{k=0}^\infty 2^{-k-1} \sum_{\val: \abr{\val} = k} C^k_\val \xvec^\val \yvec^\val
	\\&= \sum_{\val} 2^{-\abr{\val}-1}  C^{\abr{\val}}_\val \xvec^\val \yvec^\val,
\end{align}
where $C^k_\val = \frac{k!}{\prod_{i=1}^d \alpha_i!}$.
So we can read off the feature mapping for $\xvec$ as
\begin{align}
\phi(\xvec) = \{ w_\val \xvec^\val: \val \},
\where 
w_\val = 2^{-\frac{1}{2} (\abr{\val} + 1)} C^{\abr{\val}}_\val,
\end{align}
and the functions in $\Hcal$ are
\begin{align}
\label{eq:RKHS_expression}
\Hcal = \cbr{f = \sum_{\val} \theta_\val w_\val \xvec^\val : \nbr{\vth}_{\ell_2} < \infty}.
\end{align}

Note this is just an intuitive ``derivation'' while a rigorous proof for \eqref{eq:RKHS_expression} can be constructed in analogy to that of Theorem 1 in \citet{Minh10}.

\subsection{Background of eigenvalues of a kernel}

We now use \eqref{eq:RKHS_expression} to find the eigenvalues of inverse kernel.

Now specializing to our inverse kernel case,
let us endow a uniform distribution over the unit ball $B$: $p(x) = V_d^{-1}$ where 
$V_d = \pi^{d/2} \Gamma(\frac{d}{2}+1)^{-1}$ is the volume of $B$,
with $\Gamma$ being the Gamma function.
Then $\lambda$ is an eigenvalue of the kernel if there exists $f = \sum_{\val} \theta_\val w_\val \xvec^\val$ such that
$\int_{\yvec \in B} k(\xvec,\yvec) p(\yvec) f(\yvec) \rmd \yvec = \lambda f(\xvec)$.
This translates to
\begin{align}
	V_d^{-1} \int_{\yvec \in B} \sum_{\val} w_\val^2 \xvec^\val \yvec^\val \sum_{\vbeta} \theta_\vbeta w_\vbeta \yvec^\vbeta \rmd \yvec 
	= \lambda \sum_{\val} \theta_\val w_\val \xvec^\val, \qquad \forall \ \xvec \in B.
\end{align}
Since $B$ is an open set, that means
\begin{align}
w_\val \sum_\vbeta w_\vbeta q_{\val+\vbeta} \theta_\vbeta 
= \lambda \theta_\val, \qquad \forall \ \val,
\end{align}
where 
\begin{align}
q_{\val} = 
V_d^{-1} \int_{\yvec \in B} \yvec^{\val} \rmd \yvec = 
\begin{cases}
\frac{2 \prod_{i=1}^d \Gamma\rbr{\smallfrac{1}{2} \alpha_i + \smallfrac{1}{2}}}{V_d \cdot (\abr{\val} + d)\cdot \Gamma\rbr{\smallfrac{1}{2} \abr{\val} + \smallfrac{d}{2}}} & \text{if all } \alpha_i \text{ are even} \\
0 & \text{otherwise}
\end{cases}.
\end{align}
In other words, $\lambda$ is the eigenvalue of the infinite dimensional matrix $Q = [w_\val w_\vbeta q_{\val + \vbeta}]_{\val, \vbeta}$,

\subsection{Bounding the eigenvalues}
To bound the eigenvalues of $Q$, we resort to the majorization results in matrix analysis.
Since $k$ is a PSD kernel, all its eigenvalues are nonnegative,
and suppose they are sorted decreasingly as $\lambda_1 \ge \lambda_2 \ge \ldots$.
Let the row corresponding to $\val$ have $\ell_2$ norm $r_\val$, 
and let them be sorted as $r_{[1]} \ge r_{[2]} \ge \ldots$.
Then by \cite{Schneider53,ShiWan65}, we have
\begin{align}
\prod_{i=1}^n \lambda_i \ \le \ \prod_{i=1}^n r_{[i]}, \quad \forall \ n \ge 1.
\end{align}
So our strategy is to bound $r_\val$ first.
To start with, we decompose $q_{\val+\vbeta}$ into $q_\val$ and $q_\vbeta$ via Cauchy-Schwartz:
\begin{align}
q_{\val+\vbeta}^2 = V_d^{-2} \rbr{\int_{\yvec \in B} \yvec^{\val+\vbeta} \rmd \yvec}^2
\le V_d^{-2} \int_{\yvec \in B} \yvec^{2 \val} \rmd \yvec \cdot
\int_{\yvec \in B} \yvec^{2 \vbeta} \rmd \yvec
= q_{2\val} q_{2\vbeta}.
\end{align}
To simplify notation, we consider without loss of generality that $d$ is an even number, and denote the integer $b \defeq d/2$.
Now $V_d = \pi^b / b!$.
Noting that there are $\mychoose{k+d-1}{k}$ values of $\vbeta$ such that $\abr{\vbeta} = k$, 
we can proceed by (fix below by changing $\mychoose{k+d}{k}$ into $\mychoose{k+d-1}{k}$, or no need because the former upper bounds the latter)
\begin{align}
r_\val^2 
&= w_\val^2 \sum_\vbeta w_\vbeta^2 q_{\val+\vbeta}^2 
\le w_\val^2 q_{2\val} \sum_\vbeta w_\vbeta^2 q_{2\vbeta}
= w_\val^2 q_{2\val} \sum_{k=0}^\infty 2^{-k-1} \sum_{\vbeta: \abr{\vbeta} = k} C^k_\vbeta q_{2\vbeta} \\
&\le w_\val^2 q_{2\val} \sum_{k=0}^\infty 2^{-k-1} \mychoose{k+d}{d}
\max_{\abr{\vbeta} = k} C^k_\vbeta q_{2\vbeta} \\
&= w_\val^2 q_{2\val} \sum_{k=0}^\infty 2^{-k-1} \mychoose{k+d}{d}
\max_{\abr{\vbeta} = k} \frac{k!}{\prod_{i=1}^d \beta_i!} 
\cdot \frac{2 \prod_{i=1}^d \Gamma(\beta_i + \frac{1}{2})}{V_d \cdot (2k+d) \cdot \Gamma(k + \frac{d}{2})} \\
&=  w_\val^2 q_{2\val} V_d^{-1} \sum_{k=0}^\infty 2^{-k} \mychoose{k+d}{d} \frac{k!}{(2k+d) \Gamma(k+\frac{d}{2})} \cdot 
\max_{\abr{\vbeta} = k} \prod_{i=1}^d \frac{\Gamma(\beta_i + \frac{1}{2})}{\beta_i!} \\
&< w_\val^2 q_{2\val} \cdot \frac{b!}{\pi^b d! } \cdot
\sum_{k=0}^\infty 2^{-k-1} \frac{(k+d)!}{(k+b)!},
\end{align}
since $\Gamma(\beta_i + \smallfrac{1}{2}) < \Gamma(\beta_i + 1) = \beta_i!$.
%
%where the last step used the fact that 
%$\Gamma(\beta_i + \frac{1}{2}) < \Gamma(\beta_i + 1) = \beta_i!$.
The summation over $k$ can be bounded by
\begin{align}
\sum_{k=0}^\infty 2^{-k-1} \frac{(k+d)!}{(k+b)!} 	 
= \frac{1}{2}b! \rbr{2^d + \mychoose{d}{b}} 
\le \frac{1}{2} \rbr{b! 2^d + 2^b} \le b! 2^d,
\end{align}
where the first equality used the identity
$\sum_{k=1}^\infty 2^{-k} \mychoose{d+k}{b} = 2^d$.
Letting $l \defeq \abr{\val}$, we can continue by
\begin{align}
r_\val^2 
&<  w_\val^2 q_{2\val} \cdot \frac{b!}{\pi^b d! } b! 2^d 
= 2^{-l - 1} \frac{l!}{\prod_{i=1}^d \alpha_i!} \frac{2 \prod_{i=1}^d \Gamma\rbr{\alpha_i + \smallfrac{1}{2}}}{V_d \cdot (2l + d)\cdot \Gamma\rbr{ l + b}} \frac{(b!)^2 2^d}{\pi^b d! } \\
&\le 2^{-l + d} \pi^{-2b} \frac{l! (b!)^3}{d! (l+b-1)! (2l+d)} \qquad (\text{since } \Gamma(\alpha_i+\smallfrac{1}{2}) < \Gamma(\alpha_i+1) = \alpha_i!) \\
&\le 2^{-l + b -1} \pi^{-2b} \mychoose{l+b}{l}^{-1} \qquad (\text{since } \frac{(b!)^2}{d!} \le 2^{-b}).
\end{align}

This bound depends on $\val$, not directly on $\val$.
Letting $n_l = \mychoose{l+d-1}{l}$ and $N_L = \sum_{l=0}^L n_l = \mychoose{d+L}{L}$,
it follows that
\begin{align}
\sum_{l=0}^L l n_l
=
\sum_{l=1}^L 
\frac{l (l+d)!}{d! \cdot l!}
&=
(d+1) \sum_{l=1}^L 
\frac{(l+d)!}{(d+1)! (l-1)!} \\
=
&(d+1) \sum_{l=1}^{L} \mychoose{l+d}{d+1} = (d+1) \mychoose{L+d+1}{d+2}.
\end{align}
Now we can bound $\lambda_{N_L}$ by
\begin{align}
\lambda_{N_L}^{N_L} &\le \prod_{i=1}^{N_L} \lambda_i 
\le 
\prod_{l=0}^L 
\rbr{2^{-l + b -1} \pi^{-2b} \mychoose{l+b}{l}^{-1}}^{n_l} \\
\implies \quad \log \lambda_{N_L} &\le N_L^{-1} \sum_{l=0}^L n_l 
\rbr{-(l-b+1) \log 2 - 2b \log \pi - \log \mychoose{l+b}{l}} \\
&\le -N_L^{-1} \cdot \log 2 \cdot \sum_{l=0}^L l n_l
\intertext{since $\log 2 < 2 \log \pi$ as the coefficients of  $b$}
&= -\mychoose{d+L+1}{d+1}^{-1} \cdot \log 2 \cdot (d+1) \mychoose{d+L+1}{d+2} \\
&= -\frac{d+1}{d+2} L \log 2 \\
&\approx - L \log 2 \\
\implies \quad \lambda_{N_L} &\le 2^{-L}.
\end{align}

This means that the eigenvalue $\lambda_i \le \epsilon$ provided that 
$i \ge N_L$ where $L = \ceil{\log_2 \frac{1}{\epsilon}}$.
Since $N_L \le d^{L+1}$,
that means it suffices to choose $i$ such that
\begin{align}
i \ge d^{\ceil{\log_2 \frac{1}{\epsilon}}+1}.
\end{align}

This is a quasi-polynomial bound.
It seems tight because even in Gaussian RBF kernel,
the eigenvalues follow the order of $\lambda_\val = O(c^{-\abr{\val}})$ for some $c > 1$
\citep[p.A742]{FasMcC12}.

\section{Algorithm for training a Lipschitz binary SVMs}
\label{sec:algo_detail}

The pseudo-code of training binary SVMs by enforcing Lipschitz constant is given in Algorithm \ref{algo:constrlip}.

% \IncMargin{0.5em}
\begin{algorithm}[ht]
	\LinesNotNumbered
  \nextnr
  Initialise the constraint set $S$ by some random samples from $X$.
  
  \nextnr
  \For{$i = 1, 2, \ldots$}{
    \nextnr
    Train SVM using one of the following constraints:
    
    \begin{enumerate}[label=\textcircled{\raisebox{-0.9pt}{\small\arabic*}},rightmargin=2cm]
        \item {\bf Brute-force:} $\nbr{\nabla f(w)}_2^2 \le L^2,\ \forall \ w \in S$
        \item {\bf Nystr\"om holistic:} $\lambda_{\max} (\Gtil^\top \Gtil) \le L^2$ using $S=\{w^1, \ldots, w^n\}$ in \eqref{eq:nystrom_approx}
        \item {\bf Nystr\"om coordinate wise:} $\lambda_{\max}(\Ptil_G) \le L^2$ using $S=\{w^1, \ldots, w^n\}$ in \eqref{eq:nystrom_coordinate}
    \end{enumerate}
    
    \nextnr
   Let the trained SVM be $f^{(i)}$.
    
    \nextnr
    Find a new $w$ to add to $S$ by one of the following methods:
    \begin{enumerate}[label=\textcircled{\small\alph*},rightmargin=2cm]
        \item  {\bf Random:} randomly sample $w$ from $X$.
        \item  {\bf Greedy:} find $\argmax_{x\in X} \nbr{\nabla f^{(i)}(x)}$ (local optimisation) by L-BFGS with 10 random initialisations
        and add the distinct results upon convergence to $S$.
    \end{enumerate}
    
    \nextnr
    {\bf Return} if $L^{(i)} \defeq  \max_{x\in X} \nbr{\nabla f^{(i)}(x)}$ falls below $L$.
  }
% %
    \caption{Training binary SVMs by enforcing Lipschitz constant $L$}\label{algo:constrlip}
% \Indentp{-2em}    
\end{algorithm}

Finding the exact $\argmax_{x\in X} \nbr{\nabla f^{(i)}(x)}$ is intractable, so we used  a  local  maximum  found  by  L-BFGS  with  10 random  initialisations as the Lipschitz constant of the current solution $f^{(i)}$ ($L^{(i)}$ in step 6).
The solution found by L-BFGS is also used as the new greedy point added in step 5b.

Furthermore, the kernel expansion
$f(x) = \frac{1}{l} \sum_{a=1}^l \gamma_a k(x^a, \cdot)$ can lead to high cost in optimisation 
(our experiment used $l = 54000$),
and therefore we used \emph{another} Nystr\"om approximation for the kernels.
We randomly sampled 1000 landmark points,
and based on them we computed the Nystr\"om approximation for each $k(x^a, \cdot)$, denoted as $\tilde{\phi}(x^a) \in \R^{1000}$.
Then $f(x)$ can be written as 
$\frac{1}{l} \sum_{a=1}^l \gamma_a \tilde{\phi}(x^a)^\top \tilde{\phi}(x)$.
Defining $w = \frac{1}{l} \sum_{a=1}^l \gamma_a \tilde{\phi}(x^a)$,
we can equivalently optimise over $w$,
and the RKHS norm bound on $f$ can be equivalently imposed as the $\ell_2$-norm bound on $w$.

To summarise, Nystr\"om approximation is used in two different places:
one for approximating the kernel function, 
and one for computing $\nbr{g_j}_\Hcal$ either holistically or coordinate wise.
For the former, we randomly sampled 1000 landmark points;
for the latter, we used greedy selection as option b in step 5 of Algorithm~\ref{algo:constrlip}.

\subsection{Detailed algorithm for multiclass classification}
\label{sec:algo_multiclass}

It is easy to extend Algorithm \ref{algo:constrlip} to multiclass. For example, with MNIST dataset, we solve the following optimisation problem to defend $\ell_2$ attacks:
\begin{gather}
	\mathclap{\begin{aligned}
		\underset{\vga^1, \ldots, \vga^{10}}{\mathrm{minimise}}&\quad \sum_{i=1}^n \ell(F(x), \yvec), \quad\text{where}\quad F \defeq \sbr{\sum_{i=1}^n \vga_i^1 k(x_i, \cdot);\ldots;\sum_{i=1}^n \vga_i^{10} k(x_i, \cdot)  }
		\\\mathrm{subject\ to}&\quad \sup_{\nbr{\phi}_{\Hcal}\le 1} \lambda_{\max} \rbr{\sum_{c=1}^{10} G_c^\top \phi \phi^\top G_c} \approx 
		\sup_{\nbr{v}_2\le 1} \lambda_{\max} \rbr{\sum_{c=1}^{10} \Gtil_c^\top v v^\top \Gtil_c} \le L^2,
	\end{aligned}}
\end{gather}
where $\ell(F(x), \yvec)$ is the Crammer \& Singer loss, 
and the constraint is derived from \eqref{eq:lip_multiclass} by using its Nystr\"om approximation $\Gtil_c = \sbr{\gtil_1^c, \ldots, \gtil_d^c}$,
which depends on $\{\vga^1,\ldots, \vga^{10}\}$ linearly. 
Note that the constraint itself is a supremum problem:
\begin{align}
\sup_{\nbr{v}_2\le 1} \lambda_{\max} \rbr{\sum_{c=1}^{10} \Gtil_c^\top v v^\top \Gtil_c}
 = 
\sup_{\nbr{v}_2\le 1, \nbr{u}_2\le 1} u^\top \rbr{\sum_{c=1}^{10} \Gtil_c^\top v v^\top \Gtil_c} u.
\end{align}

Since there is only one constraint, interior point algorithm is efficient.
It requires the gradient of the constraint, which can be computed by Danskin's theorem.
In particular, we alternates between updating $v$ and $u$,
until they converge to the optimal $v_*$ and $u_*$. 
Finally, the derivative of the constraint with respect to  $\{\vga^c\}$ can be calculated from  
$\sum_{c=1}^{10} (u_*^\top \Gtil_c^\top v_*)^2$,
as a function of $\{\vga^c\}$.

To defend $\infty$-norm attacks, we need to enforce the  $\infty$-norm of the Jacobian matrix:
\begin{align*}
\sup_{x \in \Xcal} \nbr{ \sbr{g^1(x), \ldots, g^{10}(x)}^\top }_\infty &= \sup_{x \in \Xcal} \max_{1 \le c \le 10} \nbr{g^c(x)}_1 \\
&=  \max_{1 \le c \le 10} \sup_{x \in \Xcal} \nbr{g^c(x)}_1\\
&\le  \max_{1 \le c \le 10} \sup_{\nbr{\phi}_2\le 1, \nbr{u}_\infty\le 1} u^\top \Gtil_c^\top \phi,
\end{align*}
where the last inequality is due to 
\begin{align*}
\sup_{x \in \Xcal} \nbr{g(x)}_1
= \sup_{x \in \Xcal} \sup_{\nbr{u}_\infty\le 1} u^\top g(x) 
\le \sup_{\nbr{v}_2\le 1, \nbr{u}_\infty\le 1} u^\top \Gtil^\top v.
\end{align*}
Therefore, the overall optimisation problem for defense against $\infty$-norm attacks is
\begin{gather}
	\begin{aligned}
		\underset{\vga^1, \ldots, \vga^{10}}{\mathrm{minimise}}&\quad 
			\sum_{i=1}^n \ell(F(x), \yvec), 
		\\\mathrm{subject\ to}&\quad\forall{c\in[10]} 
		\sup_{\nbr{v}_2\le 1, \nbr{u}_\infty\le 1} u^\top \Gtil_c^\top v \le L
	\end{aligned}\label{eq:app_linf_constr}
\end{gather}
For each $c$, we alternatively update $v$ and $u$ in \eqref{eq:app_linf_constr}, 
converging to the optimal $v_*$ and $u_*$. 
Finally, the derivative of $\sup_{\nbr{v}_2\le 1, \nbr{u}_\infty\le 1} u^\top \Gtil_c^\top v$ with respect to  $\vga^c$ can be calculated from  $u_*^\top \Gtil_c^\top v_*$,
as a function of $\vga^c$. 
 
\section{More experiments}
\label{sec:robust_kk_experiments}

\subsection{Efficiency of enforcing Lipschitz constant by different methods}
\label{sec:convergence}

% \begin{wrapfigure}[17]{r}{0.5\textwidth}
\begin{figure*}[h!]
	\centering
% 	\subfloat[Inverse kernel]{
%		\label{fig:acc_pgd_train}
		\includegraphics[clip=true, width=0.5\textwidth, viewport = 3cm 8cm 17cm 19cm]{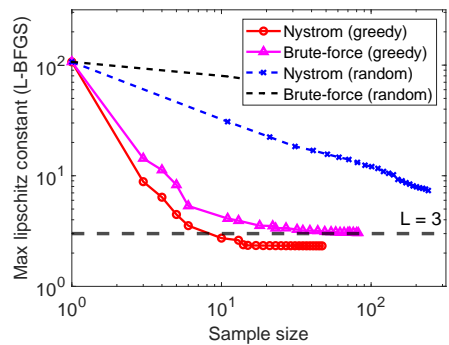}
% 	}
% 	\hspace{1em}
% 	\subfloat[Gaussian kernel with $\sigma = 0.5$]{
% %		\label{fig:acc_pgd_train}
% 		\includegraphics[clip=true, width=0.29\textwidth, viewport = 3cm 8cm 17cm 19cm]{{{./image/rbf_0.5_lip_converge}}}
% 		}
% 	\hspace{1em}
% 	\subfloat[Gaussian kernel with $\sigma = 1$]{
% %		\label{fig:acc_pgd_test}
% 		\includegraphics[clip=true, width=0.29\textwidth, viewport = 3cm 8cm 17cm 19cm]{./image/rbf_1_lip_converge}
% 		}
	\caption{Comparison of efficiency in enforcing Lipschitz constant by various methods}
	\label{fig:comp_convergence}
\end{figure*}
% \end{wrapfigure}

The six different ways to train SVMs with Lipschitz regularisation are summarized in Algorithm~\ref{algo:constrlip}.
\autoref{fig:comp_convergence} plots how fast the regularisation on gradient norm becomes effective when more and more points $w$ are added to the constraint set.
We call them ``samples'' although it is not so random in the greedy method, modulo the random initialization of BFGS within the greedy method.
The horizontal axis is the loop index $i$ in Algorithm~\ref{algo:constrlip},
and the vertical axis is $L^{(i)}$ therein,
which is the estimation of the Lipschitz constant of the current solution $f^{(i)}$.
We used 400 random examples (200 images of digit 1 and 200 images of digit 0) in the MNIST dataset and set $L=3$ and RKHS norm $\nbr{f}_\Hcal \le \infty$ for all algorithms.
Inverse kernel is used, hence no results are shown for coordinate-wise Nystr\"om.

Clearly the Nystr\"om algorithm is more efficient than the Brute-force algorithm, 
and the greedy method significantly reduces the number of samples for both algorithms. 
In fact, Nystr\"om with greedy selection eventually fell below the prespecified $L$, 
because of the gap in \eqref{eq:norm_grad_by_rkhs}.
\subsection{More results on Cross-Entropy attacks}
\label{sec:cw_result}

\begin{figure}[H]
	\centering
	\subcaptionbox{MNIST}[0.33\linewidth-0.5em]{
		\label{fig:mnist_l2_crossent}
		\includegraphics[clip=true, width=\linewidth, viewport = 3.5cm 9.5cm 17.7cm 17.5cm]{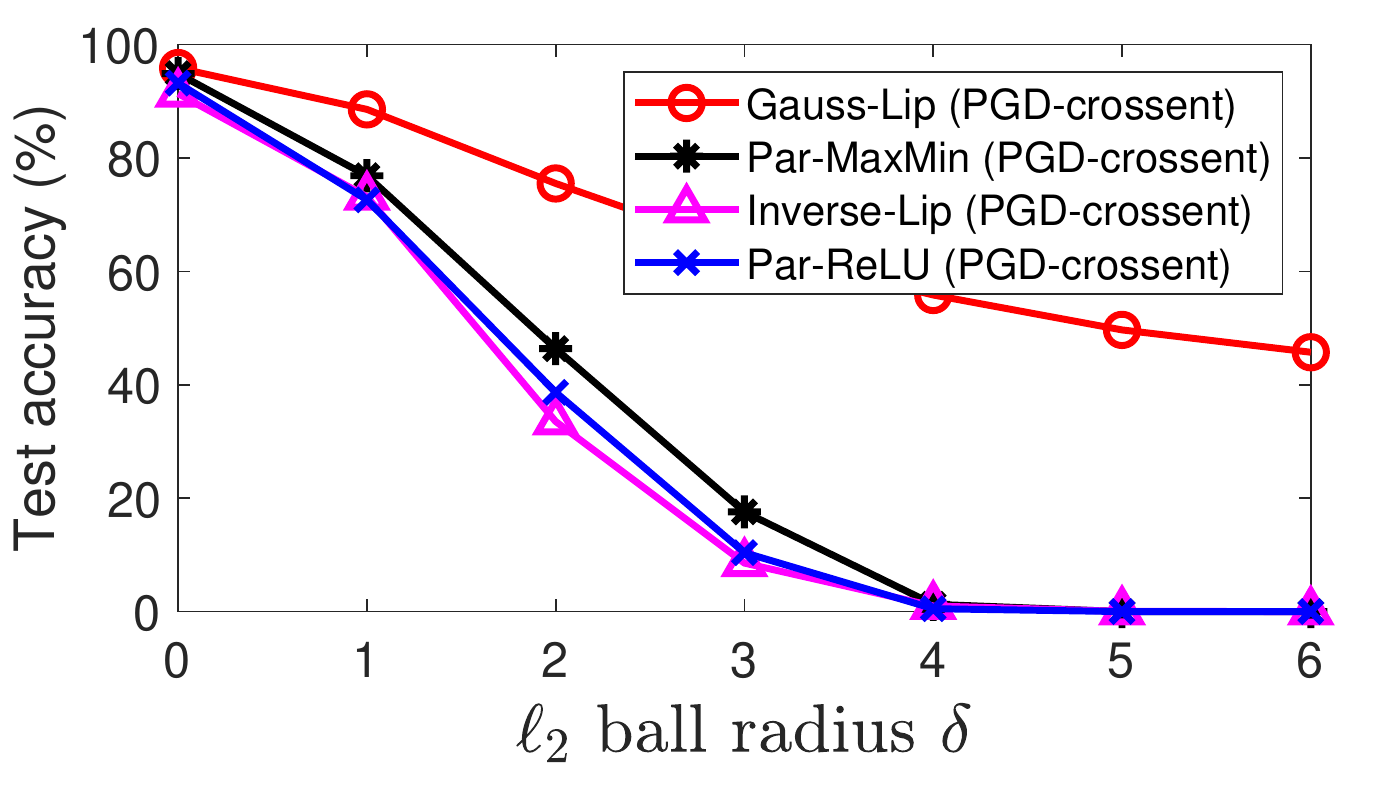}
		}\hfill
	\subcaptionbox{Fashion-MNIST}[0.33\linewidth-0.5em]{
		\label{fig:fashion_l2_crossent}
		\includegraphics[clip=true, width=\linewidth, viewport = 3.5cm 9.5cm 17.7cm 17.5cm]{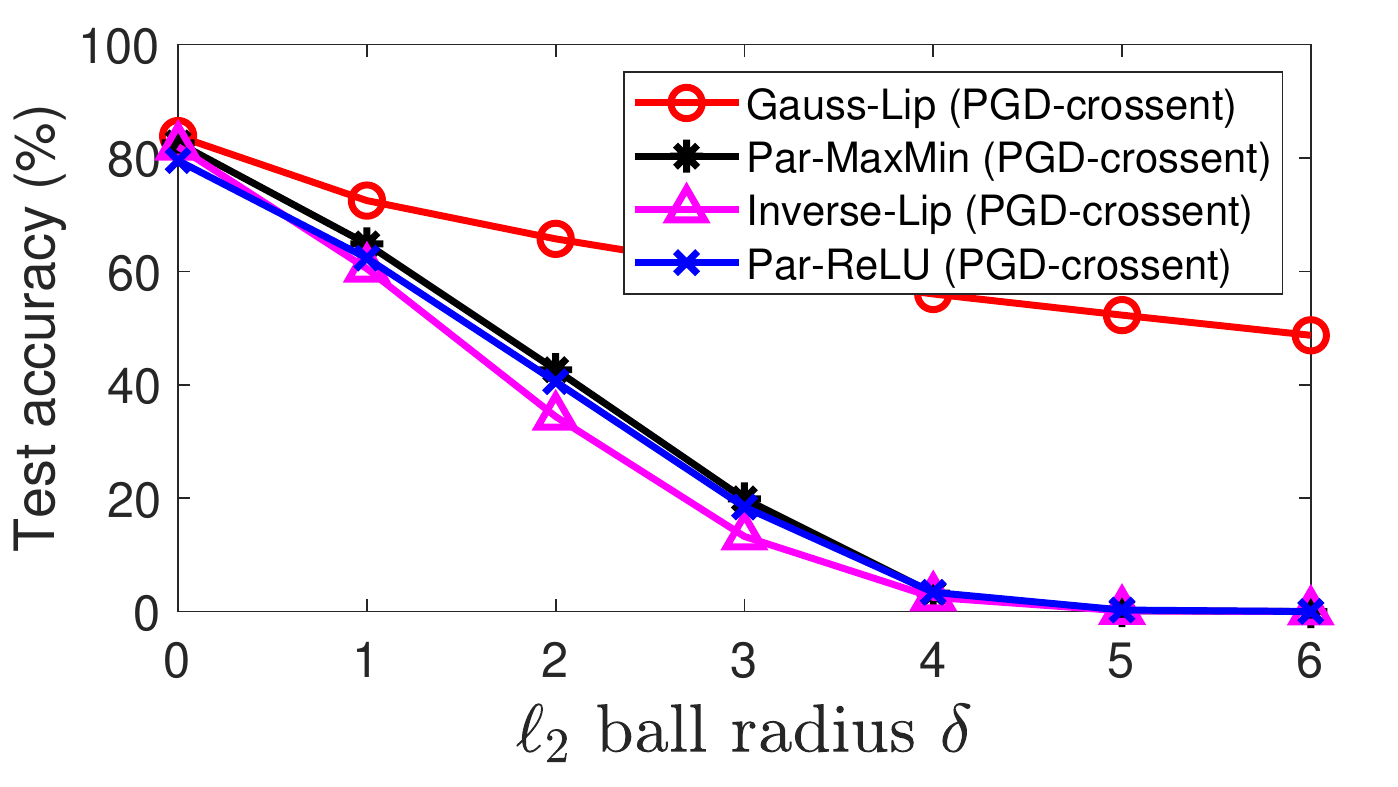}
		}\hfill
	\subcaptionbox{CIFAR10}[0.33\linewidth-0.5em]{
		\label{fig:cifar_l2_crossent}
		\includegraphics[clip=true, width=\linewidth, viewport = 3.5cm 9.5cm 17.7cm 17.5cm]{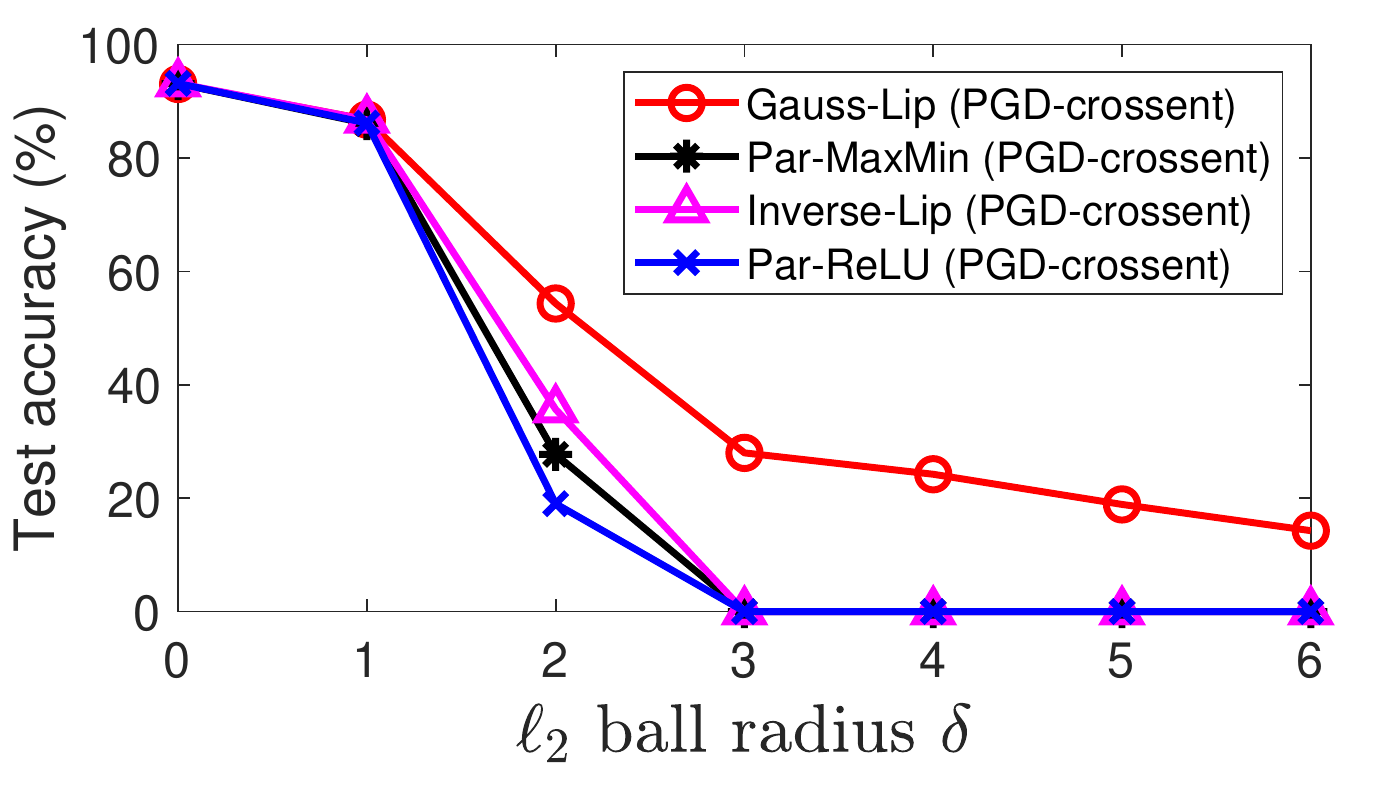}
		}
	\caption{Test accuracy under PGD attacks on cross-entropy approximation with $\ell_2$ norm bound}
	\label{fig:robustness_l2_crossent}
	\vspace{\abovecaptionskip}
	\subcaptionbox{MNIST}[0.33\linewidth-0.5em]{
		\label{fig:mnist_linf_crossent}
		\includegraphics[clip=true, width=\linewidth, viewport = 3.5cm 9.5cm 17.7cm 17.5cm]{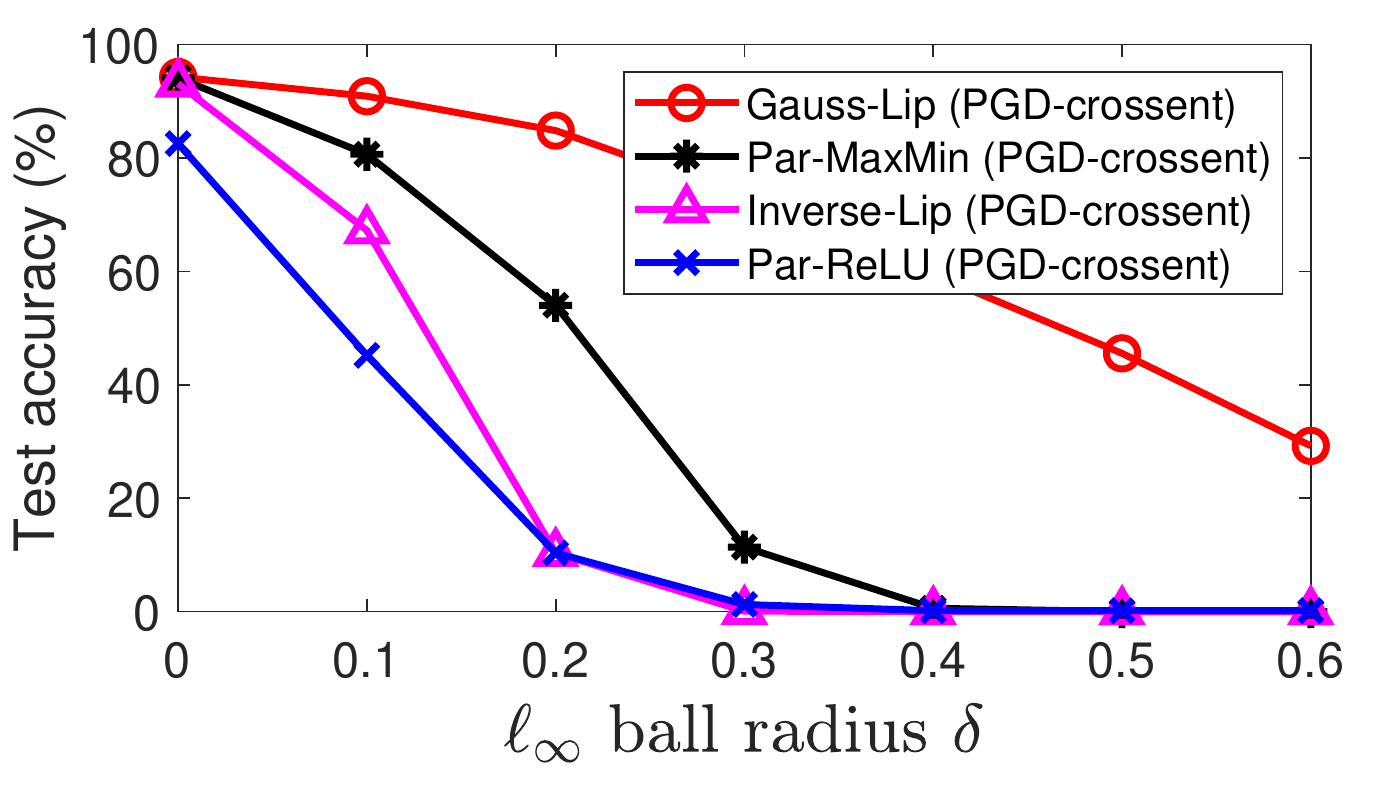}
		}\hfill
	\subcaptionbox{Fashion-MNIST}[0.33\linewidth-0.5em]{
		\label{fig:fashion_linf_crossent}
		\includegraphics[clip=true, width=\linewidth, viewport = 3.5cm 9.5cm 17.7cm 17.5cm]{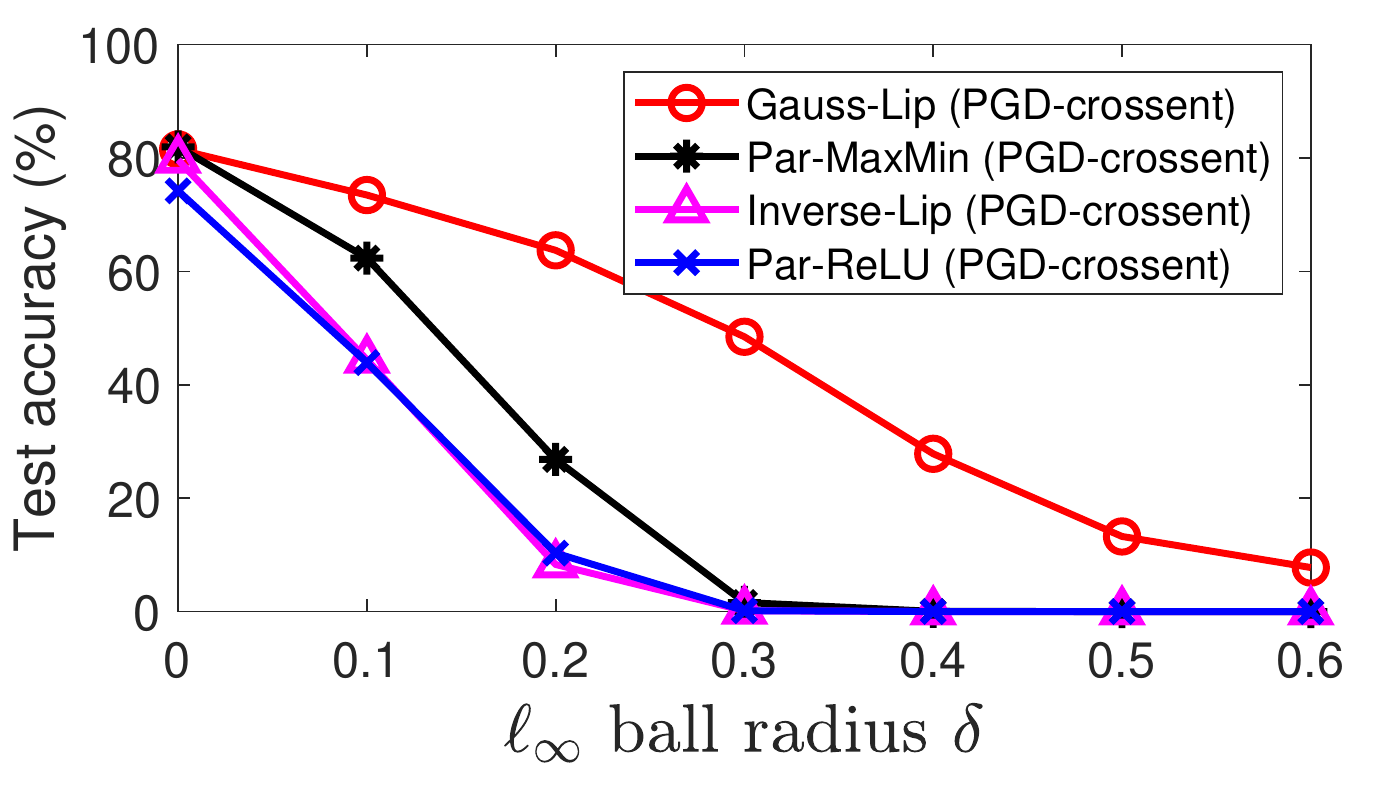}
		}\hfill
	\subcaptionbox{CIFAR10}[0.33\linewidth-0.5em]{
		\label{fig:cifar_linf_crossent}
		\includegraphics[clip=true, width=\linewidth, viewport = 3.5cm 9.5cm 17.7cm 17.5cm]{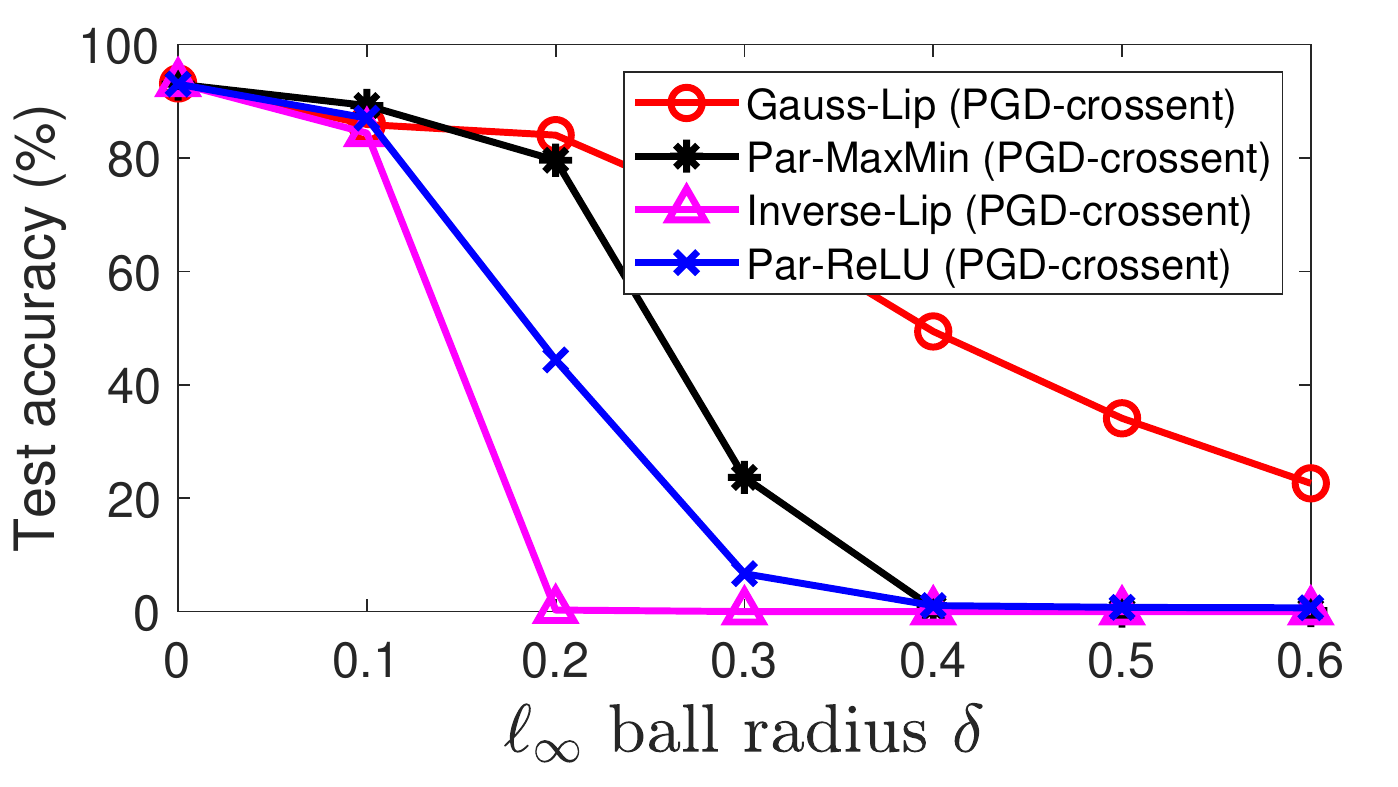}
		}
	\caption{Test accuracy under PGD attacks on cross-entropy approximation with $\infty$-norm bound}
	\label{fig:robustness_linf_crossent}
\end{figure}

% \subsection{Consideration for obfuscated gradient}
% \label{sec:obfuscate}

% The improved robustness of \glip\ does not seem to be attributed to the obfuscated gradient \citep{AthCarWag18},
% because as shown \ref{fig:robustness_l2_cw}, \ref{fig:robustness_linf_cw}, \ref{fig:robustness_l2_crossent}, \ref{fig:robustness_linf_crossent},
% increased distortion bound does increase attack success and unbounded attacks drive the success rate to very low.
% In practice, we also observed that random sampling finds much weaker attacks, and taking 10 steps of PGD is much stronger than just one step.

\subsection{Visualization of attacks}
\label{sec:visualization}

In order to verify that the robustness of \glip\ is not due to obfuscated gradient,
we randomly sampled 10 images from MNIST,
and ran \textbf{targeted} PGD for 100 steps with cross-entropy objective and the $\ell_2$ norm upper bounded by 8. For example, in \autoref{fig:crossent_PGD100_Delta6_new}, the row corresponding to class 4 tries to promote the likelihood of the target class 4. Naturally the diagonal is not meaningful, hence left empty. 
At the end of attack, PDG turned 89 out of 90 images into the target class by following the gradient of the defense model.

Please note that despite the commonality in using the cross-entropy objective,
the setting of targeted attack in \autoref{fig:crossent_PGD100_Delta6_new} is not comparable to that in \autoref{fig:robustness_l2_crossent},
where to enable a batch test mode, an \emph{untargeted} attacker was employed by increasing the cross-entropy loss of the correct class, i.e., decreasing the likelihood of the correct class.
This is a common  practice.

We further ran PGD for 100 steps on C\&W approximation (an untargeted attack used in \autoref{fig:robustness_l2_cw}), and the resulting images after every 10 iterations are shown in \autoref{fig:CW_PGD100_Delta8_new}. Here all 10 images were eventually turned into a different but untargeted class, and the final images are very realistic.

\begin{figure}[h]
\centering
    \subcaptionbox{\label{fig:P9_sample2_100step}}[0.5\linewidth]{
    \includegraphics[width=\linewidth]{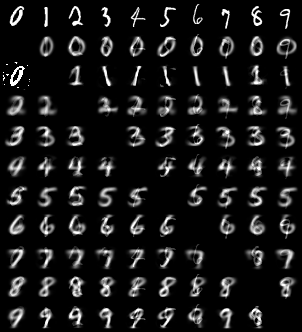}}
	\par
    \vspace{\baselineskip}
	\subcaptionbox{\label{fig:pred_sample2_100step}}[0.5\linewidth]{
	\includegraphics[width=\linewidth]{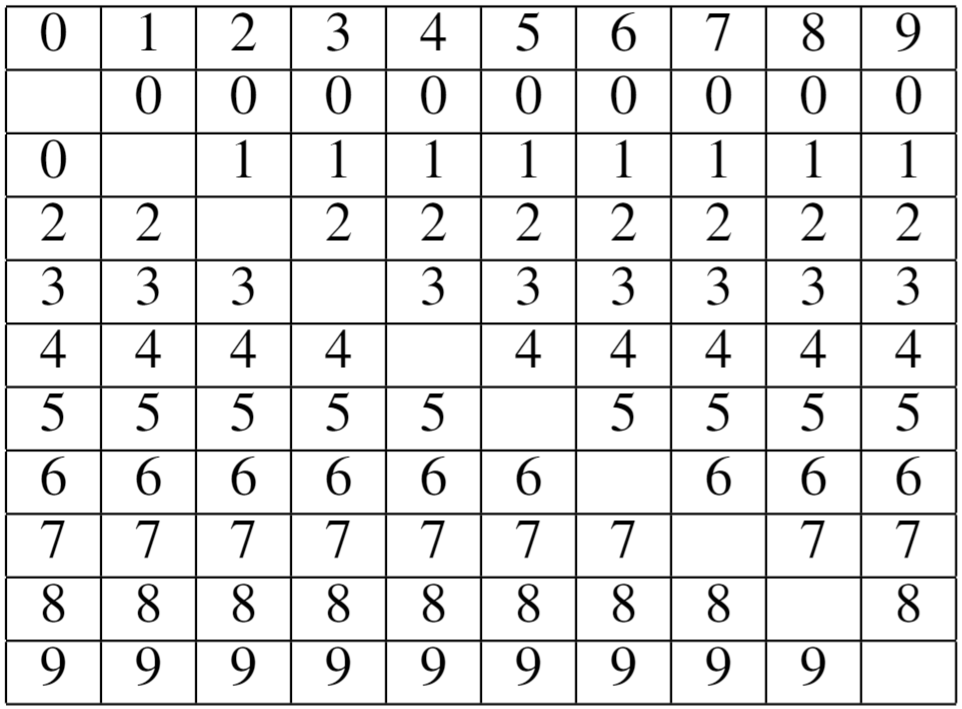}}
	\caption{\subref{fig:P9_sample2_100step} perturbed images at the end of 100-step PGD attack using the (\textbf{targeted}) cross-entropy approximation.  
	The top row shows 10 random images, one sampled from each class.
	The 10 rows below correspond to the target class.
% 	The attack is $\ell_2$ with $\delta = 6$.
	\subref{fig:pred_sample2_100step} classification on the perturbed image given by the trained \glip. The left images are quite consistent with human's perception.}
	\label{fig:crossent_PGD100_Delta6_new}
\end{figure}

\begin{figure}[h]
\centering
%     \includegraphics[scale=0.315]{./image/visualization/gradient/P9_delta_8}
% 	~~~
% 	\includegraphics[scale=0.24]{./image/visualization/gradient/pred_delta_8}
% 	\caption{Left: perturbed images at the end of 100-step PGD attack using (targeted) cross-entropy approximation.  
% 	The attack is $\ell_2$ with $\delta = 8$.
% 	Right: classification on the perturbed image given by the trained \glip. They are quite consistent with human's perception on the left images.}
% 	\label{fig:crossent_PGD100_Delta8}
\includegraphics[width=0.5\linewidth]{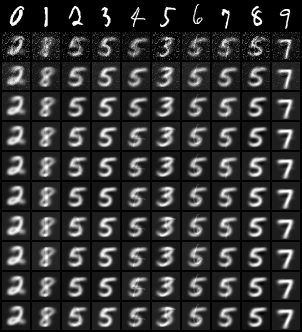}
\caption{Perturbed images at the end of 100-step PGD attack using the (\textbf{untargeted}) C\&W approximation. 
% 	The attack is $\ell_2$ with $\delta = 6$.
The top row shows 10 random images, one sampled from each class.
	The 10 rows below show the images after 10, 20, ..., 100 steps of PGD.}
	\label{fig:CW_PGD100_Delta8_new}
\end{figure}

\end{document}